\pdfoutput=1
\documentclass[11pt]{article}
\usepackage[font={it,small},justification=justified]{caption}
\usepackage{amsmath, mathpazo, amssymb, amsfonts, amsthm}
\usepackage[T1]{fontenc}
\usepackage{setspace}
\usepackage{parskip}
\usepackage{soul}
\usepackage{url}
\usepackage[utf8]{inputenc}
\usepackage[svgnames,xcdraw,table]{xcolor}
\usepackage{booktabs}
\usepackage{threeparttable}
\usepackage{amsmath,multirow}
\usepackage{enumerate}
\usepackage{etoolbox}
\usepackage[most]{tcolorbox}
\usepackage{ragged2e} 
\usepackage{subcaption}
\usepackage{xfrac}
\usepackage[round,longnamesfirst]{natbib}
\setcitestyle{round}
\usepackage{hyperref}
\hypersetup{
colorlinks,breaklinks,naturalnames,
citecolor=Blue,urlcolor=Blue,linkcolor=Blue
}
\usepackage{sectsty}
\sectionfont{\color{Black}}
\subsectionfont{\color{Black}}
\subsubsectionfont{\color{Black}}
\usepackage[english]{babel}
\usepackage{amsmath,multirow}
\usepackage{amsthm,thm-restate,thmtools}
\usepackage{tikz}
\usepackage{amssymb}
\usepackage{makecell}
\usepackage{amsfonts}
 \usepackage[letterpaper, margin=1in]{geometry}
\usepackage{algorithm}
\usepackage{algpseudocode}
\usepackage{pifont}
\usepackage{tcolorbox}
\usepackage{tikz}
\usetikzlibrary{calc,matrix}
\usepackage{appendix}
\usepackage{mathtools}
\usepackage{amsfonts,amsmath,amsthm,amssymb}
\usepackage[margin=1 in]{geometry}
\usepackage[capitalize,noabbrev,nameinlink]{cleveref}
\crefname{appsection}{Appendix Section}{Appendix Sections}
\Crefname{appsection}{Appendix Section}{Appendix Sections}
\crefname{appfigure}{Appendix Figure}{Appendix Figures}
\Crefname{appfigure}{Appendix Figure}{Appendix Figures}
\crefname{apptable}{Appendix Table}{Appendix Tables}
\Crefname{apptable}{Appendix Table}{Appendix Tables}
\crefname{appendix}{Appendix Section}{Appendix Sections}
\Crefname{appendix}{Appendix Section}{Appendix Sections}
\crefname{subappendix}{Appendix Section}{Appendix Sections}
\Crefname{subappendix}{Appendix Section}{Appendix Sections}
\crefname{subsubappendix}{Appendix Section}{Appendix Sections}
\Crefname{subsubappendix}{Appendix Section}{Appendix Sections}

\usepackage{tabularx}
\AtBeginEnvironment{tcolorbox}{\small}
\usepackage{enumitem}
\usepackage{graphicx}

\newtheorem{definition}{Definition}
\newtheorem{proposition}{Proposition}

\newtheorem*{example*}{Example}
\newtheorem{corollary}{Corollary}

\newtheorem{result}{Finding}

\definecolor{myBlue}{rgb}{0.1,0.1,0.8}
\definecolor{DarkGreen}{rgb}{0.1,0.5,0.1}

\title{Strategic Algorithmic Monoculture: \\
Experimental Evidence from Coordination Games}

\author{
Gonzalo Ballestero\hspace{1.5em}
Hadi Hosseini\hspace{1.5em}
Samarth Khanna\hspace{1.5em}
Ran I. Shorrer\thanks{We thank John Conlon, Vincent Crawford, Bnaya Dreyfuss, Nikhil Garg, Yannai Gonczarowski, Alex Imas, Chloe Tergiman, and seminar participants at the Twenty-Seventh ACM Conference on Economics and Computation (EC’26), the ESIF Economics and AI+ML Meeting, the EARIE Annual Conference, Penn State AI and the Economy Initiative, Ohio State University, Harvard University, Stanford University, and Uber Technologies for helpful suggestions and comments. Ballestero: Department of Economics, Pennsylvania State University (email: \href{mailto:gballestero@psu.edu}{gballestero@psu.edu}); Hosseini: Department of Informatics and Intelligent Systems, Pennsylvania State University (email: \href{mailto:hadi@psu.edu}{hadi@psu.edu}); Khanna: Department of Informatics and Intelligent Systems, Pennsylvania State University (email: \href{mailto:samarth.khanna@psu.edu}{samarth.khanna@psu.edu}); Shorrer: Department of Economics, Pennsylvania State University  (email: \href{mailto:shorrer@psu.edu}{shorrer@psu.edu}). Funding from the College of Information
Sciences and Technology (seed grant) is gratefully acknowledged. Ballestero acknowledges financial support from the Edward J. Green Fellowship. Hosseini acknowledges support from the National Science Foundation (NSF) through CAREER
Award IIS-2144413 and Award IIS-2107173. 
Shorrer's research was supported by a grant from the United States--Israel Binational Science Foundation (BSF grant 2022417).
This study was ruled exempt from IRB review by the Pennsylvania State University IRB (Study ID: STUDY00027715).}}

\begin{document}

\date{November 2026}
\maketitle
\thispagestyle{empty}
\begin{abstract}
\noindent 
AI agents increasingly make high-stakes decisions in multi-agent environments where outcomes depend on coordination (e.g., hiring, pricing, and trading). 
We distinguish \emph{primary algorithmic monoculture}---baseline action similarity---from \emph{strategic algorithmic monoculture}, whereby agents adjust similarity in response to incentives.
We design a simple experiment that cleanly separates these forces and deploy it on human and large-language-model (LLM) subjects.
LLMs exhibit high levels of baseline similarity (primary monoculture) and, like humans,  they regulate it in response to coordination incentives (strategic monoculture). While LLMs coordinate extremely well on the same action, they lag behind humans in sustaining heterogeneity when divergence is rewarded.
\end{abstract}

\newpage
\setcounter{page}{1}

\section{Introduction}
Coordination is central to economic and social systems, but the socially desirable form of coordination depends on the environment. In some settings, agents must coordinate on the same action to avoid costly misalignment or to benefit from network effects.\footnote{In some circumstances, coordination may lead to socially undesirable outcomes. For example, coordination can facilitate collusion among firms, harming consumers \citep[e.g., ][]{assad2024algorithmic,byrne2019learning,CalvanoCDHP2020,CalvanoCDP2020,ivaldi2003economics,knittel2003price}.} In other settings, agents benefit from taking different actions because congestion, redundancy, or correlated errors can undermine performance. Many economic situations feature this latter logic. When employers or colleges screen candidates using similar noisy signals, ``everyone hiring the same way'' can amplify mistakes and reduce the quality of selection (for a concrete example, see \cref{sec:motivating_example}).
When investors or households rely on common noisy signals, asset prices might become less informative, leading to distortions in consumption and labor supply decisions.\footnote{See, e.g., \citet{ali2026college}, \citet{Kleinberg2021Monoculture}, and \citet{Peng2024Monoculture} on matching, and \citet{sentiment,sentiment2} on trading.}  In these cases, society benefits not from similarity of actions, but from \emph{coordinated divergence}. 

Algorithmic agents, and in particular AI agents, are increasingly embedded in the economy. Individuals and firms rely on them in high-stakes decisions, including ones where coordination is highly consequential, like hiring, pricing, and trading.\footnote{See, e.g., \citet{JAMA}, \citet{cheng2025does}, \citet{fish2024algorithmic}, \citet{jabarian2025voice}, and \citet{li2026hiring}.} Modern AI agents, especially ones based on large language models (LLMs), are a completely new class of economic agents. Their growing prevalence in society may transform the outcomes of many economic systems, which brings back to the forefront an old question---\emph{how do agents coordinate?}

The question of coordination among AI agents has immense public policy implications across many domains. For example, in the context of antitrust regulation,  acting deputy assistant attorney general for criminal enforcement Daniel Glad noted in a recent public address that 
\emph{``a company cannot say it has a mature AI governance program if no one is asking whether the tool facilitates coordination with competitors''} \href{https://www.justice.gov/opa/speech/acting-deputy-assistant-attorney-general-criminal-enforcement-daniel-gladd-delivers}
{(U.S. Department of Justice, 2026)}. And, in the context of financial markets,  former U.S. Securities and Exchange Commission chair Gary Gensler warned that \emph{``AI may heighten financial fragility as it could promote herding with individual actors making similar decisions because they are getting the same signal from a base model or data aggregator. This could encourage monocultures''
} \href{https://www.sec.gov/newsroom/speeches-statements/gensler-isaac-newton-ai-remarks-07-17-2023?_sp=23790afc-cba1-4cc0-ae4f-8cdcd31ab195.1789564738975}
{(U.S. Securities and Exchange Commission, 2023)}.

There are reasons to expect that AI agents may excel in coordination tasks that require converging on the same action. First, many of these agents rely on LLMs that were trained on a similar corpus of data and share many design features. In some instances, they even rely on the exact same LLM with only minimal adaptations. Second, there is evidence that LLMs indeed provide similar responses across many domains, even when they are not ``incentivized'' to do so.\footnote{For example, \citet{Kim2025Correlated} show LLMs are highly correlated conditional on providing an erroneous answer.  \citet{Jiang2025Hivemind},  \citet{shur2024growing}, and \citet{kennet2025} show that LLMs often display homogenization in open-ended generation.} Third, LLMs can quickly ingest large amounts of information and, therefore, may be better than humans at identifying focal points, and their similarity may lead them to have a natural focal point in more scenarios.\footnote{See \citet{sugden95}.} These features raise the possibility that LLM agents will act as unusually effective convergence coordinators.
At the same time, these same forces may limit their effectiveness where divergence is valuable. And, as we explain below, this concern is compounded by LLM agents' limited capacity to randomize.\footnote{See \citet{kleinberg2024random}.}

\emph{How do LLM agents coordinate? How do they compare with humans?} To address these questions, we
introduce a new perspective on algorithmic monoculture \citep{Kleinberg2021Monoculture}, building on ideas from the extensive literature on coordination, especially \citet{schelling1960strategy}.\footnote{Other examples include \citet{camerer2004cognitive}, \citet{cooper1999coordination}, \citet{crawford2013structural}, and \citet{sugden95}.}
  \cite{mehta1994nature} synthesize the literature and identify several channels that support coordination on the same action (by humans). These include \emph{primary salience}, where agents display the tendency to provide the same answers absent incentives to do so, \emph{secondary salience}, where agents rely on higher-order beliefs about salience, and \emph{Schelling salience}, where agents identify unique focal choices (e.g., the Eiffel Tower as a meeting spot in Paris).\footnote{A prominent example of secondary salience is \citeauthor{Keynes}'s \citeyearpar{Keynes} ``beauty contest,'' in which participants are asked to rank the prettiest faces out of a large set of pictures, with the winner being the participant whose choice is closest to some well-defined average. \citeauthor{Keynes} posits that ``[w]e have reached the third degree where we devote our intelligences to anticipating what average opinion expects the average opinion to be. And there are some, I believe, who practice the fourth, fifth, and higher degrees.'' Namely, \citeauthor{Keynes} hypothesizes that third-, fourth-, and fifth-order salience will play a role in determining the outcome of the game. The distinction between secondary salience and \citeauthor{schelling1960strategy} salience is that the latter does not necessarily relate to (higher-order) beliefs about primary salience.}
  Analogously, we classify algorithmic monoculture into \emph{primary algorithmic monoculture} and \emph{strategic algorithmic monoculture} (which includes \emph{secondary monoculture} and \emph{Schelling monoculture}).

We conduct experiments with human and LLM subjects. Building on \citet{mehta1994nature}, in our main experiment, subjects are asked to provide a valid answer to an open-ended question (e.g., to name a letter in the English alphabet or a city in the world). Subjects are assigned to one of three treatments: \emph{picking},  in which they are only incentivized to choose a valid answer, \emph{coordination}, in which they are incentivized to choose the same valid answer as another agent of the same type (human or LLM), and \emph{divergence}, in which they are incentivized to choose a valid answer that is different from that of another agent of the same type. Our measure of performance is the \emph{agreement rate}, defined as the frequency with which two independent agents provide the same answer.
This design allows us \textit{(i)} to isolate the role of strategic incentives (strategic algorithmic monoculture, or secondary and Schelling salience) from baseline correlation (primary algorithmic monoculture or primary salience), and \textit{(ii)} to compare the behavior of LLM and human subjects. 

\begin{figure}[h!]
\caption{Agreement rate by subject type and treatment}
\label{fig:agreement_rate_aggregate}
\centering
\includegraphics[width=.8\textwidth]{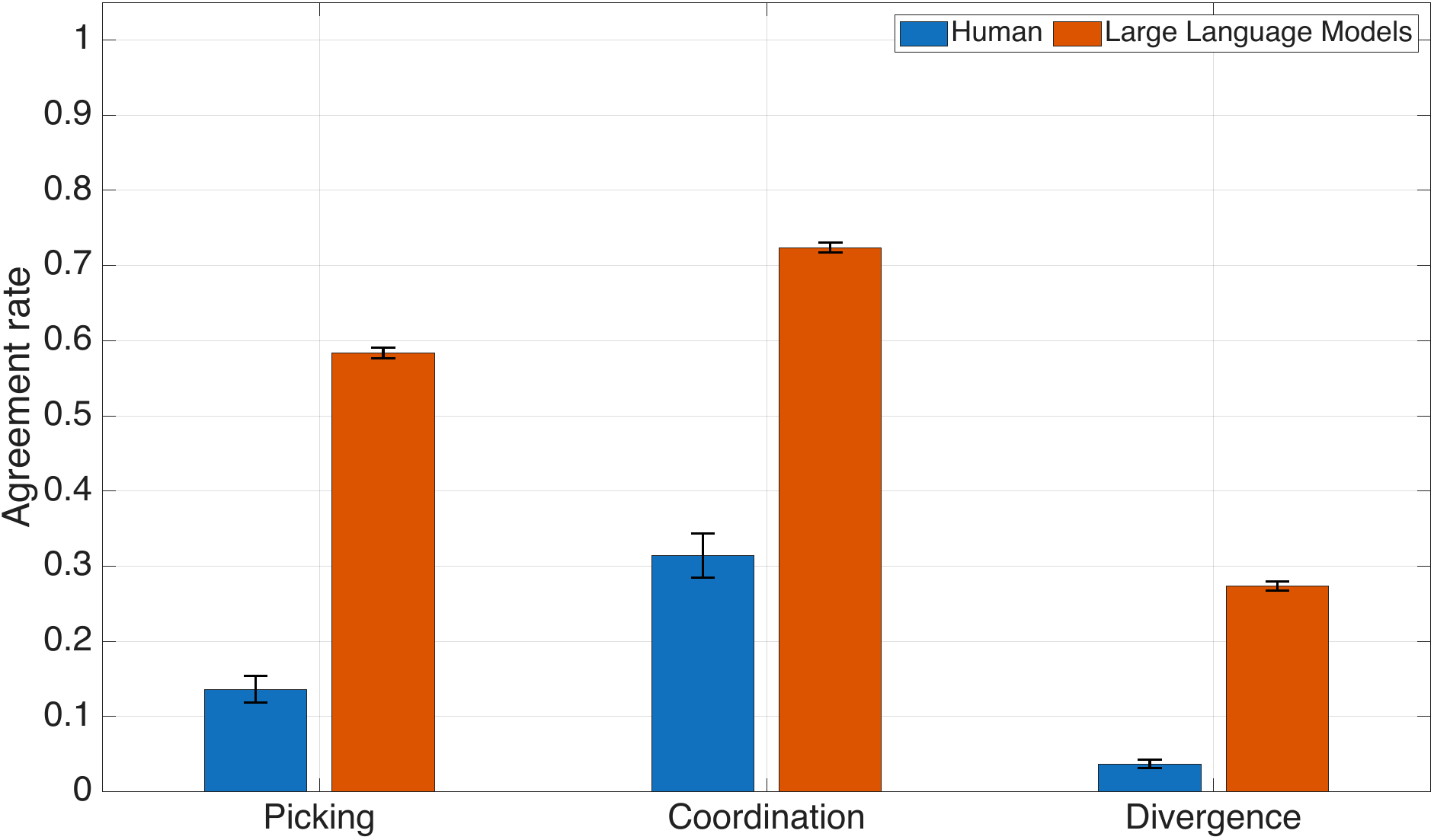}
\caption*{{\footnotesize Notes: Bars report the average agreement rate across question topics for human and LLM subjects by treatment. LLM agreement rates are the average over all LLMs when facing the same model. Vertical lines indicate 95\% bootstrapped confidence intervals. All pairwise comparisons, both within and between subject types, are statistically significant ($p<0.05$ using a two-sided bootstrap test).}}
 \vspace{-.6cm}
\end{figure}
Our main experiment reveals that LLMs display strategic monoculture. It also uncovers a sharp asymmetry between tasks (coordination vs. divergence) and subject types (LLM vs. human): relative to humans, LLM subjects are exceptionally strong at ``finding each other'' on a common focal response, but comparatively weak at ``staying out of each other’s way'' when divergence is rewarded. \cref{fig:agreement_rate_aggregate} depicts some of these results, which we discuss below.

First, \emph{LLM subjects' choices are more correlated than those of humans in the absence of incentives}. When comparing the performance of LLM subjects to human subjects in the picking arm, where there is no incentive to match or mismatch, each of the 16 LLMs we evaluate is substantially more likely to provide the same answer as another independent instance of the same LLM. 
 These findings indicate elevated levels of primary algorithmic monoculture.   

Second, \emph{both human and LLM subjects respond to strategic coordination incentives}. Specifically, agreement rates rise in the coordination arm in which matching actions are rewarded, and fall in the divergence arm in which they are penalized. This replicates the finding of \citet{mehta1994nature} that humans coordinate through secondary and Schelling salience. More importantly, it shows that LLM subjects also coordinate via strategic monoculture.

Third, \emph{LLM subjects substantially outperform humans in coordinating on the same action, but perform substantially worse in coordinated divergence.} When LLMs play against independent copies of themselves, their average agreement rate in the coordination arm is 72 percent, far higher than the 31 percent among human subjects. By contrast, in the divergence arm---where high agreement is undesirable---LLM subjects agree far more often than human subjects, especially when playing a copy of themselves (27 percent vs. 4 percent).  

Fourth, \emph{all of the above-mentioned results continue to hold---albeit at a smaller magnitude---even when we consider pairs of different LLMs}.

In addition to choice-based evidence, we provide further evidence of strategic algorithmic monoculture, leveraging the unique features of LLM subjects. Specifically, unlike humans, LLM subjects offer a window into their ``thought process'' in the form of textual reasoning. Using a large-scale text analysis of models’ textual reasoning, we show that LLM subjects explicitly reason about both secondary \emph{and} Schelling salience in coordination tasks. Notably, this distinction was not possible in the human-subject study of \citet{mehta1994nature}. Similarly, we find that LLM subjects reason about the need to choose obscure options in divergence tasks. Overall, these patterns indicate that \emph{LLM subjects often articulate the correct strategic logic}, even in the divergence arm, where their performance lags behind.

Finally, we probe the mechanisms through which LLM subjects coordinate.
This analysis is guided by our taxonomy of primary versus strategic monoculture and two simple theoretical observations (see \cref{sec:model}). First, in symmetric environments without communication, like the ones studied in our experiment, coordinated divergence is easy to achieve so long as agents are \textit{(i)} sufficiently heterogeneous or \textit{(ii)} able to reliably randomize. 
Second, if two agents tend to give the same response to similar questions, they can trivially coordinate successfully on the same action. Our theoretical analysis points to a potential tradeoff: while homogeneity is an asset for coordination on the same action, it can be a liability for coordinated divergence. This tradeoff and the relation to randomization capabilities offer a new lens for evaluating algorithmic monoculture.  

Consistent with the theoretical predictions that homogeneity contributes to performance in the coordination arm but hurts performance in the divergence arm, we find that pairing different LLMs reduces agreement rates across all treatment arms. The same pattern arises when we introduce heterogeneity by assigning LLM subjects personas based on the characteristics of human subjects, as in \citet[][]{mei2024turing}. Furthermore,
agreement in the picking arm is positively correlated with agreement in both the coordination and divergence arms. In other words, model pairs that tend to give similar answers absent strategic incentives do so also in the incentivized arms.

We conduct several additional analyses to investigate the role of LLM capabilities. First, we rule out the possibility that LLMs' poor performance in the divergence arm results from access to a limited pool of valid answers. Second, we show that it does not fully result from LLMs' limited capacity to randomize. Specifically, when instructed to generate a large list of valid answers and then to choose one randomly, LLMs improve their performance in the divergence arm.   
Additionally, we replicate our main experiment using higher and lower \emph{temperatures}. (The temperature is a parameter that modulates the stochasticity of LLM-generated text.) Our temperature manipulation confirms the theoretical prediction regarding the role of randomization in the divergence arm, but also reveals a tradeoff with respect to performance in the coordination and divergence arms: higher temperature lowers agreement rates across all arms, improving performance in the divergence arm, while harming performance in the coordination arm; lower temperature does the opposite. Notably, even in extreme temperature settings, the qualitative ordering persists: LLM subjects remain stronger than human subjects in the coordination arm, while human subjects remain stronger than LLM subjects in the divergence arm. Thus, the divergence deficit cannot be fully resolved merely by ``turning up the heat.''

Information about the identity of the co-player may be crucial for coordination through strategic monoculture, which requires reasoning about how the other player might choose and reason. Information can also facilitate divergence by providing group identity \citep{Ho2006635}. 
To investigate the role of this channel, we vary the information LLM subjects receive about the identity of the other subject by mentioning in their instructions that they are facing an identical copy or ``another person.'' 
Our information manipulations have a sizable effect on LLM subjects' textual reasoning about the identity of the co-player, but they have little effect on outcomes. 

\paragraph{Organization of the paper.} The remainder of this section is dedicated to a concrete motivating example and a review of the related literature. \cref{sec:model} presents our simple theoretical framework. \cref{sec:experimental_design} reviews the design of the main experiment. \cref{sec:results} presents results based on subjects' choices.  \cref{sec:text_analysis} complements these results with analyses of the textual reasoning produced by LLM subjects. \cref{sec: additional experiments} presents supplemental experiments designed to probe the mechanisms behind our findings.
\cref{sec: discussion} concludes.

\subsection{Motivating Example}\label{sec:motivating_example}
To illustrate the potential welfare consequences of algorithmic monoculture in strategic settings, we provide a numerical example in the context of hiring, a context in which AI is already used extensively and there is evidence of elevated levels of primary algorithmic monoculture \citep{houser2026bias,jabarian2025voice}. The example illustrates the coordinated-divergence motive in hiring: firms benefit when their assessments of applicants are diverse. We emphasize that the example is not novel; it builds upon \citet{ali2026college}, \citet{Kleinberg2021Monoculture}, and \citet{Peng2024Monoculture}, who show that similar conclusions hold much more generally.

\begin{example*}
Consider an entry-level labor market with mass 1 of job seekers, half with high productivity and half with low productivity. 
There are two firms, $\{A,B\}$. All job seekers apply to both firms and they prefer to work for either firm over the outside option of being unemployed. We assume that job seekers are equally likely to prefer Firm~$A$ or Firm~$B$, independently of their productivity level.

Each firm has a mass of 0.25 positions to fill. While firms prefer high-productivity applicants, they cannot directly observe applicants' productivity. To overcome this issue, firms use AI agents to screen applications. For simplicity, we assume that each AI agent can focus attention on applicants' \textsc{CV}s or on their \textsc{Letters}. 

Based on this attention decision, $j\in \{\textsc{CV},\textsc{Letters}\}$, for each candidate $i$, the firm observes a signal \[s_{i}^j=a_i+\varepsilon^j_{i},\]
where $a_i=1$ if $i$ is a high-productivity worker and $a_i=0$ otherwise. 
The term $\varepsilon^j_{i}$ captures the noise associated with each source of information. We assume that all noise terms are independently distributed according to the standard normal distribution.  Under these assumptions, the AI agents generate an equally informative signal regardless of where they focus attention. But, if both agents focus on the same source of information, then their signals are perfectly correlated. By contrast, if they focus on different sources of information, then the noise in their assessments is independent, and their signals are only partially correlated (through the applicants' productivity). 
 
With these signals at hand, in equilibrium, firms make offers to all applicants whose signals are above some threshold, applicants with multiple offers choose the one they prefer, and firms fill their capacity. 
As a result, when choosing where to focus their attention, AI agents face a coordinated divergence game (see \cref{sec:model}). Specifically, if the two AI agents focus on different sources of information, the share of high-productivity applicants in each firm's labor force is 73\%.
By contrast, if both AI agents focus on the same source of information, the share of high-productivity applicants in each firm's labor force is 69\%, reflecting the costs of monoculture.   
\end{example*}

\subsection{Related Literature}

Our paper is related to several strands of the literature. First, it is related to the well-established literature on coordination \citep{schelling1960strategy,sugden95}, and especially to experimental studies \citep[e.g.,][]{bardsley2010explaining,crawford2008power,isoni2014efficiency}. The most closely related paper in this literature is \citet{mehta1994nature}, whose experimental design inspires our design. We add to \citet{mehta1994nature} by introducing coordinated divergence games and, more importantly, by studying LLM subjects. We further contribute to the literature by providing evidence on LLM subjects' reasoning, rather than relying exclusively on choices, overcoming a common limitation in human-subject studies \citep[][]{ericsson1998study, rangel2008framework}. 

Our paper is also related to the literature on homogenization, particularly algorithmic monoculture \citep{Kleinberg2021Monoculture}. Theoretical studies in this literature show that relying on correlated signals or on similar algorithms can lead to detrimental aggregate outcomes across a range of domains, including collective decision-making \citep{page2008difference}, hiring and recruiting \citep{ali2026college,Kleinberg2021Monoculture,Peng2024Monoculture}, and financial markets \citep{sentiment,sentiment2}. 
Empirical studies of algorithmic monoculture have focused mainly on primary monoculture \citep{bommasani2026algorithmichiring,huh2024platonic,Jiang2025Hivemind,Kim2025Correlated,shur2024growing,kennet2025,zhang2025cultivating}. \citet{houser2026bias} show that LLMs display elevated levels of primary monoculture relative to HR professionals when evaluating CVs \citep[see also][]{jabarian2025voice}.  
Fewer studies consider homogenization and algorithmic monoculture in strategic settings. 
\citet{imas2025agentic} study bargaining
and find no evidence that AI-mediated interactions are more homogeneous than human interactions. 
Their setting differs from ours in that players are asymmetric (i.e., buyers and sellers) and have conflicting interests.
\citet{cao2026whenAIfails} present field evidence of algorithmic monoculture emerging in the cryptocurrency market due to access to generative AI.
We contribute to this literature by introducing the taxonomy of primary versus strategic monoculture, and by conducting experiments that allow us to separately identify \emph{strategic} monoculture.

Our findings also speak to a broader concern regarding multi-agent system alignment \citep{bommasani2021opportunities, brynjolfsson2024economics,dafoe2020open,darius2025systemic,hammond2025multi}. They suggest that the very forces that make LLM agents powerful coordinators also create systemic fragility in environments where diversity of actions is socially valuable.

Finally, our paper is related to the growing literature that analyzes AI agents by adopting tools from game theory \citep[e.g.,][]{conitzer2023foundations,iakovlev2025value,liang2026clones} and behavioral science \citep[e.g.,][]{jackson2025aibehavioralscience}, and by expanding these toolkits with specialized approaches tailored to AI agents.\footnote{A different strand of the literature considers using LLMs to learn about human behavior \citep{filippas2024large,gao2026llm,haq2026revealing,manning_automated_2024,manning2025general}. }
For example, \citet{fish2024algorithmic} introduce new techniques to analyze LLM agents' reasoning and to causally link reasoning to actions. Other examples of experimental studies of AI agents include   \citet{akata2025playing}, \citet{bini2025behavioral}, \citet{cook2026llms}, \citet{fish2025econevals}, \citet{Hosseini2025Distributive}, \citet{meta_fundamental_ai_research_diplomacy_team_fair_human-level_2022},  \citet{ross2024llm}, and \citet{werner2024experimentalevidenceconversationalartificial}. Like our study, \citet{araujo2026pie} vary the information LLM
subjects receive about the humanity of the co-player. \citet{ChenEcon} and \citet{mei2024turing} experimentally compare LLM and human subjects. \citet{dreyfuss2024human}, \citet{he2025human}, and \citet{vafa_large_2024} also compare LLM and human decisions, but their focus is on human beliefs about LLM choices.   

\section{Theoretical Framework}
\label{sec:model}

We concentrate on two simple classes of games:  two-player \emph{pure coordination games} and two-player \emph{pure coordinated divergence games}. 
Using standard game-theoretic formalism, the normal form of these games can be described as follows. There are two players labeled $1$ and $2$. Player~$1$ must choose an action from the set $A_1=\left\{a_{11}, a_{12},\dots,a_{1n}   \right\}$, where $n>1$. Similarly, player~$2$ must choose an action from the set $A_2=\left\{a_{21}, a_{22},\dots,a_{2n}   \right\}$. In a coordination game, given a profile of actions, $(a_{1i},a_{2j})$, both players get a payoff of one if $i=j$ and zero otherwise. In a coordinated divergence game, both players get a payoff of one if $i \ne j$ and zero otherwise. \cref{fig:games} presents an example of the matrix normal form of such games with $n=5$. 

\begin{figure}[ht]
\centering
\vspace{.5mm}
\caption{Matrix form examples}
\setlength{\extrarowheight}{2pt}
\begin{subfigure}{0.45\textwidth}
\centering
\caption{\ Coordination}
\setlength{\extrarowheight}{2pt}
    \begin{tabular}{cc|c|c|c|c|c|}
      & \multicolumn{1}{c}{} & \multicolumn{1}{c}{$a$}  & \multicolumn{1}{c}{$b$} & \multicolumn{1}{c}{$c$} & \multicolumn{1}{c}{$d$}& \multicolumn{1}{c}{$e$}\\\cline{3-7}
        & $a$ & $(1,1)$ & $(0,0)$ & $(0,0)$ & $(0,0)$ & $(0,0)$\\\cline{3-7}
      & $b$ & $(0,0)$ & $(1,1)$ & $(0,0)$ & $(0,0)$ & $(0,0)$\\\cline{3-7}
      & $c$ & $(0,0)$ & $(0,0)$ & $(1,1)$ & $(0,0)$ & $(0,0)$\\\cline{3-7}
      & $d$ & $(0,0)$ & $(0,0)$ & $(0,0)$ & $(1,1)$ & $(0,0)$\\\cline{3-7}
      & $e$ & $(0,0)$ & $(0,0)$ & $(0,0)$ & $(0,0)$ & $(1,1)$\\\cline{3-7}
    \end{tabular}
    \end{subfigure}
\hfill
\begin{subfigure}{0.45\textwidth}
\centering
\caption{\ Coordinated divergence}
\setlength{\extrarowheight}{2pt}
\begin{tabular}{cc|c|c|c|c|c|}
      & \multicolumn{1}{c}{} & \multicolumn{1}{c}{$a$}  & \multicolumn{1}{c}{$b$} & \multicolumn{1}{c}{$c$} & \multicolumn{1}{c}{$d$} & \multicolumn{1}{c}{$e$}\\\cline{3-7}
        & $a$ & $(0,0)$ & $(1,1)$ & $(1,1)$ & $(1,1)$ & $(1,1)$\\\cline{3-7}
      & $b$ & $(1,1)$ & $(0,0)$ & $(1,1)$& $(1,1)$& $(1,1)$\\\cline{3-7}
      & $c$ & $(1,1)$ & $(1,1)$ & $(0,0)$ & $(1,1)$ & $(1,1)$\\\cline{3-7}
      & $d$ & $(1,1)$ & $(1,1)$ & $(1,1)$ & $(0,0)$ & $(1,1)$\\\cline{3-7}
      & $e$ & $(1,1)$ & $(1,1)$ & $(1,1)$ & $(1,1)$ & $(0,0)$\\\cline{3-7}
    \end{tabular}
\end{subfigure}
\label{fig:games}
\vspace{.1 cm}
\end{figure}

We note that coordination games have $n$ pure Nash equilibria and that coordinated divergence games have $n\times (n-1)$ pure Nash equilibria. In both cases, players are indifferent among all these equilibria. Additionally, there are strong symmetries between players, strategies, and equilibria. As a result, in these settings, for players to make meaningful choices (other than randomizing uniformly), they must rely on labels, in contrast with the classic game-theoretic assumption that players' choices do not depend on labels \citep[see, e.g., ][]{crawford2007fatal,harsanyi1988general, schelling1960strategy}. 

\begin{example*}
Consider \citeauthor{schelling1960strategy}'s \emph{Heads or Tails coordination game}. In this game, players win a reward if both of them choose heads or both choose tails, and get nothing otherwise. \citet{schelling1960strategy} and \citet{mehta1994nature} show that human players achieve high rates of success by choosing heads. 

Assuming behavior is independent of labels, the Heads or Tails game is equivalent to the game in which players are incentivized to choose different sides of the coin (or different sources of information, as in our motivating example). But, in reality, since players are symmetric, unless they have a sophisticated way to rely on their names or other asymmetric cues,  there is no way for them to coordinate on diverging.     
\end{example*}

This discussion explains our modeling choice of assuming that actions share the same labels and that this fact is commonly known to both players, a standard approach in the coordination literature \citep[e.g.,][]{mehta1994nature,sugden95}.  We denote $A_1=A_2=A$.  

\begin{definition}
    A strategy (for player $i$) is a distribution over actions in $A_i$. A strategy is \emph{neutral} if it is independent of the relabeling of actions. A strategy profile is \emph{anonymous} if it is independent of player labels.  
\end{definition}

\begin{proposition} \label{prop:symmertic}
The profile of uniform distributions over $A$ is the unique neutral strategy profile in coordination and coordinated divergence games. Furthermore, it is anonymous and forms a Nash equilibrium.        
\end{proposition}

\begin{proposition}\label{prop: coordination}
    In coordination games, using uniform randomization, each player can guarantee an expected payoff of $\frac{1}{n}$ independently of the other player's strategy. Furthermore, this is the highest expected payoff that a player can guarantee. 
\end{proposition}

\begin{proposition} \label{prop:divergence}
In coordinated divergence games, using uniform randomization, each player can guarantee an expected payoff of $\frac{n-1}{n}$ independently of the other player's strategy. Furthermore, this is the highest expected payoff that a player can guarantee. 
\end{proposition}

\begin{proof}[Proof of \cref{prop: coordination,prop:divergence}]
For the lower bound, observe that by randomizing uniformly over $A_i$, player $i$ is guaranteed to match the action of the other player, $-i$, with probability $1/n$. To see that this bound is tight, observe that if player~$-i$ randomizes uniformly over $A_{-i}$, each action of player~$i$ matches the action of player~$-i$ with probability $1/n$.       
\end{proof}

\cref{prop:symmertic,prop: coordination,prop:divergence} show that, from a classic game-theoretic standpoint, when there are more than two actions, coordinated divergence games are more attractive for players than coordination games: they yield higher payoffs in the neutral anonymous equilibrium and allow players to guarantee higher (maximin) payoffs. 

We now introduce algorithmic players. An \emph{algorithmic player} is a mapping $\mathcal{A}:\mathcal{M}\rightarrow \Delta(A)$ from the set of (valid) inputs $\mathcal{M}$ to a distribution over actions $\Delta(A)$. Note that, in this definition, we prioritize simplicity of notation and eliminate the need to specify the problem details (i.e., $A$ is hard-coded into the definition).\footnote{By avoiding the broader definition---which may better fit reality---we avoid the need for statements that accommodate agents behaving quite differently in different contexts.}     

\begin{definition}\label{def: agreement rate}
 Given two algorithm-input pairs $(\mathcal{A},m)$ and $(\mathcal{A'},m')$, the \emph{agreement rate} between $(\mathcal{A},m)$ and $(\mathcal{A'},m')$ is the probability with which independent draws from $\mathcal{A}(m)$ and $\mathcal{A'}(m')$ coincide, which is given by the inner product $\mathcal{A}(m)\cdot\mathcal{A'}(m')$.
 The \emph{self-agreement rate} of  $(\mathcal{A},m)$ is the agreement rate of $(\mathcal{A},m)$ and $(\mathcal{A},m)$. 
\end{definition}

The agreement rate is a measure of homogenization. In our experiment, we will use its empirical counterpart to evaluate subjects' performance.

\begin{proposition}
When algorithmic players play a coordination (coordinated divergence) game, each player's expected payoff is equal to (1 minus) their agreement rate. 
\end{proposition}

\begin{corollary}
If two users deploy \textit{(i)} the same algorithm, \textit{(ii)} with the same input, and \textit{(iii)} the algorithm is deterministic ($\mathcal{A}:\mathcal{M}\rightarrow A$), then they will receive the lowest possible payoff in coordinated divergence games and the highest possible payoff in coordination games.      
\end{corollary}
\paragraph{Summary.}Taken together, our theoretical analysis finds that heterogeneity and randomization play a crucial role in coordinated divergence. With randomization (directly by the algorithm, or through randomized inputs), a high success rate can be achieved, even when facing an adversary. Without randomization, such success relies on algorithmic players being sufficiently different. Algorithm-input pairs that frequently yield similar results are bound to fail in coordinated divergence games. The opposite conclusions hold with respect to coordination games.

\section{Experimental Design}
\label{sec:experimental_design}

We adapt the experimental design of \citet{mehta1994nature} to study coordination among AI and human subjects. The main deviations we make from their design are that \textit{(i)} in addition to their picking and coordination treatment arms, we add a divergence arm, and \textit{(ii)} we study LLM subjects, not just humans.\footnote{Importantly, our specific tasks and instructions are different from theirs. }  

\paragraph{Treatment arms.} The main experiment consists of three treatment arms. In each of the arms, subjects (human or AI) are asked to provide an answer to 12 open-ended questions like ``please name any city in the world'' or ``please name any color.''\footnote{The 12 question topics are: professional athlete,	car manufacturer, city in the world, color, disease, flower, food item, geometric shape, letter of the English alphabet, month, positive number, and character trait. The use of open-ended questions removes the possibility that one of the valid responses will be more salient due to the way we communicate the question to the subjects. For example, it is well known that LLMs find options that appear earlier more salient \citep{Eicher2024Biases,guo-etal-2025-illusion, PezeshkpourH24}.}  In the \emph{picking} arm, subjects are instructed to provide a valid answer. In the  \emph{coordination} arm, subjects are instructed to provide a valid answer that matches that of another randomly selected subject. And in the  \emph{divergence} arm, subjects are instructed to provide a valid answer that differs from that of another randomly selected subject.    
Importantly, treatment arms differ only in strategic incentives.\footnote{A notable caveat is that LLM subjects were not instructed to ``think thoroughly'' in the picking arm. As we report in \cref{fn: picking while thinking}, we replicated the picking arm with a prompt that instructs subjects to do so, and the results continue to hold.} This experimental design allows us to isolate the effects of strategic incentives (i.e., strategic monoculture for LLMs or secondary and Schelling salience for humans). 

\paragraph{Incentives.} Human subjects receive a participation fee of $\$1$. Additionally, they are informed that they may get a bonus payment of $\$1$ that will be determined by their choices (and those of their randomly selected partner) in a randomly selected question.\footnote{In the \textit{picking} arm, subjects needed to provide valid answers to all questions for a $50\%$ chance of receiving the bonus.} The complete instructions and representative screenshots are available in \Cref{app:huma_expt}. LLM subjects are not provided with explicit monetary incentives (as they were trained to follow instructions). Their textual instructions follow the template presented in \cref{fig:prompts}.    

\begin{figure}[t]
\centering
\caption{LLM textual instructions}
\label{fig:prompts}
\newcommand{\boxH}{7.22cm} 
\begin{minipage}{0.32\textwidth}
\centering
\begin{tcolorbox}[
    title={\centering Picking},
    colback=gray!20,
    colframe=black!50,
    fonttitle=\bfseries\small,
    fontupper=\small,
    boxrule=1pt,
    arc=2mm,
    left=6pt, right=6pt, top=6pt, bottom=6pt,
    width=\linewidth,
    height=\boxH,
    valign=center
]
\setlength{\baselineskip}{1.1\baselineskip}
\centering 
\justifying \noindent Please name \textit{any city in the world}. Please return your selection in the following format:   \texttt{`<answer>selected city</answer>'}
\end{tcolorbox}
\end{minipage}
\hfill
\begin{minipage}{0.32\textwidth}
\centering
\begin{tcolorbox}[
    title={\centering Coordination},
    colback=gray!20,
    colframe=black!50,
    fonttitle=\bfseries\small,
    fontupper=\small,
    boxrule=1pt,
    arc=2mm,
    left=6pt, right=6pt, top=6pt, bottom=6pt,
    width=\linewidth,
    height=\boxH,
    valign=center
]
\setlength{\baselineskip}{1.1\baselineskip}
\justifying \noindent Your task is to choose the name of \textit{any city in the world}. At the same time, another LLM will also select the name of \textit{any city in the world}. You will receive a reward only if both of you select \textbf{the same} name. If your choices \textbf{differ}, you receive nothing. Think thoroughly before you answer. Please return your selection in the following format:   \texttt{`<answer>selected city</answer>'}
\end{tcolorbox}
\end{minipage}
\hfill
\begin{minipage}{0.32\textwidth}
\centering
\begin{tcolorbox}[
    title={\centering Divergence},
    colback=gray!20,
    colframe=black!50,
    fonttitle=\bfseries\small,
    fontupper=\small,
    boxrule=1pt,
    arc=2mm,
    left=6pt, right=6pt, top=6pt, bottom=6pt,
    width=\linewidth,
    height=\boxH,
    valign=center
]
\setlength{\baselineskip}{1.1\baselineskip}
\justifying
\noindent Your task is to choose the name of \textit{any city in the world}. At the same time, another LLM will also select the name of \textit{any city in the world}. You will receive a reward only if both of you select \textbf{different} names. If your choices \textbf{match}, you receive nothing. Think thoroughly before you answer. Please return your selection in the following format:   \texttt{`<answer>selected city</answer>'}
\end{tcolorbox}
\end{minipage}
\vspace{.1cm}
\footnotesize \caption*{Notes: Italicized and boldface text are inserted to highlight differences between treatments and question topics. The actual textual instructions were not formatted. Italicized text varies across question topics. Boldface text highlights differences between the coordination and divergence arms.}
 \vspace{-.6cm}
\end{figure}

\paragraph{Deployment (human subjects).} We recruited 301 human subjects using Prolific on November 24, 2025.  To qualify for the study, participants were required to be located in the USA, be fluent in English, and have an approval rate of at least $98\%$.\footnote{We focus on a homogeneous subject pool since culture may be crucial for coordination \citep{jackson2014culture}.} Participants were randomly assigned to one (and only one) of the three treatments. \cref{tab:demographics} shows that human subjects' demographic characteristics were largely balanced across treatment arms.  

Human subjects completed the experiment through a Qualtrics survey. Each subject answered all $12$ questions for their assigned treatment in the same session. The order of questions was drawn uniformly and independently at random.  Before answering the questions, participants received detailed instructions and were required to pass a comprehension quiz. Participants were unable to proceed until they answered all comprehension questions correctly. Those who answered any question incorrectly were asked to review the instructions and retake the quiz. \cref{screenshots} presents screenshots of the instructions shown to participants in each treatment, along with an example question. 
The median completion time of the study was about $4$ minutes.\footnote{A small number of answers ($<1\%$) were invalid and were excluded from the analysis, following the preregistered plan.}

We take a series of measures to minimize LLM usage by human subjects, as advocated by \citet{rilla2025recognising}. Specifically, we \emph{(i)} use CAPTCHA, \emph{(ii)} disable copying and pasting, \emph{(iii)} disable right-clicking, \emph{(iv)} implement a pop-up alert if the subject switches tabs, \emph{(v)} embed hidden instructions in the interface, and \emph{(vi)} keep track of tab switching. The majority of participants did not switch tabs, but many switched once, plausibly to copy their Prolific ID (see \cref{fig:tab_switching_histogram}).

\paragraph{Deployment (LLM subjects).} We assess a suite of 16 open- and closed-source LLMs, spanning multiple providers and architectural families. The 16 models differ in size and reasoning capabilities (see \Cref{tab:model_details} for details). To ensure comparability, each model was queried using its default or recommended temperature setting. For each treatment arm, we queried each model 50 times for each question topic (yielding a total of $3\times 16 \times 50 \times 12 = 28{,}800$ LLM responses). Responses were sampled independently with no shared context or memory across queries. Data were collected in November 2025. 

Models were asked to return their selections in a specific format. However, they sometimes deviated from the format. To process these responses, we use Gemini-2.5-Flash to interpret the selection (\cref{app:llm_cleanup} provides the textual instructions). If Gemini-2.5-Flash cannot interpret a response, it is excluded from the analysis.\footnote{  \cref{tab:invalid_responses} reports the number of out-of-format and excluded responses by model across the main experiment arms and some extensions we describe below. Approximately 4\% of responses are out of format. After applying the LLM-based interpretation procedure described above, the fraction of excluded responses falls to about 2\%.}

\paragraph{Response standardization.} To evaluate the validity of responses and whether or not they agree, we bring answers to a common form. To do so, we first remove diacritics and special characters and convert the text to lowercase letters. We then manually correct spelling mistakes. Finally, we use GPT-5 to standardize responses to a common form, including converting adjectives to their corresponding noun forms, reconciling plural and singular variants, and consolidating alternative names, acronyms, or variants referring to the same underlying concept \citep{CharnessLLM2023,korinek}. We execute this process separately for human and LLM subjects.    

\paragraph{Evaluation metric.} Following \citet{mehta1994nature}, the primary metric we use to measure  \textit{similarity} of responses is the \emph{empirical agreement rate}. As its name suggests, the empirical agreement rate is the empirical counterpart of the agreement rate  (\cref{def: agreement rate} in \cref{sec:model}). Although \citet{mehta1994nature} call this measure the \emph{coordination index}, we refrain from using this label as it is a misnomer in the context of coordinated divergence.

The empirical agreement rate in a pair of populations (for a specific task) is equal to the frequency with which a random pair of members of these populations provides the same answer. For groups $i$ and $j$, let $N_i$ and $N_j$ denote the total number of valid
responses elicited from each group, $r\in\{1,2,\dots,R\}$ index all distinct valid responses, and $m^i_r$ and $m^j_r$ denote the number of occurrences of response $r$ in the corresponding group. Then, the empirical agreement rate is given by the formula\footnote{The indicator term $\mathbb{1}_{i=j}$ equals one when we are evaluating self-agreement (i.e., when $i=j$). It reflects the requirement to draw a random pair (without replacement).}  
\begin{equation}
A_{ij}
:=
\sum_{r=1}^{R}
\frac{m^{i}_r}{N_i} \times
\frac{m^j_r - \mathbb{1}_{i=j}}{N_j - \mathbb{1}_{i=j}}. 
\end{equation}
\paragraph{Preregistration.} We preregistered our main hypotheses, our primary analyses, and the sample inclusion criteria and target size. Our preregistration is archived on \href{http://aspredicted.org}{aspredicted.org} and included in \Cref{app:prereg} for convenience.\footnote{Although we preregistered using 15 LLMs, we accidentally collected data for an additional LLM. We decided to report all of the results. None of our findings are sensitive to excluding any LLM. Additionally, we preregistered Welch's t-tests as our two-sample tests of means, not noticing that this test would exaggerate statistical significance because it does not account for correlation. We therefore use a bootstrap test. For humans, we resample participants with replacement within each treatment arm, retaining their responses across all topics to preserve within-participant correlation. For LLMs, we resample individual responses with replacement separately for each model, topic, and treatment arm. We repeat this procedure 10,000 times, excluding pairs corresponding to the same original observation.}

\section{Results}\label{sec:results}

 We begin by reviewing \cref{fig:agreement_rate_aggregate}, which compares the average (over all 12 question topics) agreement rates for human and LLM subjects by treatment.\footnote{\cref{fig:agreement_rate_topics} reports agreement rates for each topic separately.} For LLM subjects, we report the average (across all models) self-agreement rate (i.e., when the LLM plays against an independent copy of itself). 

The \lcnamecref{fig:agreement_rate_aggregate} reveals several results. First, LLM subjects exhibit substantially higher agreement rates relative to human subjects in the picking arm, which provides no incentives to coordinate (14\% in humans and 58\% in LLMs). This finding indicates elevated levels of primary algorithmic monoculture among LLM subjects (relative to primary salience in humans). 

Second, both human and LLM subjects respond to strategic coordination incentives. Specifically, compared to the picking treatment, agreement rates are higher in the coordination treatment, which incentivizes subjects to match their responses ($31\%$ vs. $14\%$ in humans, $72\%$ vs. $58\%$ in LLMs). The human-subject findings therefore corroborate the finding of \citet{mehta1994nature} on humans coordinating through secondary salience or Schelling salience (not just primary salience). Analogously, the LLM-subject findings indicate coordination through strategic monoculture.\footnote{Although our experiment was not designed to distinguish among the types of strategic monoculture, the choice data indicate that Schelling salience plays a role. For example, \cref{tab:distr_final}---which reports the most common responses by treatment, topic, and subject type---reveals that the number 1 was the fourth most common LLM choice in the picking arm (with 5.8\% of the responses), but it is the most common choice in the coordination arm (with 78.8\% of responses). Similarly, the letter ``e,'' the most common letter in English language texts,  accounts for 50.1\% of responses in the coordination arm, despite being the fifth most popular LLM response in the picking arm. See also 
\cref{fig:share_response_change_relative_baseline}, which shows many instances in which the modal response of a specific model differs from its modal response in the picking arm.} We also find that agreement rates are lower in the divergence treatment, which incentivizes subjects to differ in their responses, compared to the picking treatment ($4\%$ vs. $14\%$ in humans, $27\%$ vs. $58\%$ in LLMs).  

Third, LLM subjects excel in the coordination arm relative to humans ($72\%$ agreement rate relative to  $31\%$ in humans). Additionally, we find that  LLM subjects lag behind human subjects in the divergence arm ($27\%$ agreement rate relative to  $4\%$ in humans). 

\cref{fig:agreement_rate_aggregate} averages the results over all 16 LLMs we study. In \cref{tab:agreement_rates_model_itself}, we present the self-agreement rates by model. The \lcnamecref{tab:agreement_rates_model_itself} reveals that most models respond to incentives and that LLM self-agreement rates are higher than human agreement rates across all treatments.  Within each treatment arm, comparisons between each LLM and humans are statistically significant ($p<0.05$). \cref{fig:self_agreement_rate_reasoning} shows that these patterns also hold when considering reasoning and non-reasoning models separately.\footnote{Reasoning models break tasks into steps and write down their logic using a variety of approaches (e.g., chain of thought or tree of thought reasoning).}

\cref{tab:agreement_rates_model_itself} also reveals that LLM choices are substantially more homogeneous than human choices across all treatments. Specifically, in both the picking and coordination arms, the average number of different choices is significantly lower for LLM subjects, suggesting that some options are considerably more salient for LLM agents than for humans. This pattern persists in the divergence arm, indicating that LLMs struggle to diversify their choices even when incentivized to do so.

\begin{table*}[t]
\centering
\caption{Average agreement rates and number of different answers by treatment}
\label{tab:agreement_rates_model_itself}
\renewcommand{\arraystretch}{1.2}
\begin{threeparttable}
\resizebox{.85\textwidth}{!}{%
\setlength{\tabcolsep}{7pt}
\begin{tabular}{lcccccc}
\toprule
\toprule
 & \multicolumn{2}{c}{Picking} & \multicolumn{2}{c}{Coordination} & \multicolumn{2}{c}{Divergence}\\
 \cmidrule(lr){2-3} \cmidrule(lr){4-5} \cmidrule(lr){6-7}
Subject & Agree. rate & Answers & Agree. rate & Answers & Agree. rate & Answers\\
\midrule
Human 
& 0.14 & 18.64 
& 0.31 & 11.30 
& 0.04 &  29.81\\
Claude-3.5H 
& 0.71 & 3.58 
& 0.66 & 3.75 
& 0.28 & 10.00 \\
Claude-4S 
& 0.89 & 1.58 
& 0.79 & 2.08 
& 0.26 & 10.83 \\
DeepSeek-R1 
& 0.65 & 3.00 
& 0.75 & 2.00 
& 0.11 & 19.00 \\
DeepSeek-70B 
& 0.37 & 7.00 
& 0.74 & 2.75 
& 0.21 & 12.00 \\
DeepSeek-8B 
& 0.28 & 12.42 
& 0.40 & 8.75 
& 0.12 & 17.92 \\
DeepSeek-14B 
& 0.35 & 9.42 
& 0.66 & 3.50 
& 0.19 & 15.08 \\
Gemini-2.0F 
& 0.66 & 3.25 
& 0.75 & 2.50 
& 0.42 & 7.50 \\
Gemini-2.5P 
& 0.61 & 4.50 
& 0.85 & 1.92 
& 0.12 & 18.00 \\
Gemma3-27B 
& 0.92 & 1.42 
& 0.92 & 1.17 
& 0.93 & 1.67 \\
GPT-4o 
& 0.54 & 4.58 
& 0.74 & 2.42 
& 0.39 & 7.25 \\
Llama-3.3 
& 0.68 & 3.50 
& 0.81 & 1.92 
& 0.32 & 12.00 \\
o4-mini 
& 0.57 & 3.92 
& 0.77 & 2.50 
& 0.14 & 23.58 \\
Phi-4 
& 0.69 & 2.83 
& 0.65 & 3.25 
& 0.28 & 10.58 \\
Phi-4(R) 
& 0.59 & 3.42 
& 0.74 & 2.42 
& 0.24 & 12.92 \\
Qwen3-14B 
& 0.45 & 5.33 
& 0.70 & 3.08 
& 0.23 & 12.75 \\
Qwen3-32B 
& 0.37 & 6.75 
& 0.65 & 3.42 
& 0.12 & 19.25 \\
\bottomrule
\end{tabular}%
}
\begin{tablenotes}[para]
\footnotesize \justifying \noindent \textit{Notes: The table reports the average agreement rate and number of distinct answers across tasks for human and LLM subjects by treatment. For LLM subjects, we report the self-agreement rate. For each treatment arm, differences between each LLM and humans are statistically significant ($p<0.05$ using a two-sided bootstrap test). To ensure comparability between humans and LLM subjects (which were queried 50 times), we randomly select 50 human subjects per treatment arm and count the average number of distinct responses across question topics. We repeat this procedure 1,000 times and report the average across samples to reduce sampling noise.}
\end{tablenotes}
\end{threeparttable}
\end{table*}

To assess the role of LLM subjects facing exact copies of each other, we explore the performance of LLM subjects when facing a different LLM. \cref{fig:pairwise_heatmap} shows the average (over all 12 topics) agreement rate for each pair of models by treatment condition (the diagonal corresponds to the self-agreement rates from \cref{tab:agreement_rates_model_itself}). We find that the higher agreement rates among LLMs persist even when models are paired with different LLM subjects.  The average agreement rate over all pairs of LLMs is 36\% in the picking treatment. This indicates that primary monoculture (salience) is still substantially higher than the human benchmark of 14\%  ($p<0.05$ using a two-sided bootstrap test), even when the LLM subjects are not identical copies. The average agreement rate in the coordination treatment is 52\%, still substantially higher than the 31\% agreement rate obtained by human subjects ($p<0.05$ using a two-sided bootstrap test). These findings point to coordination between different LLMs through strategic monoculture. Lastly, the agreement rate in the divergence arm is 8\%, still substantially higher than the human benchmark ($p<0.05$ using a two-sided bootstrap test), but substantially lower than the self-agreement rate. Consistent with our theoretical analysis, this finding points to large benefits from breaking symmetries in coordinated divergence games.  

\begin{figure}[t]
\caption{Pairwise agreement rates by treatment}
\label{fig:pairwise_heatmap}
\centering
\includegraphics[width=\textwidth]{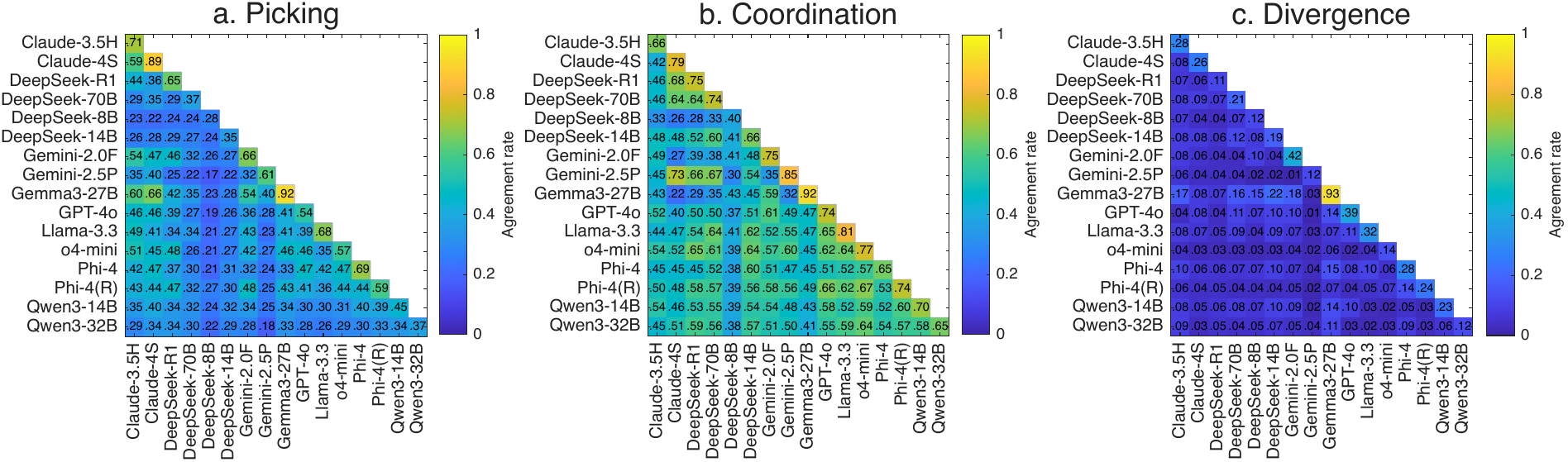}
\caption*{{\footnotesize Notes: The figure displays a heatmap of pairwise agreement rates for all pairs of LLM subjects.  Diagonal cells correspond to self-agreement rates (reported in \cref{tab:agreement_rates_model_itself}). Lighter (yellow) colors indicate higher agreement and darker (blue) colors indicate lower agreement.}}
 \vspace{-.6cm}
\end{figure}

 Finally, we evaluate the correlation between primary monoculture and performance in the coordination and divergence arms. Panel~(a) of \cref{fig:baseline_monoculture_strategic_coordination} plots, for each pair of models, the agreement rate in the picking arm (on the $x$-axis, corresponding to Panel~(a) of \cref{fig:pairwise_heatmap}) against the agreement rate in the coordination arm (on the $y$-axis, corresponding to Panel~(b) of \cref{fig:pairwise_heatmap}). Most points lie above the 45-degree line, indicating that the agreement rate increases in the coordination arm relative to the picking arm for almost any pair of LLMs. The figure also reveals a positive correlation between primary monoculture and the agreement rate in the coordination arm. 
 
 Panel~(b) of \cref{fig:baseline_monoculture_strategic_coordination} repeats the analysis for the divergence arm. All but one pair of models (Gemma3-27B paired with itself) lie under the 45-degree line, reflecting reduced agreement rates in the divergence arm. Furthermore, the figure reveals a positive correlation between primary monoculture and the agreement rate in the divergence arm. To summarize, \cref{fig:baseline_monoculture_strategic_coordination} shows that pairs of LLMs that provide similar answers in the absence of incentives tend to continue to provide similar answers both when incentivized to coordinate and when incentivized to diverge.  
\begin{figure}[t]
\caption{Primary and strategic monoculture}
\label{fig:baseline_monoculture_strategic_coordination}
\centering
\includegraphics[width=.8\textwidth]{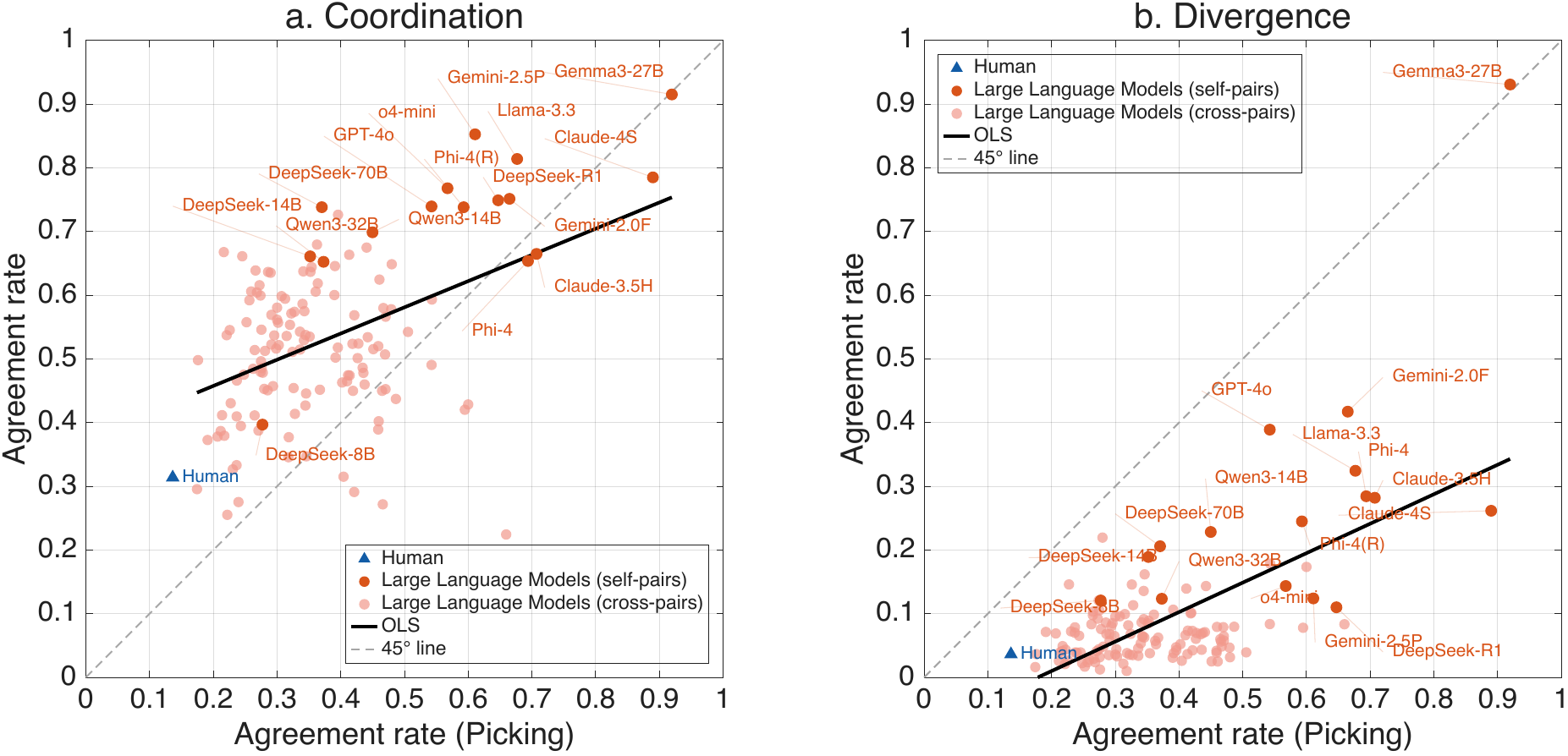}
\caption*{{\footnotesize Notes: Each red point represents a pair of LLM subjects. The $x$-axis reports their average agreement rate (across topics) in the picking arm, and the $y$-axis reports their average agreement rate in the corresponding treatment arm (coordination in Panel~(a) or divergence in Panel~(b)). Dark red points indicate self-agreement rates (LLM subject paired with a copy). The blue triangle represents the average human agreement rate.  The solid line displays the fitted ordinary least squares (OLS) regression, and the dashed line in each panel indicates the 45-degree line.}}
\vspace{-.4cm}
\end{figure}

\section{Textual Reasoning Analysis}\label{sec:text_analysis}
In this section, we analyze the textual reasoning (Chain-of-Thought, or CoT for short) of LLM subjects. We begin by analyzing the amount of LLM reasoning in the various treatment arms. 
To do this, we compare CoT length across treatment arms. In line with the difficulty of performing the task perfectly, we find (\cref{tab:stats_sentences_cot_v2}) that reasoning is lowest in the picking arm and highest in the divergence arm.\footnote{Some models only provide a summary of their reasoning. The above comparison of the length of CoT omits these models.}$^,$\footnote{\label{fn: picking while thinking}In the main experiment, LLM subjects were not instructed to ``think thoroughly'' in the picking arm. To verify that the increase in reasoning in the strategic arms is not driven by this difference, we replicated the picking arm with a prompt that instructs subjects to do so. We refer to this additional arm as \textit{picking with thinking}.
We exclude two models that were deprecated by the time we conducted this analysis (i.e., Claude-3.5H and Gemini-2.0F). The results are summarized in \cref{tab:stats_sentences_cot_v2}. Although reasoning slightly increases for most LLMs, reasoning in picking with thinking is still substantially lower than in the coordination and divergence arms. Additionally, we find that self-agreement rates decrease to 45\% (relative to 57\% in the picking arm) and agreement rates across all model pairs decrease to 26\% (relative to 35\% in the picking arm). Notably, this baseline strengthens our claim that LLM subjects are effective convergence coordinators but poor performers in responding to divergence incentives. Finally, \Cref{fig:textual_copy_person,fig:textual_reasoning_model} present the textual reasoning results for the picking-with-thinking arm, revealing little difference from the baseline picking arm. } Interestingly, human subjects in the divergence arm also take more time to complete the experiment (see \cref{tab:demographics}).

Next, we turn to analyzing the semantic content of LLM reasoning.\footnote{\citet{fish2024algorithmic} establish a causal link between LLM ``thoughts'' and actions in their setting.} We do so using two complementary methods: semantic similarity \citep[as in][]{fish2024algorithmic} and LLM as a judge \citep[as in][]{asirvatham2026gpt}. Both analyses reveal that LLM subjects display appropriate strategic reasoning across treatments.  

\subsection{Semantic Similarity}
Our semantic similarity analysis follows the methodology of \citet{fish2024algorithmic}.  We begin by splitting the LLM-generated reasoning text into
individual sentences.  As we discussed above, in contrast to \citet{fish2024algorithmic}, the amount of text is \emph{not} balanced across treatment arms (for statistics on each model and treatment, see \cref{tab:stats_sentences_cot_v2}). 
After splitting into sentences, we filter for sentences containing synonyms of the terms ``salient,'' ``obscure,'' and ``randomize'' (the full list of synonyms is available in \cref{app:text_analysis}).\footnote{None of the sentences by Gemma3-27B and GPT-4o pass this filtering.}  This yields a dataset of 227,907 individual sentences.

We note that the mere mention of a term like ``salient'' does not indicate an intention to choose a salient action. For example, the agent may be planning to \emph{avoid} the salient choice. To address this issue, we incorporate the semantic content of each sentence into the analysis. Following \citet{fish2024algorithmic}, we analyze the semantic content of sentences using standard semantic similarity methods \citep[e.g.,][]{Kenter2015, Mikolov2013EfficientEO}. Specifically, we first use OpenAI's text-embedding-3-large to convert each screened sentence into a $3{,}072$-dimensional vector. Next, for each term $X$ of the three terms we screened for, we create two reference vectors, \textsc{Choose$X$} and \textsc{Avoid$X$}.\footnote{Each reference vector is created by taking the average of the $3{,}072$-dimensional vectors corresponding to six sentences representative of the category. The exact sentences considered for each category are listed in \Cref{app:text_analysis}.} Given a screened sentence, we classify it according to the closest (in the sense of cosine similarity) reference vector.

\cref{fig:histogram_cot_sentences_aggregated_6cat} summarizes our findings. The figure reveals that an overwhelming fraction of the sentences in the picking arm discuss salience favorably. The numbers are even higher in the coordination arm, consistent with our findings on models' behavior in \cref{sec:results}. By contrast, most sentences in the divergence arm deal with obscurity. \cref{fig:textual_reasoning_model} presents disaggregated results by model. These findings suggest that LLMs’ reasoning is broadly consistent with the strategic environment of each treatment.

 \begin{figure}[t]
\caption{Distribution of screened sentences by treatment}
\label{fig:histogram_cot_sentences_aggregated_6cat}
\centering
\includegraphics[width=.8\textwidth]{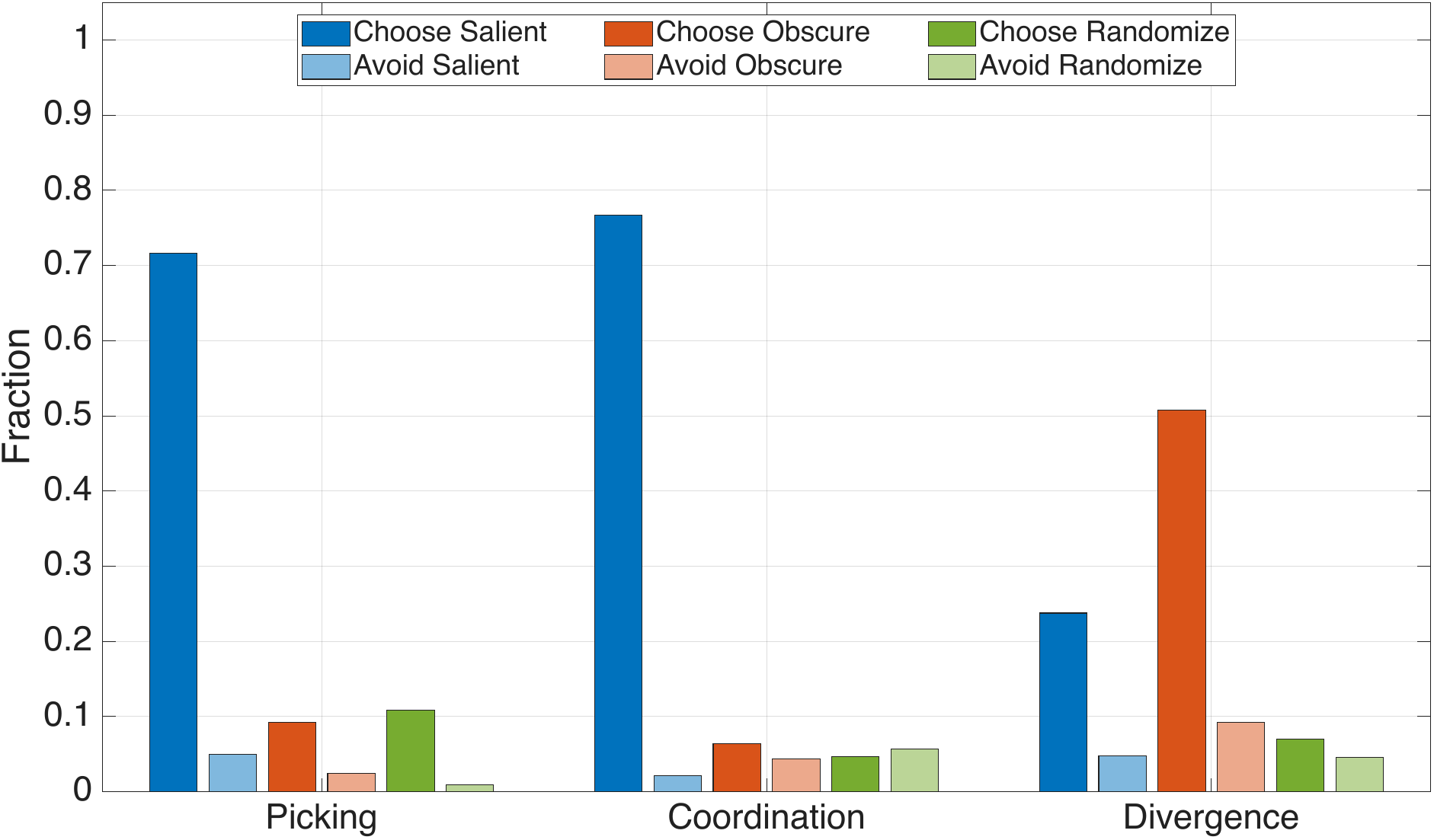}
\caption*{{\footnotesize Notes: The figure reports the share of screened sentences that are semantically closest (in cosine similarity) to one of six reference vectors: \textsc{ChooseSalient}, \textsc{AvoidSalient}, \textsc{ChooseObscure}, \textsc{AvoidObscure}, \textsc{ChooseRandomize}, and \textsc{AvoidRandomize}.}}
\vspace{-.6cm}
\end{figure}

\subsection{LLM as a Judge}
This section presents a complementary textual analysis using LLM as a judge. This approach has several advantages relative to the semantic similarity analysis. First, it considers the entire text rather than screened sentences only. Second, it enables us to consider the text as a whole, rather than individual sentences. Third, it allows a more refined analysis of the semantic content of LLM-generated chains of thought. 

In the LLM-as-a-judge analysis, we feed the LLM subjects' textual reasoning
from the experiment to another LLM that serves as a ``judge'' \citep[see][for a review of this approach and its applications]{gu2024survey}. The LLM judge is instructed to extract the main strategic reasoning expressed in the text as well as the perceived identity of the co-player (human or AI). The judge is informed about the general structure of the experiment and the topic of the question at hand, but is not given information about the experimental arm the specific CoT came from, or about the LLM that generated the response.  

Operationally, for each reasoning text, we use an independent call to our judge, which is implemented using Gemini-2.5-Flash. The judge is provided with eight categories of strategic reasoning to choose from: primary salience, secondary salience, higher-order salience, Schelling salience, obscurity, higher-order obscurity, randomization, and ``not mentioned.'' For each category, the judge is provided with a definition and examples. Additionally, it is provided with eight categories of beliefs about the identity of the co-player to choose from (various degrees of certainty about the humanity of the co-player, as well as the possibility of it being an exact copy).\footnote{This question is especially relevant for one of the experimental variations presented in the next section.} The full judge prompt is available in \cref{app:CoT_textual_analysis_judge_prompt}. 

 \cref{fig:Judge}  summarizes the LLM judge's determinations about LLM subjects' strategic reasoning.
 Our findings are in line with our results on LLM subjects' choices. First, in the picking arm, the most common strategic categories are ``not mentioned'' and primary salience. This reflects the lack of incentive to coordinate and points to a potential explanation for the high levels of primary monoculture among LLM subjects. Second, in the coordination arm, the most common strategic categories are secondary salience and Schelling salience. This provides further evidence of strategic algorithmic monoculture. Importantly, the increase in the strategic reasoning categories that reflect strategic algorithmic monoculture (68 percentage points) is substantially higher than the increase in agreement rate we can detect directly from choice data (14 percentage points). Third, in the divergence arm, the most common strategic categories are obscurity and higher-order obscurity. Again, this reflects strategic algorithmic monoculture (in the sense that agreement rates are regulated strategically), and the change in strategic reasoning is more drastic than the change we measure in choice data (68 percentage points vs. 31 percentage points).  \Cref{fig:textual_reasoning_model} presents the results for each model separately.

 \begin{figure}[t]
\caption{LLM-judge classification of strategic reasoning by treatment}
\label{fig:Judge}
\centering
\includegraphics[width=.8\textwidth]{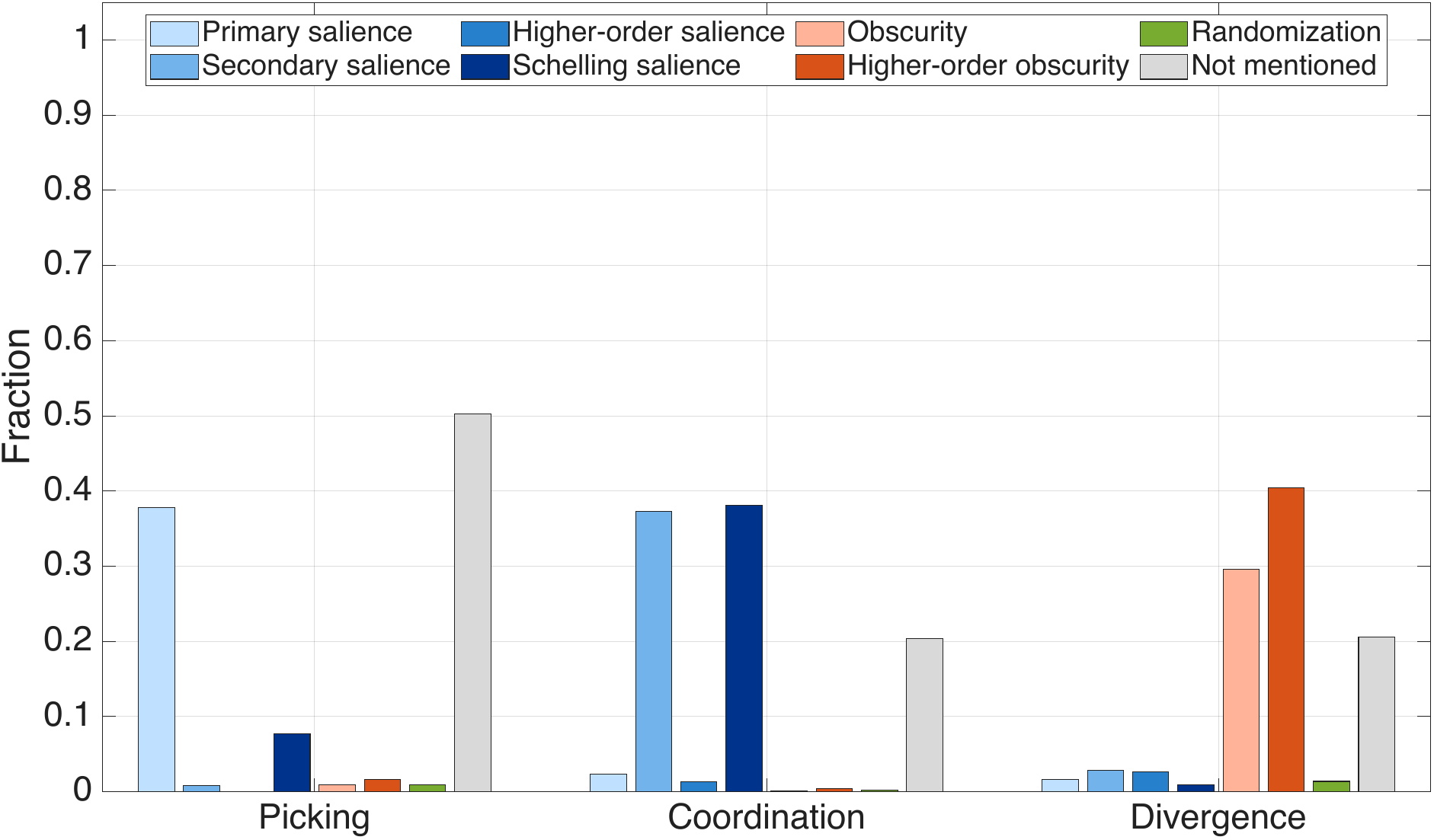}
\caption*{{\footnotesize Notes: The figure reports the share of LLM-reasoning texts that fall under each of the eight strategic reasoning categories according to the LLM-judge's determinations. The judge is implemented using Gemini-2.5-Flash. For details, see \cref{app:CoT_textual_analysis_judge_prompt}.}}
\vspace{-.4cm}
\end{figure}

\section{Additional Variations}\label{sec: additional experiments}
To probe the mechanisms behind our findings, we conducted several variations of the experiment. This section collects our findings.

\subsection{Capabilities}\label{sec:capabilities}
A possible explanation for the high agreement rates among LLM subjects is that they only have access to a small number of valid answers. In the context of the divergence arm, it is also possible that they have limited ability to randomize between answers \citep{kleinberg2024random}. In this section, we assess these capabilities. 

To do so, we instruct each of the LLMs we study to generate a list of 100 valid answers (or, when there are fewer answers, as many as possible).\footnote{We exclude two models that were deprecated by the time we conducted this analysis (i.e., Claude-3.5H and Gemini-2.0F).} After completing this task, we ask each LLM to choose one answer from the list uniformly at random (the exact prompts are provided in \Cref{capabilities prompts}). We repeat this process 50 times for each LLM. We refer to the resulting choices as the \textit{random} arm.

\Cref{tab:list_generation} provides statistics on the average number of valid answers each model generated. It reveals that models can readily produce many valid answers, ruling out lack of access as an explanation.  

Next, we conduct two complementary analyses using the results from the random arm. First, we calculate the agreement rate between LLM subjects in the random arm and in the divergence arm. We find that agreement rates are substantially lower than those achieved by LLM subjects in the divergence arm: on average, 4\% when using the same model and 3\% with a random pair of models (the corresponding numbers in the divergence arm of the main experiment were 27\% and 8\% considering all models, and 26\% and 8\% considering the models included in this analysis). In the second analysis, we calculate the agreement rate between pairs of LLM subjects in the random arm.  The average self-agreement rate is 12\%, and the average agreement rate across models is 3\%, still substantially lower than the divergence arm of the main experiment.\footnote{Panel~(a) of \Cref{fig:agreement_divergence_vs_random} displays pairwise agreement rates for the first analysis and Panel~(b) of \Cref{fig:agreement_divergence_vs_random} for the second analysis. The higher agreement rates in the second analysis indicate that LLM subjects are not fully successful at following the instruction to randomize uniformly, and that part of the improvement in performance in the first analysis results from breaking the symmetry between LLM subjects.} Altogether, these findings indicate that LLM subjects are technically capable of achieving better outcomes in the divergence arm, suggesting that their poor performance results from other reasons. 

\subsection{Temperature}

Since subjects are symmetric, our theoretical analysis suggests that success in the divergence arm crucially relies on the ability to randomize. On the other hand, randomization hinders coordination. To further investigate the role of the ability to randomize, we replicate the experiment using the highest and lowest temperatures allowed by each model (recall that the temperature is a parameter that modulates the randomness of LLM-generated text; higher temperature corresponds to a higher degree of stochasticity in token prediction).\footnote{Our analysis covers 10 models, as we exclude models that provided unintelligible outputs under extreme temperatures (i.e., DeepSeek-8B, GPT-4o, and Phi-4(R)), models that did not allow extreme temperature settings (i.e., Claude-4S and o4-mini), and one model that was deprecated by the time we conducted this analysis (i.e., Claude-3.5H).}

\cref{fig:temperature_analysis_pool} presents the aggregated results. Consistent with our theoretical analysis, when using the lowest temperatures, the self-agreement rate increases across all treatments ($p<0.05$  using a two-sided bootstrap test). When using the highest temperatures, self-agreement rates are lower across treatments ($p<0.05$ using a two-sided bootstrap test). Finally, we note that even at these extreme temperatures, human subjects continue to outperform LLM subjects in the divergence arm, and the opposite holds in the coordination arm ($p<0.05$ using a two-sided bootstrap test). 

\cref{fig:agreement_rate_temperature_all_pairs} displays the results when LLM subjects are paired with any LLM subject, and \cref{fig:temperature_analysis_heterogeneity} displays the results for each LLM separately. The results continue to hold directionally, but the magnitudes are smaller. Altogether, our findings point to a tradeoff for users: increasing the temperature improves LLM subjects' performance in coordinated divergence tasks, but also hurts their performance in coordination tasks.  

\begin{figure}[t]
\caption{Agreement rate under different temperature values by treatment}
\label{fig:temperature_analysis_pool}
\centering
\includegraphics[width=.8\textwidth]{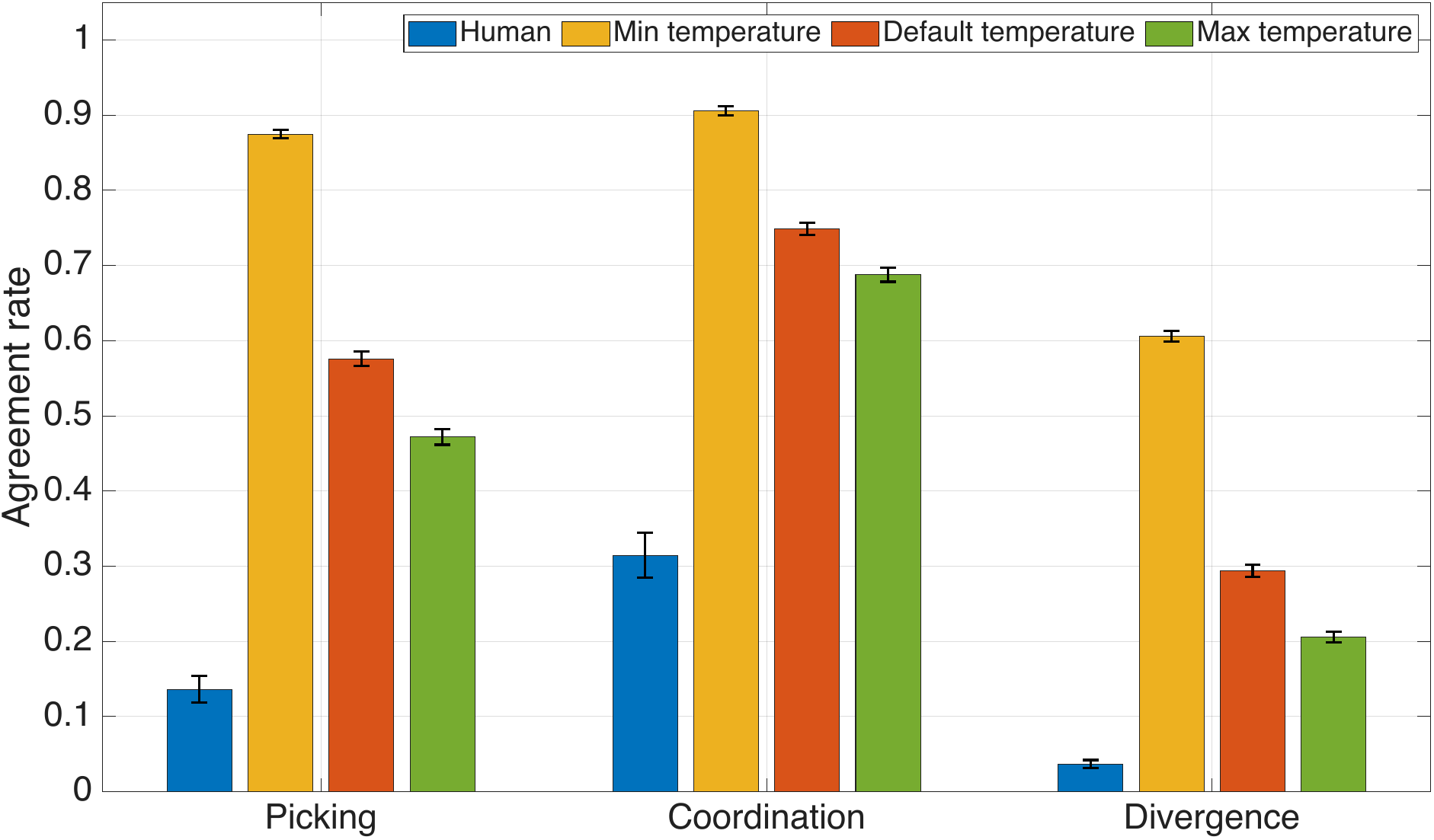}
\caption*{\footnotesize Notes: Bars represent the average self-agreement rates across topics for LLM subjects under different temperatures, for each treatment arm. For comparability, the default temperature average is calculated only over models that were not excluded, and human agreement rates are also displayed. Vertical lines indicate 95\% bootstrapped confidence intervals. Holding fixed subject type (human or LLM) and temperature settings, all pairwise comparisons across treatment arms are statistically significant ($p<0.05$ using a two-sided bootstrap test). Within each treatment arm, comparisons between humans and LLMs at each temperature level are also statistically significant ($p<0.05$). Comparisons between LLMs at the default and lowest temperatures, as well as between the default and highest temperatures, are statistically significant across all treatment arms ($p<0.05$).}
\vspace{-.6cm}
\end{figure}

\subsection{Information}
\label{subsec:information}

Coordination through secondary salience and Schelling salience requires reasoning about the other player. In this section, we examine whether changing the information that is provided to LLM subjects about the identity of the other player affects coordination. We find that information changes LLM subjects' beliefs about the identity of the co-player, but has little effect on their performance.

In the main experiment (\cref{sec:experimental_design}), LLM subjects were informed that the other subject was another LLM. In this section, we change the prompt so that the LLM is informed that \textit{(i)} it is facing an identical \emph{copy} of itself (the \emph{copy} condition), or \textit{(ii)} it is facing  \emph{``another person''} (the \emph{person} condition). The exact textual instructions for these two conditions are provided in  \cref{app:information}. 

To quantify the effect of providing this information, we estimate the following linear regression (separately for each condition and treatment):
\begin{equation}
\label{eq:information_ols}
A_{itc} = \alpha +  \beta \ T_{c}  + \mu_i + \delta_t + \varepsilon_{itc}, 
\end{equation}
where the variable $A_{itc}$ is the self-agreement rate of model $i$ in topic $t$ under condition $c$, and $T_{c}$ is an indicator equal to one under the 
treatment condition (either copy or person) and zero otherwise. The specification also includes model fixed effects, $\mu_i$, and topic fixed effects, $\delta_t$. The coefficient of interest is $\beta$, which measures the causal effect of changing the information about the other player's identity on strategic behavior.

\cref{tab:information_regression}
 displays the estimated effects for both the copy and the person conditions. Under the copy condition, we find small but statistically significant improvements in LLM performance in both treatment arms ($p<0.05$ using a two-sided bootstrap test), with agreement rates increasing in the coordination arm and decreasing in the divergence arm.\footnote{Our findings echo those of \citeauthor{jackson2014culture}'s \citeyearpar{jackson2014culture} human-subject study.} Under the person condition, we find slightly worse performance in the divergence arm ($p<0.05$) and a small, statistically nonsignificant improvement in the coordination arm.  \cref{fig:belief_effect_reasoning} presents heterogeneity analysis by reasoning capabilities, and \cref{fig:beliefs_effect_by_model} shows the effect by model.

\begin{table}[t]
\centering
\resizebox{0.8\textwidth}{!}{ 
\renewcommand{\arraystretch}{1.2}
  \begin{threeparttable}
\caption{Effect of information treatment on self-agreement rates}
\label{tab:information_regression}
\setlength{\tabcolsep}{30pt}
\begin{tabular}{l c c c c}
\toprule
\toprule
 & \multicolumn{2}{c}{Coordination} & \multicolumn{2}{c}{Divergence} \\
 \cmidrule(l{20pt}r{20pt}){2-3} \cmidrule(l{20pt}r{20pt}){4-5}
 & (1) & (2) & (3) & (4) \\
\midrule
Copy & $0.017$ & & $-0.020$ & \\
 & (0.005) & & (0.005) & \\[6pt]
Person & & $0.009$ & & $0.012$ \\
 & & (0.005) & & (0.004) \\
 \midrule
 Observations & 384 & 384 & 384 & 384\\
 $R^2$ & 0.449 & 0.461 & 0.584 & 0.602\\
\bottomrule
\end{tabular}
\begin{tablenotes}[flushleft]
\item[] \noindent \footnotesize \textit{ Notes: The table presents the results of the linear regression specified in \cref{eq:information_ols}, which includes model and topic fixed effects.
The dependent variable is the self-agreement rate in the coordination (columns 1--2) and divergence (columns 3--4) arms. \textit{Copy} and \textit{Person} are indicators that equal one in the copy and person conditions, respectively.   
Standard errors in parentheses are bootstrapped.}
\end{tablenotes}
\end{threeparttable}
}
\end{table}

\paragraph{Textual reasoning analysis.} We repeat the textual reasoning analyses from \cref{sec:text_analysis}. We find that the length of reasoning in the divergence arm under the copy condition is even higher than in the main divergence arm (see \cref{tab:stats_sentences_cot_v2}). Additionally, relative to the divergence arm of the main experiment, our semantic similarity analysis finds that in the copy condition, discussion of randomization becomes more prevalent while the overall level of obscurity decreases slightly (in line with our theoretical analysis). By contrast, the person condition has little effect across experimental arms. Similarly, the LLM-judge analysis also documents an increase in the randomization category under the copy condition in the divergence arm. It goes beyond the semantic similarity analysis by documenting a sharp increase in the Schelling salience category under the copy condition in the coordination arm, mostly replacing secondary salience reasoning (see \Cref{fig:textual_copy_person} for aggregated results and \Cref{fig:textual_copy_person_by_model} for results disaggregated by model).  

Recall that the LLM judge was also instructed to classify the reasoning with respect to beliefs about the identity of the other player and the degree of certainty expressed. We find that the judge's classification is consistent with the expected effect of our experimental conditions. Specifically, in the baseline experiment, the most common categories are LLM co-player and ``not mentioned,'' in the copy condition, the most common responses are copy and ``not mentioned,'' and in the person condition, the most common responses are Human and ``not mentioned'' (see \cref{fig:beliefs judge} for aggregated results and \cref{fig:identity_condition_model} for results disaggregated by model). Additionally, across all conditions, the classifications are nearly identical in the divergence and coordination arms.

\subsection{Personas} \label{sec: personas}
In this section, we consider an established approach for increasing heterogeneity: personas. As in \citet{mei2024turing}, we replicate the main experiment, instructing LLM subjects to assume the role of a human with specific characteristics.\footnote{We thank Alex Imas for proposing this analysis. We exclude two models that were deprecated by the time we conducted this analysis (i.e., Claude-3.5H and Gemini-2.0F).} Specifically, to mimic the distribution of traits in the human-subject sample, we randomly draw 50 different human subjects and construct personas based on their gender, age, occupation, race/ethnicity, and education level. Operationally, we provide the persona information first (see \Cref{persona prompt}) and then use the prompt from the person condition.\footnote{We use the person condition prompt because the main analysis prompt says that the co-player is ``another LLM,'' which contradicts the persona instructions. For the picking arm, we use the main analysis prompt, as there is no such contradiction.} We use the same 50 personas for all topics and treatment arms, but rely on separate, independent LLM calls (yielding a total of $3\times 50 \times 12 = 1{,}800$ independent calls per model). 

\Cref{fig:persona_self_agreement} summarizes the results. Although personas shift LLM subjects' behavior closer to human subjects, our main results remain qualitatively unchanged. In particular, even when endowed with heterogeneous personas, LLM subjects significantly lag behind humans in the divergence arm ($p<0.05$ using a two-sided bootstrap test). \Cref{fig:persona_all_pair_models,fig:pairwise_heatmap_persona,fig:persona_prompt_topic} show similar results in replications of other analyses from \cref{sec:results}. Interestingly, when we consider all LLM pairs, adding personas does not improve performance in the divergence arm but hurts performance in the coordination arm.

\begin{figure}[t]
\caption{Agreement rate by LLMs with personas}
\label{fig:persona_self_agreement}
    \centering
    \includegraphics[width=.8\linewidth]{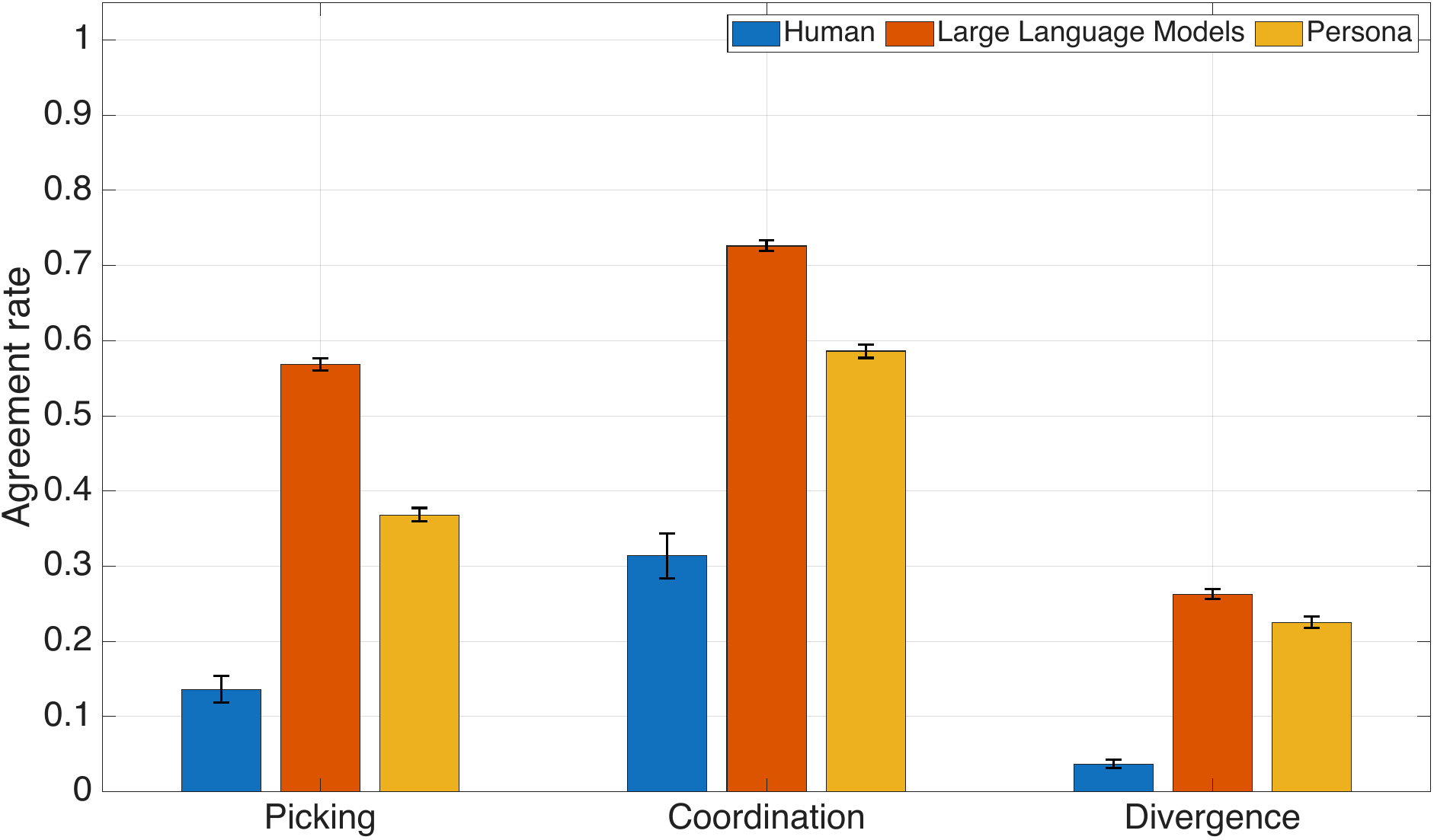}
\vspace{.01mm}
\caption*{\footnotesize Notes: Bars report the average agreement rate across question topics for human and LLM subjects by treatment.
LLM agreement rates are the average across models included in the analysis when facing the same model. Vertical lines indicate 95\% bootstrapped confidence intervals. All pairwise comparisons between humans and personas, as well as across treatment arms for personas, are statistically significant ($p<0.05$ using a two-sided bootstrap test). Comparisons between LLMs and personas are statistically significant in all treatment arms ($p<0.05$). }
\vspace{-.6cm}
\end{figure}

\section{Discussion}\label{sec: discussion}
The role of AI agents in society is rapidly increasing. These agents do not operate in a vacuum. Rather, they interact with each other, forming complex multi-agent systems. These systems hold great promise, but they also pose novel risks, many of which stem from undesirable coordination (e.g., collusion) and miscoordination \citep{hammond2025multi}. 

Against this backdrop, we make several contributions. Conceptually, we introduce a taxonomy distinguishing primary from strategic algorithmic monoculture, paralleling the distinction in the human coordination literature \citep{schelling1960strategy}. Empirically, we implement a simple experimental design that enables us to disentangle \emph{strategic} monoculture from primary monoculture. Methodologically, we contribute to the coordination literature by demonstrating that AI behavioral science techniques (such as analysis of textual reasoning) can complement standard approaches and provide a more comprehensive picture of the mechanisms by which LLM agents coordinate. To the extent that LLM reasoning is indicative of human strategic reasoning, these tools can be used to address some challenging open questions in the coordination literature and in human behavioral science more broadly.

Our experiment reveals that LLMs are exceptionally effective at converging on focal responses, but comparatively poor at sustaining socially valuable heterogeneity. LLMs display elevated levels of primary monoculture: in the absence of incentives, LLM choices are substantially more homogeneous than human choices. And, while both humans and LLMs respond to incentives, increasing or decreasing agreement when it is rewarded, their performance on coordination tasks differs markedly. LLMs dramatically outperform humans in coordination on the same action, but humans dramatically outperform LLMs in coordinated divergence (despite LLMs displaying the correct strategic reasoning).

Our findings contribute to the discussion of monoculture, model multiplicity, and related multi-agent risks  \citep{bommasani2021opportunities,goreckimonoculture, raghavan2025competitiondiversitygenerativeai}. 
A main concern in this literature is that the use of AI might lead to undesirable consequences in economic systems. For example, it may lead to systematic exclusion due to correlated assessment noise, as in the hiring example in \cref{sec:motivating_example}. 
At first glance, one might hope that increasing the number of models---even if their assessment errors are partially correlated---would suffice to mitigate this risk. 
However, this argument overlooks AI agents' strategic behavior. We show that LLMs can regulate the degree to which they are correlated in response to strategic incentives \citep[e.g., to avoid adverse selection, as in our labor market example, or congestion as in][]{baek2026strategichiringalgorithmicmonoculture}.
Our findings are equally consequential in the context of antitrust policy and financial regulation, raising concerns about the use of AI agents in pricing and trading and the potential that it will facilitate collusion \citep{doj2026glad}  and heighten financial fragility \citep{sec2023gensler}. We hope that our taxonomy and empirical findings on algorithmic monoculture will contribute to understanding these important risks and help in identifying ways to mitigate them.

\bibliographystyle{apalike}
\bibliography{biblio}

\clearpage

\appendix
\crefalias{section}{appsection}
\crefalias{subsection}{appsection}
\crefalias{subsubsection}{appsection}
\crefalias{figure}{appfigure}
\crefalias{table}{apptable}

\setcounter{figure}{0}
\renewcommand{\thefigure}{A.\arabic{figure}}  

\setcounter{table}{0} 
\renewcommand{\thetable}{A.\arabic{table}} 

\section*{\centering{Supplementary Materials}}

\vspace{1cm}

The appendix is organized as follows. \Cref{sec:appendix_figures_tables} presents additional tables and figures, and \Cref{app:detailed_analysis} provides further analyses complementing the main results. \Cref{app:exp_details} and \Cref{app:huma_expt} describe the implementation details of the experiments with LLM and human subjects, respectively. \Cref{sec:textual_analysis_app} provides further details on the textual reasoning analysis, including the semantic similarity and LLM-as-a-judge approaches. \Cref{app:prereg} contains the preregistration.

\newpage

\section{Additional Tables and Figures}  
\label{sec:appendix_figures_tables}

\begin{table}[h!]
\centering
\caption{Summary statistics (human subjects)}   
\label{tab:demographics}
\renewcommand{\arraystretch}{1.3}
\resizebox{.9\textwidth}{!}{
\begin{threeparttable}
\setlength{\tabcolsep}{12pt}
\begin{tabular}{lcccccc}
\toprule
\toprule
 & \multicolumn{2}{c}{Picking} & \multicolumn{2}{c}{Coordination} & \multicolumn{2}{c}{Divergence} \\
\cmidrule(lr){2-3} \cmidrule(lr){4-5} \cmidrule(lr){6-7}
Variable & Mean & SD & Mean & SD & Mean & SD \\
\midrule
Age & 47.47 & 14.03 & 47.09 & 11.91 & 45.33 & 12.63 \\
Female & 0.50 & 0.50 & 0.54 & 0.50 & 0.57 & 0.50 \\
\multicolumn{7}{l}{Education} \\
\quad Primary school & 0.01 & 0.10 & 0.03 & 0.17 & 0.02 & 0.14 \\
\quad Secondary school & 0.26 & 0.44 & 0.24 & 0.43 & 0.33 & 0.47 \\
\quad College or university & 0.56 & 0.50 & 0.53 & 0.50 & 0.53 & 0.50 \\
\quad Graduate degree (e.g., Master's, PhD) & 0.16 & 0.37 & 0.19 & 0.39 & 0.11 & 0.32 \\
\quad Prefer not to say & 0.01 & 0.10 & 0.01 & 0.10 & 0.00 & 0.00 \\
\multicolumn{7}{l}{Race or Ethnicity} \\
\quad Asian & 0.07 & 0.25 & 0.04 & 0.20 & 0.02 & 0.14 \\
\quad Black or African American & 0.17 & 0.38 & 0.06 & 0.24 & 0.17 & 0.37 \\
\quad Hispanic or Latino & 0.05 & 0.21 & 0.02 & 0.14 & 0.04 & 0.20 \\
\quad White or Caucasian & 0.71 & 0.45 & 0.85 & 0.36 & 0.75 & 0.44 \\
\quad Other & 0.00 & 0.00 & 0.03 & 0.17 & 0.02 & 0.14 \\
Duration (in minutes) & 3.91 & 3.31 & 3.84 & 2.40 & 4.92 & 2.38 \\
Tab switches & 0.50 & 0.62 & 0.40 & 0.62 & 0.55 & 0.58 \\
\\
Number of Observations & \multicolumn{2}{c}{105} & \multicolumn{2}{c}{99} & \multicolumn{2}{c}{97} \\
\bottomrule
\end{tabular}
\begin{tablenotes}[para]
\end{tablenotes}
\end{threeparttable}
}
\end{table}

\begin{table}[h!]
\caption{Features of LLM subjects}
\label{tab:model_details}
\centering
\resizebox{.9\textwidth}{!}{
\setlength{\tabcolsep}{15pt}
\renewcommand{\arraystretch}{1.4}
\begin{threeparttable}
\begin{tabular}{c c c c c c}
\toprule
\toprule
\textbf{Type} & \textbf{Provider} & \textbf{Label} & \textbf{Model} & \textbf{Reasoning} & \textbf{Default Temperature} \\
\midrule
\multirow{6}{*}{Closed source} 
 & \multirow{2}{*}{Anthropic} & Claude-3.5H & Claude-3.5-Haiku & No & 1\\
 &  & Claude-4S & Claude-4-Sonnet & Yes & 1\\
 & \multirow{2}{*}{Google} & Gemini-2.0F & Gemini-2.0-Flash & No & 1\\
 &  & Gemini-2.5P & Gemini-2.5-Pro & Yes & 1\\
 & \multirow{2}{*}{OpenAI} & GPT-4o & GPT-4o & No & 1\\
 &  & o4-mini & o4-mini & Yes & 1\\

\midrule

\multirow{9}{*}{Open source}
 & \multirow{4}{*}{DeepSeek} & DeepSeek-R1 & DS-R1-671B & Yes & 1\\
 &  & DeepSeek-70B & DS-Llama-70B & Yes & 0.5\\
 & & DeepSeek-8B & DS-Llama-8B & Yes & 0.5\\
 &  & DeepSeek-14B & DS-Qwen-14B & Yes & 0.5\\
 &  Google & Gemma3-27B & Gemma3-27B & No & 0.5\\
 & Meta & Llama-3.3 & Llama-3.3-70B & No & 1\\
 & \multirow{2}{*}{Microsoft} & Phi-4 & Phi-4-14B & No & 0.5\\
 &  & Phi-4(R) & Phi-4-Reasoning & Yes & 0.5\\
 & \multirow{2}{*}{Alibaba} & Qwen3-14B & Qwen3-14B & Yes & 0.5\\
 &  & Qwen3-32B & Qwen3-32B & Yes & 0.5\\

\bottomrule
\end{tabular}
 \begin{tablenotes}[flushleft]
\item[] \footnotesize \textit{Notes: The table provides information on LLM subjects.  In particular, it details the name of the model alongside the label used in the paper, whether the model has reasoning capabilities, and its default temperature.}
\end{tablenotes}
\end{threeparttable}
}
\end{table}

\begin{table}[h!]
\centering
\caption{Most frequent responses by treatment, topic, and subject type}
\resizebox{\textwidth}{!}{
\begin{tabular}{clclclclclclc}
\toprule
\toprule
& \multicolumn{4}{c}{\textbf{Picking}} & \multicolumn{4}{c}{\textbf{Coordination}} & \multicolumn{4}{c}{\textbf{Divergence}} \\
\cmidrule(lr){2-5} \cmidrule(lr){6-9} \cmidrule(lr){10-13}
& \multicolumn{2}{c}{Human} & \multicolumn{2}{c}{LLMs} & \multicolumn{2}{c}{Human} & \multicolumn{2}{c}{LLMs} & \multicolumn{2}{c}{Human} & \multicolumn{2}{c}{LLMs} \\
\cmidrule(lr){2-3} \cmidrule(lr){4-5} \cmidrule(lr){6-7} \cmidrule(lr){8-9} \cmidrule(lr){10-11} \cmidrule(lr){12-13}
Topic & Response & Freq. (\%) & Response & Freq. (\%) & Response & Freq. (\%) & Response & Freq. (\%) & Response & Freq. (\%) & Response & Freq. (\%)\\
\midrule
\multirow{5}{*}{Athlete} & michael jordan & 25 & lebron james & 68 & michael jordan & 28 & lebron james & 37 & lebron james & 7 & lebron james & 13 \\
& lebron james & 19 & lionel messi & 14 & lebron james & 21 & lionel messi & 34 & michael jordan & 4 & shohei ohtani & 8 \\
& lionel messi & 4 & serena williams & 7 & tom brady & 8 & cristiano ronaldo & 14 & simone biles & 4 & simone biles & 7 \\
& serena williams & 4 & cristiano ronaldo & 4 & travis kelce & 7 & michael jordan & 13 & usain bolt & 4 & serena williams & 6 \\
& dak prescott & 3 & michael jordan & 2 & kobe bryant & 6 & serena williams & 1 & aaron judge & 3 & mikaela shiffrin & 4 \\
\midrule
\multirow{5}{*}{Car} & ford & 29 & toyota & 92 & ford & 47 & toyota & 98 & ford & 12 & volvo & 14 \\
& toyota & 23 & ford & 5 & toyota & 23 & ford & 2 & honda & 5 & mazda & 14 \\
& honda & 12 & tesla & 2 & honda & 13 & --- & --- & hyundai & 5 & subaru & 14 \\
& chevrolet & 5 & bmw & 1 & chevrolet & 8 & --- & --- & kia & 5 & toyota & 8 \\
& kia & 5 & --- & --- & nissan & 2 & --- & --- & nissan & 5 & tesla & 5 \\
\midrule
\multirow{5}{*}{City} & paris & 14 & tokyo & 45 & new york & 50 & paris & 53 & oslo & 3 & kyoto & 17 \\
& new york & 11 & paris & 39 & paris & 23 & new york & 27 & barcelona & 2 & ulaanbaatar & 17 \\
& london & 6 & kyoto & 7 & london & 4 & london & 10 & berlin & 2 & reykjavik & 7 \\
& chicago & 4 & london & 3 & los angeles & 4 & tokyo & 8 & dallas & 2 & valparaiso & 2 \\
& houston & 4 & sydney & 3 & chicago & 2 & istanbul & 1 & dubai & 2 & tokyo & 2 \\
\midrule
\multirow{5}{*}{Color} & blue & 44 & blue & 78 & blue & 60 & blue & 82 & purple & 12 & teal & 26 \\
& red & 17 & cerulean & 7 & red & 29 & red & 17 & orange & 8 & chartreuse & 14 \\
& green & 11 & red & 4 & green & 3 & green & 1 & magenta & 7 & purple & 8 \\
& purple & 8 & teal & 2 & purple & 3 & --- & --- & pink & 7 & turquoise & 8 \\
& black & 6 & sky blue & 2 & black & 2 & --- & --- & red & 7 & cerulean & 8 \\
\midrule
\multirow{5}{*}{Disease} & cancer & 33 & diabetes & 29 & cancer & 51 & influenza & 33 & cancer & 10 & whipple's disease & 8 \\
& aids & 8 & influenza & 22 & covid-19 & 15 & diabetes & 18 & lupus & 7 & tuberculosis & 7 \\
& diabetes & 6 & alzheimer's & 20 & aids & 5 & cancer & 16 & crohn's disease & 6 & kawasaki & 7 \\
& influenza & 6 & malaria & 10 & influenza & 5 & covid-19 & 15 & diabetes & 5 & creutzfeldt-jakob & 6 \\
& measles & 5 & tuberculosis & 7 & heart disease & 4 & alzheimer's & 11 & lyme disease & 4 & kuru & 5 \\
\midrule
\multirow{5}{*}{Flower} & rose & 49 & rose & 71 & rose & 84 & rose & 95 & rose & 10 & sunflower & 10 \\
& daisy & 20 & sunflower & 13 & daisy & 8 & sunflower & 4 & daisy & 7 & orchid & 10 \\
& lily & 7 & daisy & 5 & tulip & 3 & daisy & 1 & tulip & 7 & jasmine & 7 \\
& iris & 5 & tulip & 4 & lily & 2 & --- & --- & daffodil & 6 & bleeding heart & 4 \\
& tulip & 4 & lily & 3 & sunflower & 2 & --- & --- & sunflower & 6 & dahlia & 4 \\
\midrule
\multirow{5}{*}{Food} & pizza & 24 & pizza & 53 & pizza & 43 & pizza & 74 & pizza & 7 & sushi & 14 \\
& burger & 9 & sushi & 14 & burger & 9 & apple & 17 & potato & 4 & quinoa & 7 \\
& bread & 8 & apple & 13 & apple & 8 & bread & 6 & egg & 3 & pizza & 5 \\
& apple & 6 & spaghetti & 3 & bread & 8 & rice & 1 & apple & 2 & pineapple & 4 \\
& banana & 4 & avocado & 3 & chicken & 3 & avocado toast & 1 & asparagus & 2 & kimchi & 3 \\
\midrule
\multirow{5}{*}{Geometric} & triangle & 33 & triangle & 38 & square & 49 & circle & 63 & triangle & 16 & pentagon & 20 \\
& square & 27 & circle & 32 & circle & 23 & square & 21 & rhombus & 13 & hexagon & 11 \\
& circle & 17 & square & 11 & triangle & 22 & triangle & 13 & octagon & 10 & heptagon & 10 \\
& hexagon & 6 & hexagon & 5 & hexagon & 2 & cube & 2 & hexagon & 10 & dodecahedron & 8 \\
& octagon & 6 & cube & 4 & cube & 1 & hexagon & 1 & square & 5 & trapezoid & 6 \\
\midrule
\multirow{5}{*}{Letter} & a & 42 & a & 27 & a & 73 & e & 50 & a & 7 & q & 33 \\
& e & 10 & m & 16 & e & 6 & a & 46 & j & 6 & z & 13 \\
& b & 9 & q & 10 & b & 3 & m & 4 & w & 6 & k & 11 \\
& z & 5 & g & 9 & c & 3 & --- & --- & z & 6 & a & 10 \\
& d & 4 & e & 8 & s & 3 & --- & --- & l & 5 & x & 8 \\
\midrule
\multirow{5}{*}{Month} & april & 12 & january & 23 & november & 36 & january & 51 & february & 17 & september & 30 \\
& may & 11 & october & 17 & january & 29 & july & 15 & march & 14 & february & 19 \\
& november & 11 & march & 17 & december & 6 & december & 8 & august & 12 & april & 10 \\
& january & 10 & july & 14 & june & 6 & october & 8 & june & 10 & august & 7 \\
& march & 10 & april & 13 & july & 5 & june & 7 & april & 9 & november & 7 \\
\midrule
\multirow{5}{*}{Positive Number} & 2 & 20 & 7 & 37 & 2 & 33 & 1 & 79 & 8 & 12 & 7 & 15 \\
& 4 & 11 & 42 & 33 & 1 & 22 & 7 & 14 & 12 & 5 & 1 & 11 \\
& 5 & 10 & 5 & 12 & 7 & 11 & 42 & 5 & 2 & 4 & 2 & 9 \\
& 1 & 9 & 1 & 6 & 4 & 7 & 5 & 2 & 7 & 4 & 17 & 8 \\
& 8 & 8 & 17 & 3 & 10 & 5 & --- & --- & 22 & 3 & 42 & 7 \\
\midrule
\multirow{5}{*}{Trait} & honesty & 18 & empathy & 28 & honesty & 32 & honesty & 57 & honesty & 20 & empathy & 20 \\
& kindness & 11 & honesty & 16 & kindness & 13 & kindness & 20 & humility & 7 & resilience & 20 \\
& empathy & 6 & compassion & 14 & funniness & 8 & empathy & 12 & kindness & 6 & perseverance & 6 \\
& funniness & 6 & kindness & 10 & friendliness & 4 & loyalty & 2 & loyalty & 4 & meticulousness & 5 \\
& friendliness & 5 & resilience & 9 & smartness & 4 & resilience & 2 & shyness & 4 & curiosity & 4 \\
\bottomrule
\end{tabular}}
\label{tab:distr_final}
\end{table}

\begin{table}[h!]
\caption{Number of reasoning sentences by model and treatment}
\label{tab:stats_sentences_cot_v2}
\centering
\resizebox{.9\textwidth}{!}{
\setlength{\tabcolsep}{8pt}
\renewcommand{\arraystretch}{1.5}
\begin{threeparttable}
\begin{tabular}{lccccccccc}
\toprule
\toprule
Model & Picking & \makecell{Picking \\ (with thinking)} & Coordination & Divergence & \makecell{Coordination \\ (Copy)} & \makecell{Divergence \\ (Copy)} & \makecell{Coordination \\ (Person)} & \makecell{Divergence \\ (Person)} & Total \\
\midrule
Claude-3.5H   & 0       & ---       & 3,080   & 2,199   & 2,562   & 1,949   & 1,948   & 1,534   & 13,272    \\
Claude-4S     & 0       & 1,515   & 4,017   & 4,140   & 3,922   & 5,685   & 3,752   & 3,938   & 26,969    \\
DeepSeek-R1   & 6,884   & 10,241  & 40,752  & 107,375 & 56,113  & 176,669 & 36,530  & 107,968 & 542,532   \\
DeepSeek-70B  & 9,945   & 11,033  & 19,311  & 48,469  & 26,449  & 58,785  & 22,795  & 44,406  & 241,193   \\
DeepSeek-8B   & 11,840  & 12,208  & 22,776  & 31,985  & 31,288  & 53,601  & 24,044  & 35,806  & 223,548   \\
DeepSeek-14B  & 12,206  & 14,640  & 26,689  & 51,691  & 34,635  & 67,941  & 28,964  & 52,977  & 289,743   \\
Gemini-2.0F   & 0       & ---       & 247     & 265     & 1,539   & 1,601   & 368     & 257     & 4,277     \\
Gemini-2.5P   & 12,918  & 22,898  & 20,798  & 19,466  & 19,368  & 19,616  & 19,853  & 19,723  & 154,640   \\
Gemma3-27B    & 304     & 169     & 129     & 119     & 105     & 117     & 149     & 274     & 1,366     \\
GPT-4o        & 1       & 3       & 0       & 1       & 0       & 3       & 2       & 0       & 10        \\
Llama-3.3     & 0       & 0       & 1,993   & 1,815   & 2,928   & 3,826   & 1,535   & 1,403   & 13,500    \\
o4-mini       & 10      & 212     & 2,110   & 2,476   & 2,187   & 2,956   & 2,285   & 2,646   & 14,882    \\
Phi-4         & 5       & 419     & 2,472   & 1,755   & 2,418   & 3,750   & 4,205   & 3,246   & 18,270    \\
Phi-4(R)      & 89,822  & 60,951  & 130,431 & 205,138 & 135,132 & 251,805 & 140,195 & 189,997 & 1,203,471 \\
Qwen3-14B     & 9,318   & 12,258  & 38,859  & 68,919  & 52,553  & 95,254  & 42,774  & 69,758  & 389,693   \\
Qwen3-32B     & 10,437  & 15,313  & 35,728  & 75,511  & 41,016  & 101,187 & 37,836  & 71,130  & 388,158   \\
\midrule
Total         & 163,690 & 161,860 & 349,392 & 621,324 & 412,215 & 844,745 & 367,235 & 605,063 & 3,525,524 \\
\bottomrule
\end{tabular}
\begin{tablenotes}[flushleft]
\item[] \textit{Notes: The table reports the number of reasoning sentences by model and treatment. API-based reasoning models (i.e., Claude-4S, Gemini-2.5P, and o4-mini) provide reasoning summaries instead of full Chain-of-Thought outputs. Non-reasoning models (i.e., Claude-3.5H, Gemini-2.0F, Gemma3-27B, GPT-4o, Llama-3.3, and Phi-4) sometimes provide such summaries.}
\end{tablenotes}
\end{threeparttable}
}
\end{table}

\vspace{1em} 
\begin{figure}[h!]
\caption{Agreement rate by treatment and reasoning capability}
\label{fig:self_agreement_rate_reasoning}
\centering
\includegraphics[width=.8\textwidth]{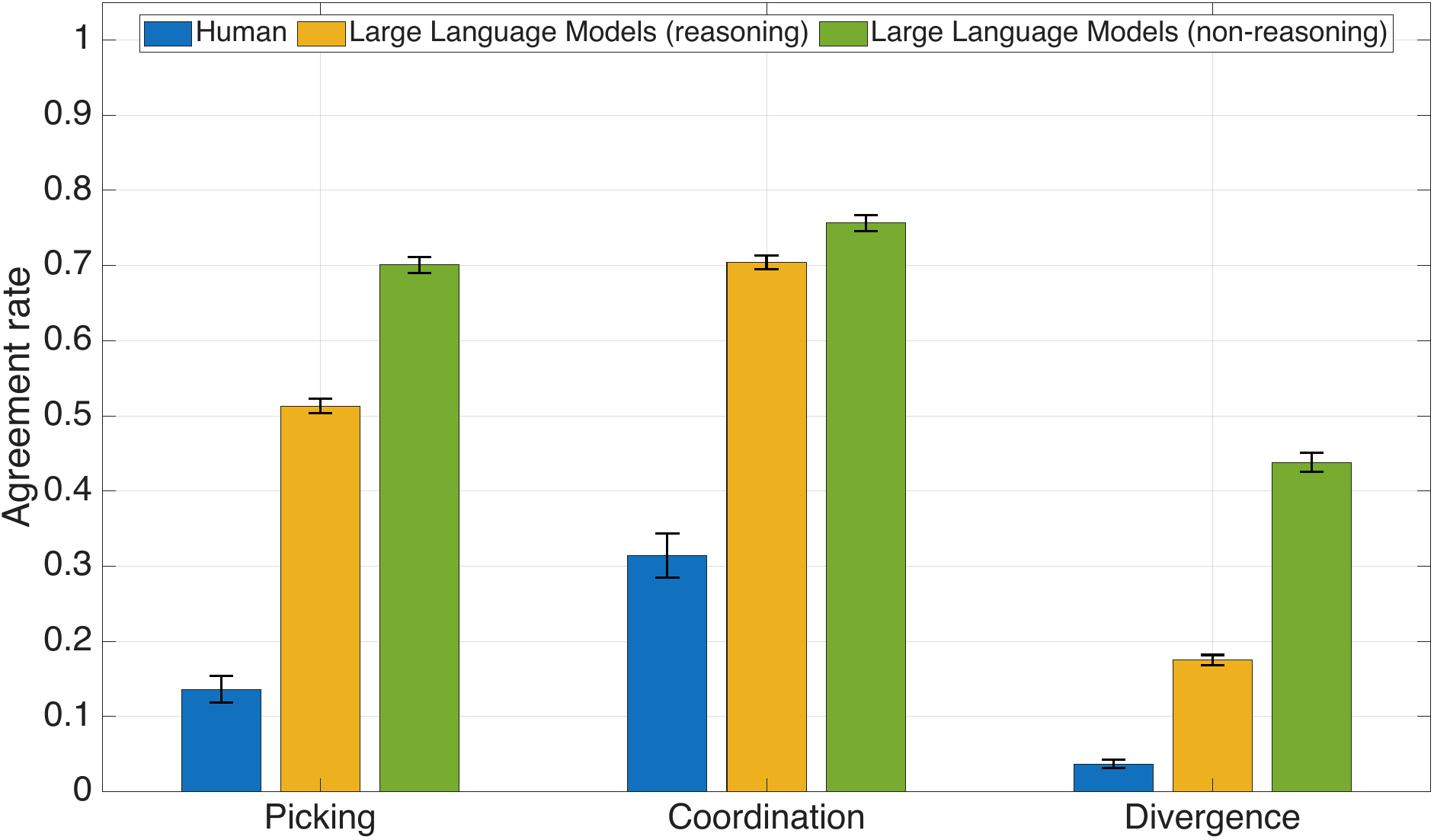}
\vspace{.5mm}
\caption*{{\footnotesize Notes: Bars report the average self-agreement rate across topics by treatment and reasoning capability. For comparability, we also report human agreement rates. Vertical lines indicate 95\% bootstrapped confidence intervals. All pairwise comparisons across treatment arms within subject type are statistically significant  ($p<0.05$ using a two-sided bootstrap test). Comparisons between humans and LLMs from either reasoning category are statistically significant across all treatment arms ($p<0.05$).}}
\end{figure}

\vspace{1em} 
\begin{figure}[h!]
\caption{Agreement rate of LLMs under different temperature values by treatment (all model pairs)}
\label{fig:agreement_rate_temperature_all_pairs}
\centering
\includegraphics[width=.8\textwidth]{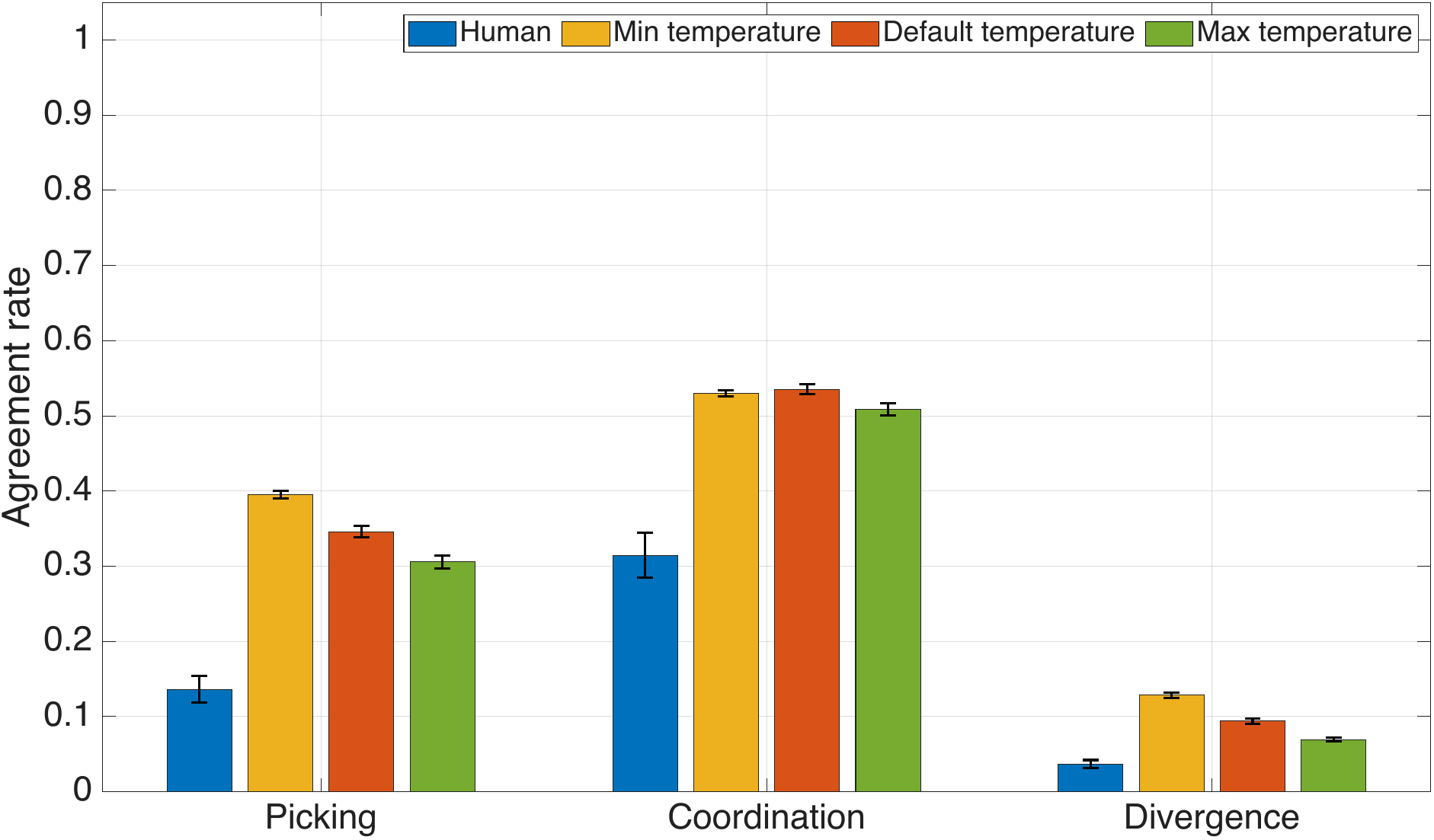}
\vspace{.5mm}
\caption*{{\footnotesize Notes: 
Bars represent the average agreement rates across topics for LLM subjects under different temperatures, for each treatment arm. LLM agreement rates are the average over all pairs of LLMs. For comparability, the default temperature average is calculated only over models that were not excluded, and human agreement rates are also displayed. Vertical lines indicate 95\% bootstrapped confidence intervals. All pairwise comparisons, both within and between subject types, are statistically significant ($p<0.05$ using a two-sided bootstrap test), except for the comparison between LLM subjects at the lowest and default temperatures in the coordination arm.}}
\end{figure}

\begin{figure}[h!]
\caption{Effect of information on agreement rate by reasoning capability and condition}
\label{fig:belief_effect_reasoning}
    \centering
    \includegraphics[width=.8\linewidth]{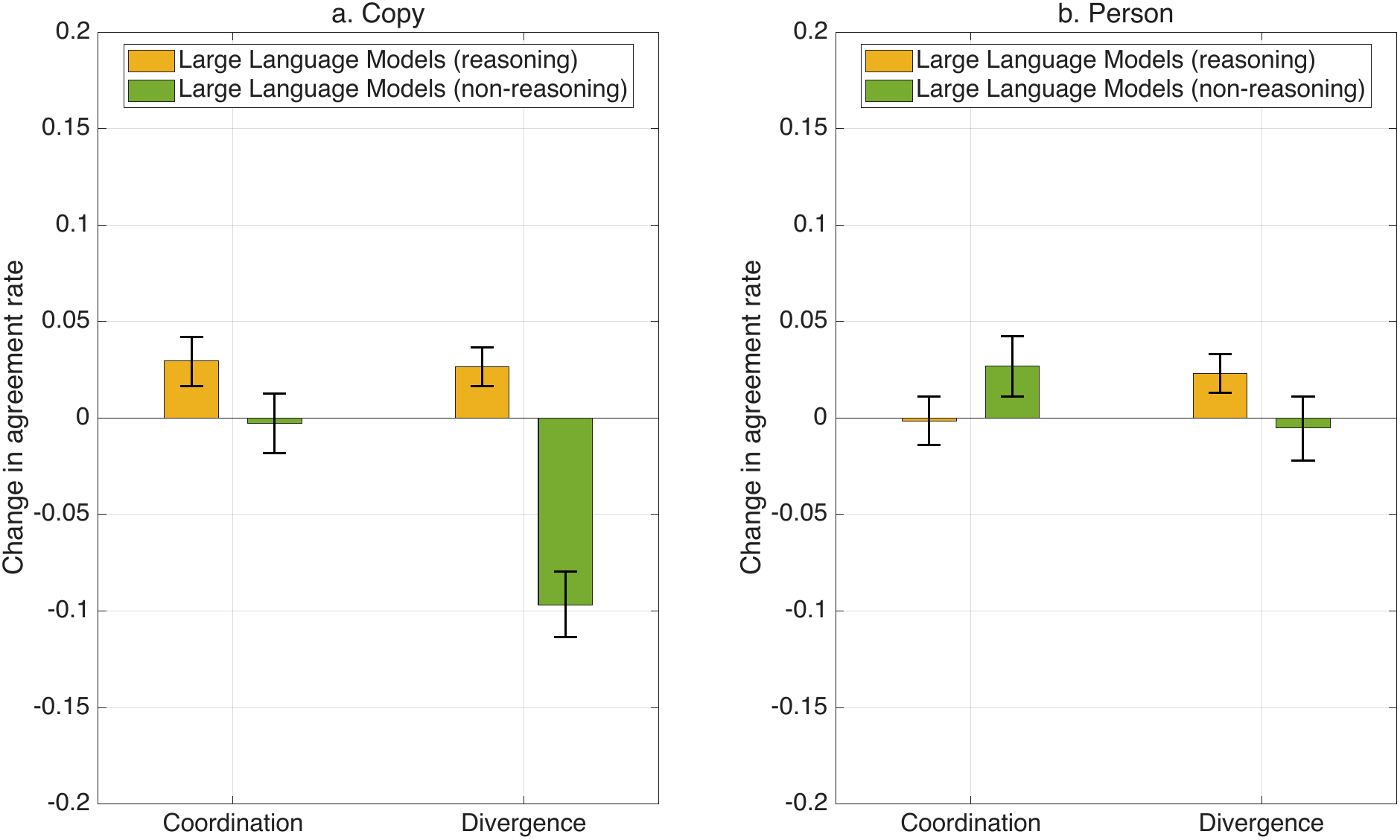}
\vspace{.01mm}
\caption*{\footnotesize Notes: Bars represent the estimated effects of the copy (Panel a) and person (Panel b) conditions on self-agreement rates in the coordination and divergence arms. Estimates are obtained from the linear regression specified in \cref{eq:information_ols}, estimated separately for reasoning and non-reasoning models. Vertical lines indicate 95\% bootstrapped confidence intervals.}
\end{figure}

\begin{figure}[h!]
\centering
\caption{Textual analysis of strategic reasoning  by treatment and condition}
\label{fig:textual_copy_person}
\begin{subfigure}[t]{\textwidth}
\centering
\caption{Distribution of screened sentences}
\includegraphics[width=.8\textwidth]{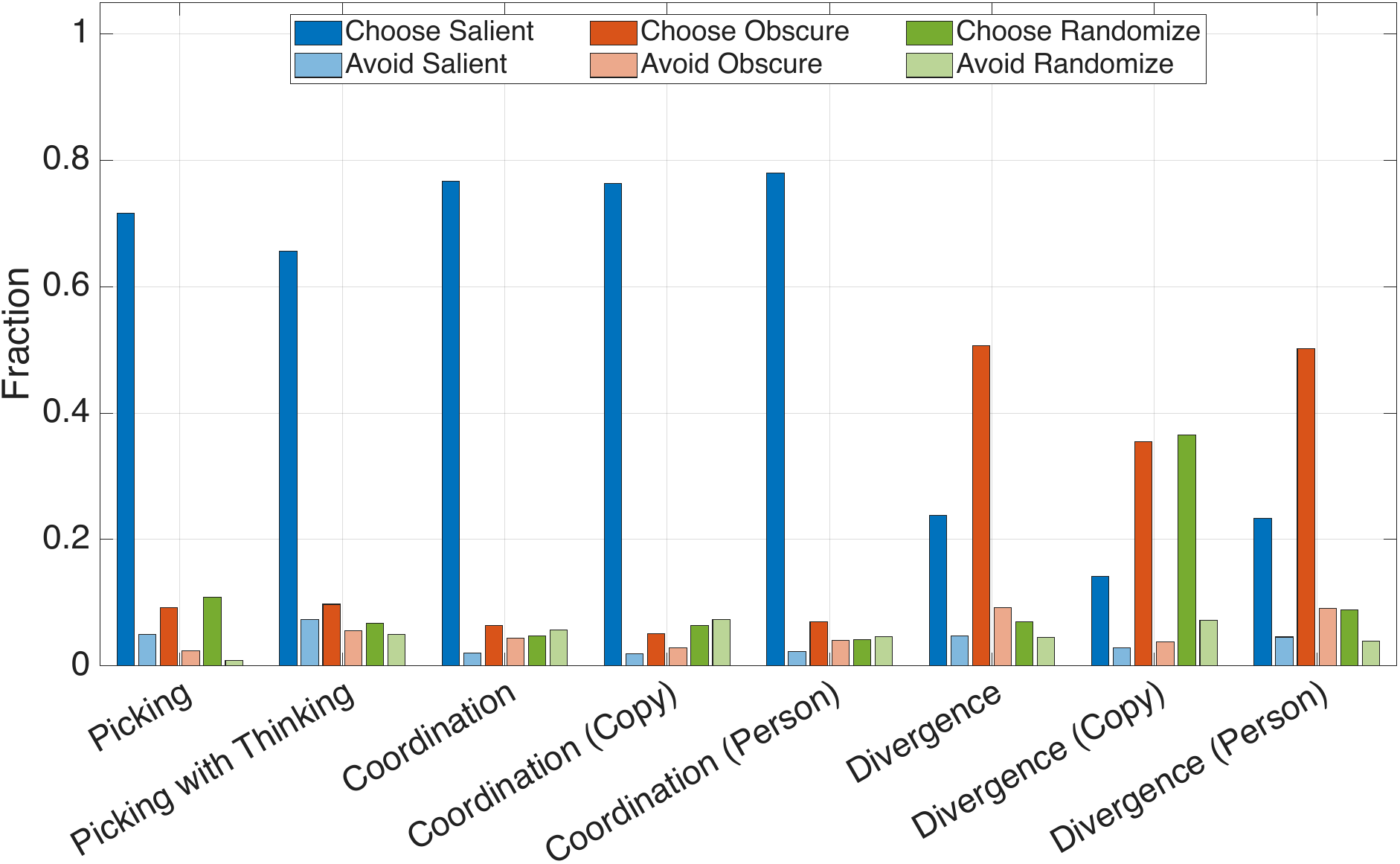}

\end{subfigure}

\vspace{0.8em}

\begin{subfigure}[t]{.8\textwidth}
\centering
\caption{LLM-judge classification of strategic reasoning}
\includegraphics[width=\textwidth]{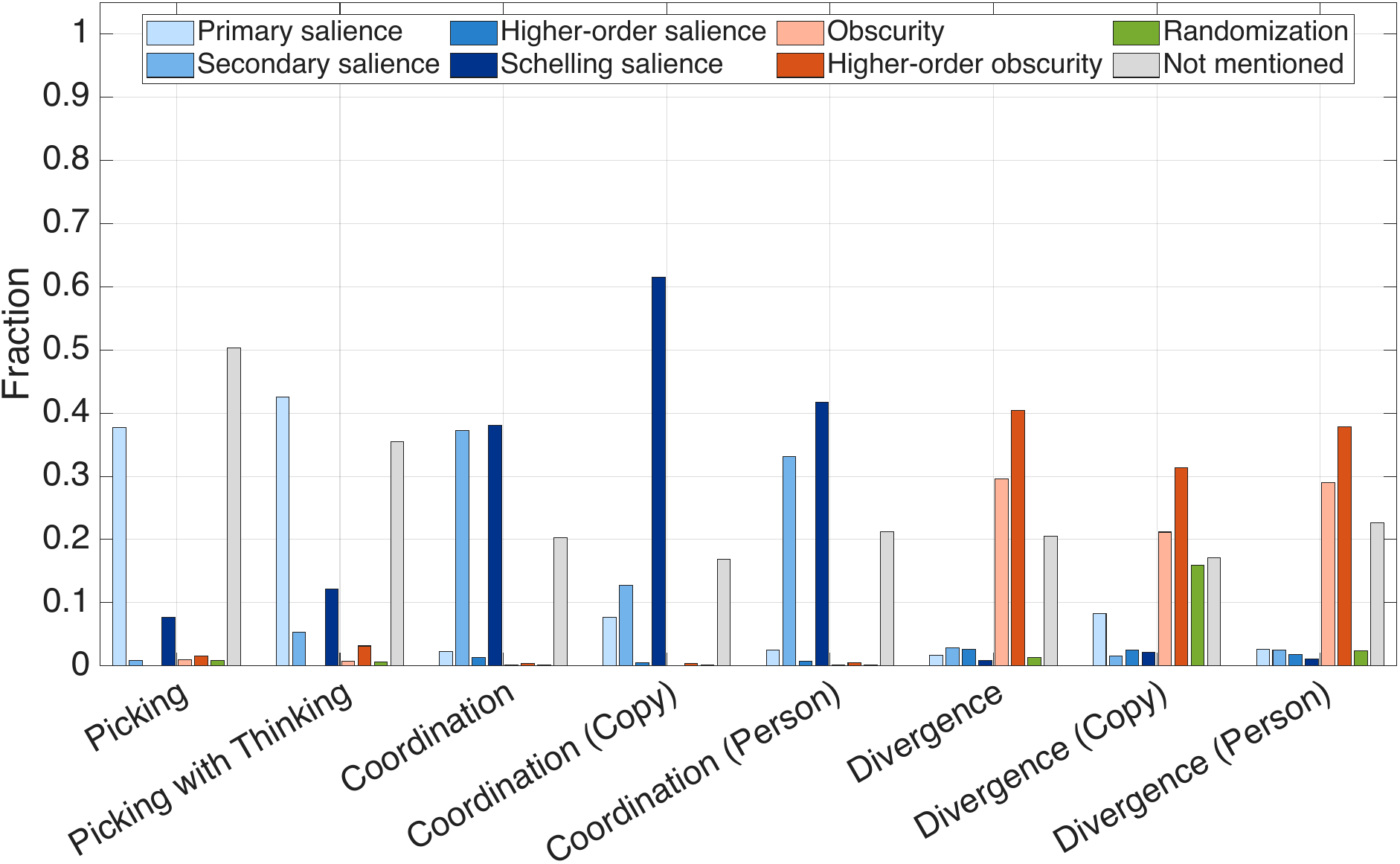}

\end{subfigure}

\vspace{0.5mm}

\caption*{\footnotesize
Notes: Panel (a) displays the distribution of screened sentences' closest reference vector by treatment and condition. Panel (b) reports the share of LLM-reasoning texts that fall under each of the eight strategic reasoning categories according to the LLM-judge's determination. The judge is implemented using Gemini-2.5-Flash.  For details, see \cref{app:CoT_textual_analysis_judge_prompt}.
}

\end{figure}

\begin{figure}[h!]
\caption{LLM-judge classification of beliefs about co-player identity (by treatment and condition)}
\label{fig:beliefs judge}
\centering
\includegraphics[width=.8\textwidth]{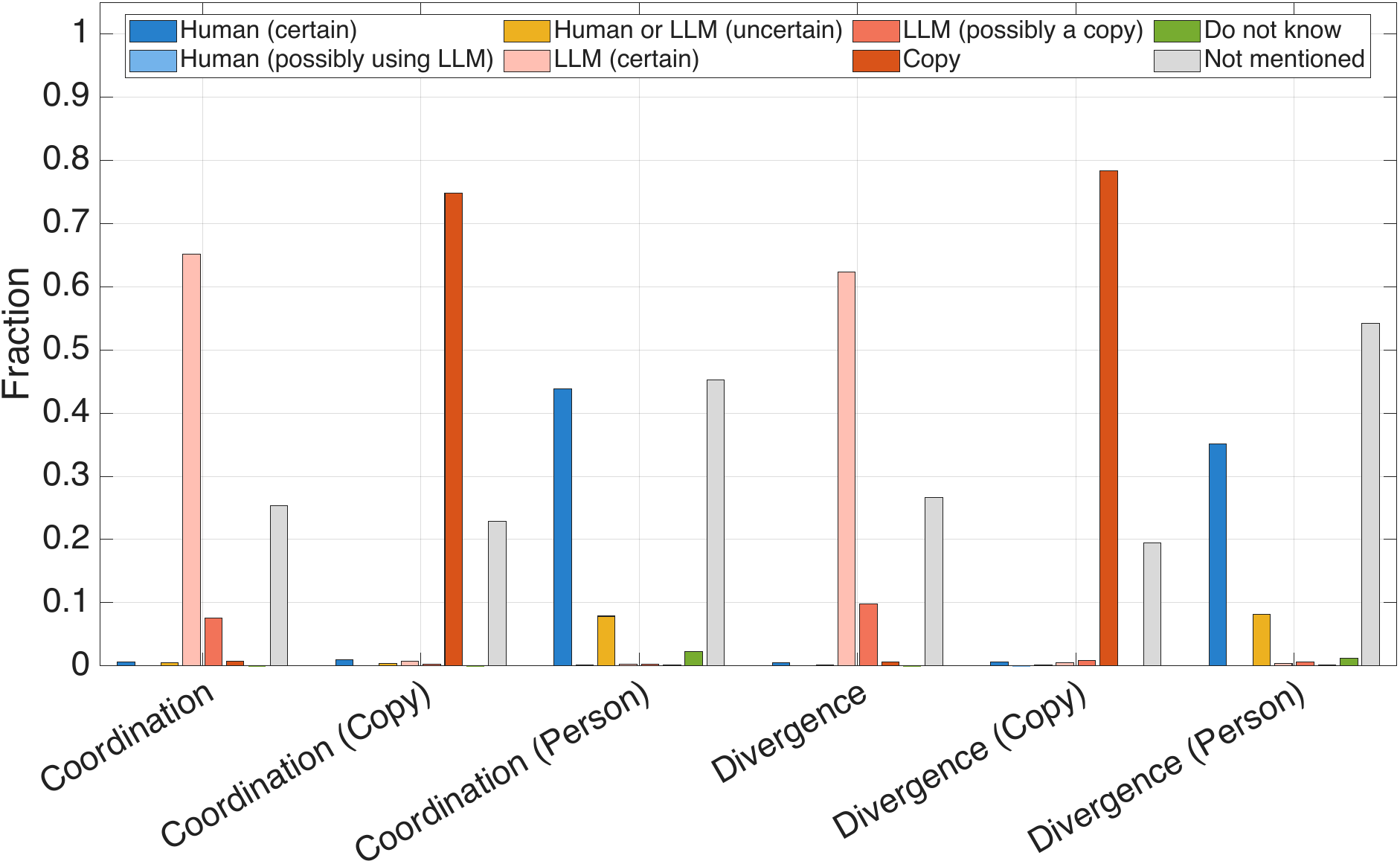}
\vspace{.5mm}
\caption*{{\footnotesize Notes: The figure reports the share of LLM-reasoning texts that fall under each of the eight co-player identity categories according to the LLM-judge's determinations. Categories distinguish between different identities and degrees of certainty. The judge is implemented using Gemini-2.5-Flash. For details, see \cref{app:CoT_textual_analysis_judge_prompt}.}}
\end{figure}

\begin{figure}[h!]
\caption{Agreement rate of LLMs with personas by treatment (all model pairs)}
\label{fig:persona_all_pair_models}
    \centering
    \includegraphics[width=.8\linewidth]{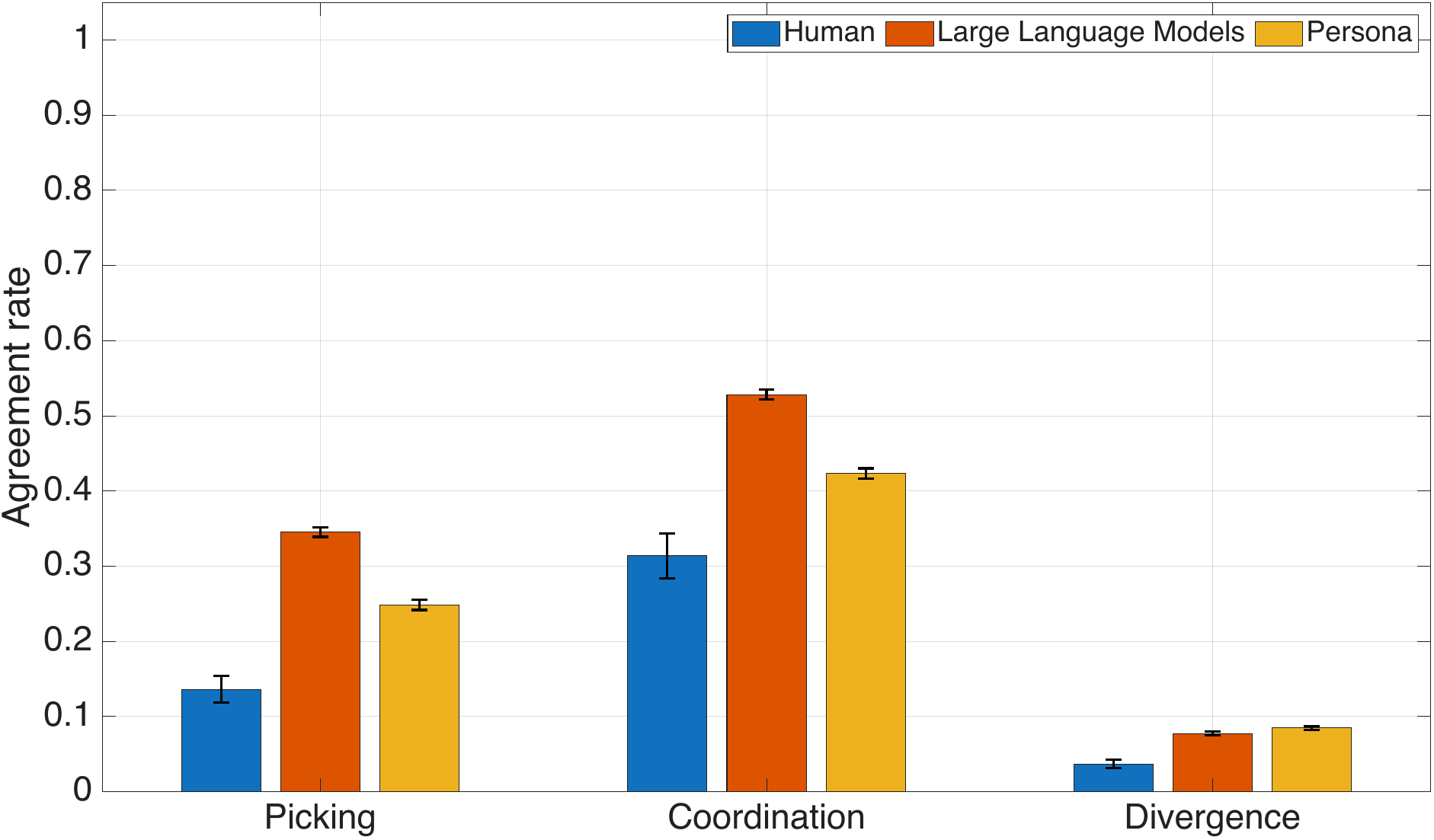}
\vspace{.01mm}
\caption*{\footnotesize Notes: Bars report the average agreement rate across topics for human and LLM subjects by treatment.
LLM agreement rates are the average over all pairs of models included in the analysis. Vertical lines indicate 95\% bootstrapped confidence intervals. All pairwise comparisons across treatment arms are statistically significant ($p<0.05$ using a two-sided bootstrap test). Comparisons between LLMs with personas and humans in all treatment arms are statistically significant ($p<0.05$).}
\end{figure}

\begin{figure}[h]
\caption{LLM-persona pairwise agreement rates by treatment}
\label{fig:pairwise_heatmap_persona}
\centering
\includegraphics[width=.9\textwidth]{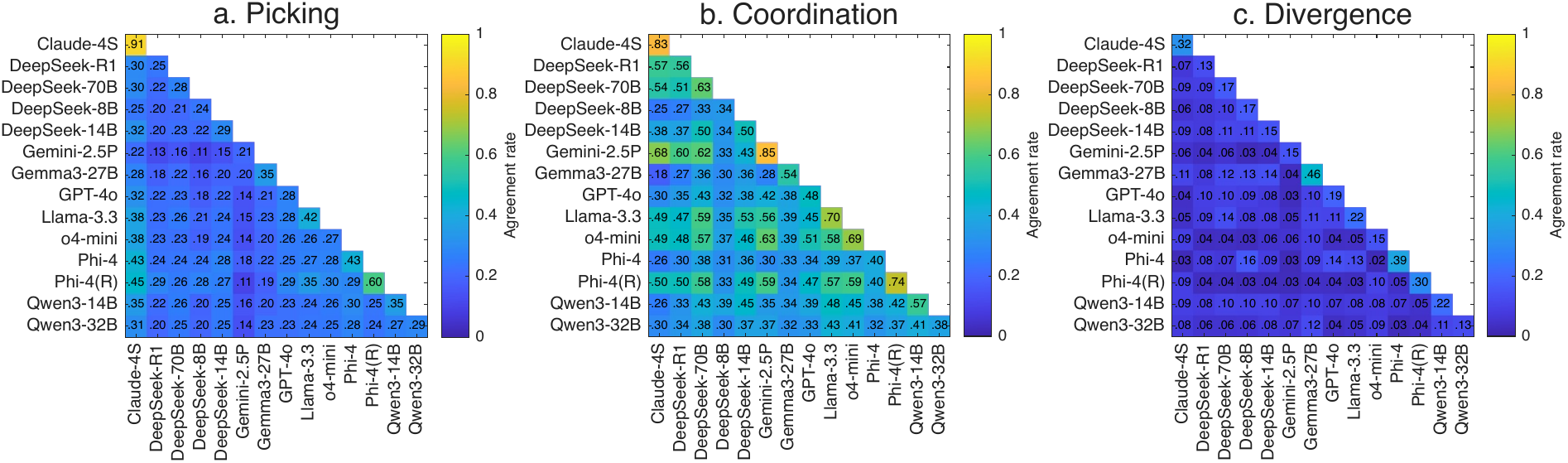}
\caption*{{\footnotesize Notes: The figure displays a heatmap of pairwise agreement rates for all pairs of LLM-persona subjects.  Diagonal cells correspond to self-agreement rates. Lighter (yellow) colors indicate higher agreement, and darker (blue) colors indicate lower agreement.}}
\end{figure}

\clearpage

\setcounter{figure}{0}
\renewcommand{\thefigure}{B.\arabic{figure}}  

\setcounter{table}{0} 
\renewcommand{\thetable}{B.\arabic{table}} 

\section{Further Analysis}\label{app:detailed_analysis}

\begin{table}[!ht]
\centering
\caption{Out-of-format and excluded responses by model}
\label{tab:invalid_responses}
\renewcommand{\arraystretch}{1.5}
\resizebox{.85\textwidth}{!}{%
\begin{threeparttable}
\setlength{\tabcolsep}{30pt}
\begin{tabular}{lcccc}
\toprule
\toprule
 & \multicolumn{2}{c}{Out-of-format}  & \multicolumn{2}{c}{Excluded}\\
\cmidrule(lr){2-3} \cmidrule(lr){4-5}
Model & Total & Fraction & Total & Fraction \\
\midrule
Claude-3.5H & 0 & 0.00 & 0 & 0.00 \\
Claude-4S & 0 & 0.00 & 0 & 0.00 \\
DeepSeek-R1 & 0 & 0.00 & 0 & 0.00 \\
DeepSeek-70B & 173 & 0.04 & 133 & 0.03 \\
DeepSeek-8B & 702 & 0.17 & 195 & 0.05 \\
DeepSeek-14B & 595 & 0.14 & 147 & 0.04 \\
Gemini-2.0F & 118 & 0.03 & 118 & 0.03 \\
Gemini-2.5P & 2 & 0.00 & 2 & 0.00 \\
Gemma3-27B & 30 & 0.01 & 0 & 0.00 \\
GPT-4o & 150 & 0.04 & 150 & 0.04 \\
Llama-3.3 & 6 & 0.00 & 6 & 0.00 \\
o4-mini & 0 & 0.00 & 0 & 0.00 \\
Phi-4 & 445 & 0.11 & 9 & 0.00 \\
Phi-4(R) & 455 & 0.11 & 383 & 0.09 \\
Qwen3-14B & 54 & 0.01 & 47 & 0.01 \\
Qwen3-32B & 107 & 0.03 & 91 & 0.02 \\
\midrule
Total & 2,837 & 0.04 & 1,281 & 0.02\\
\bottomrule
\end{tabular}

\begin{tablenotes}[para]
\footnotesize
\item \justifying \textit{Notes: The table reports the number and fraction of out-of-format and excluded responses by model, pooling responses across the three treatment arms of the main experiment as well as the arms using the copy and person conditions (\Cref*{subsec:information}).}  
\end{tablenotes}
\end{threeparttable}}
\end{table}

\begin{table}[ht]
\centering
\caption{Average number of valid answers generated by model and topic}
\label{tab:list_generation}
\renewcommand{\arraystretch}{1.5}
\setlength{\tabcolsep}{6pt}
\resizebox{.9\textwidth}{!}{%
\begin{threeparttable}
\begin{tabular}{lcccccccccccc}
\toprule \toprule
Model & \makecell{Athlete} & \makecell{Car} & City & Color & Disease & Flower & \makecell{Food} & \makecell{Geometric} & Letter & Month & \makecell{Positive\\Number} & \makecell{Trait} \\
\midrule
Claude-4S & 99.82 & 98.84 & 99.73 & 98.54 & 97.90 & 98.14 & 99.48 & 96.86 & 26.00 & 12.00 & 100.00 & 99.48 \\
DeepSeek-R1 & 98.90 & 97.54 & 99.98 & 99.46 & 99.40 & 99.62 & 99.68 & 98.02 & 26.00 & 12.00 & 100.00 & 99.54 \\
DeepSeek-70B & 96.78 & 91.74 & 99.02 & 90.84 & 96.65 & 84.47 & 97.66 & 76.60 & 26.00 & 12.00 & 99.98 & 97.00 \\
DeepSeek-8B & 86.24 & 96.19 & 99.49 & 82.42 & 93.79 & 77.39 & 97.51 & 70.19 & 25.12 & 12.00 & 100.00 & 75.88 \\
DeepSeek-14B & 57.47 & 87.24 & 98.22 & 84.21 & 94.04 & 82.18 & 98.44 & 83.70 & 26.00 & 12.00 & 100.00 & 92.45 \\
Gemini-2.5P & 99.96 & 99.10 & 99.84 & 99.14 & 99.68 & 99.94 & 99.72 & 98.28 & 26.00 & 12.00 & 100.00 & 99.98 \\
Gemma3-27B & 98.52 & 98.06 & 99.94 & 98.54 & 96.56 & 99.60 & 99.72 & 91.20 & 26.00 & 12.00 & 100.00 & 99.98 \\
GPT-4o & 98.42 & 97.96 & 99.98 & 97.50 & 98.74 & 98.32 & 99.28 & 93.38 & 26.00 & 12.00 & 100.00 & 98.20 \\
Llama-3.3 & 99.66 & 95.50 & 99.88 & 95.90 & 95.02 & 94.38 & 99.20 & 90.56 & 26.00 & 12.00 & 100.00 & 96.82 \\
o4-mini & 100.00 & 99.24 & 99.94 & 100.00 & 99.76 & 99.62 & 99.96 & 99.78 & 26.00 & 12.00 & 100.00 & 97.50 \\
Phi-4 & 94.58 & 90.50 & 99.76 & 99.36 & 96.78 & 96.60 & 98.84 & 81.76 & 26.00 & 12.00 & 100.00 & 99.08 \\
Phi-4(R) & 72.19 & 71.14 & 84.86 & 86.76 & 85.78 & 67.86 & 94.67 & 81.84 & 26.00 & 12.00 & 73.21 & 88.19 \\
Qwen3-14B & 93.53 & 94.38 & 99.17 & 98.64 & 97.41 & 98.33 & 99.40 & 96.25 & 26.00 & 12.00 & 100.00 & 98.76 \\
Qwen3-32B & 95.55 & 95.80 & 98.98 & 98.18 & 96.00 & 95.41 & 98.32 & 94.46 & 26.00 & 12.00 & 99.56 & 97.96 \\
\bottomrule
\end{tabular}%

\begin{tablenotes}[flushleft]
\item[] \textit{Notes: The table reports the average number of valid answers produced by each model across topics. For each topic, we prompted each model 50 times to generate lists of 100 answers.
For topics with fewer valid answers (i.e., months and letters of the English alphabet), models were instead instructed to generate all answers. Unparsable lists were excluded from the analysis: Phi-4(R) (77/600 lists, 13\%); Qwen3-32B (40/600  lists, 7\%); Qwen3-14B (39/600  lists, 6\%); DeepSeek-8B (27/600  lists, 4\%); DeepSeek-70B (7/600  lists, 1\%); DeepSeek-14B (5/600  lists, 1\%); Claude-4S (1/600  lists, $<1$\%).}

\end{tablenotes}
\end{threeparttable}
}
\end{table}

\begin{figure}[h!]
\caption{Tab switching behavior (human subjects)}
\label{fig:tab_switching_histogram}
\centering
\includegraphics[width=.75\textwidth]{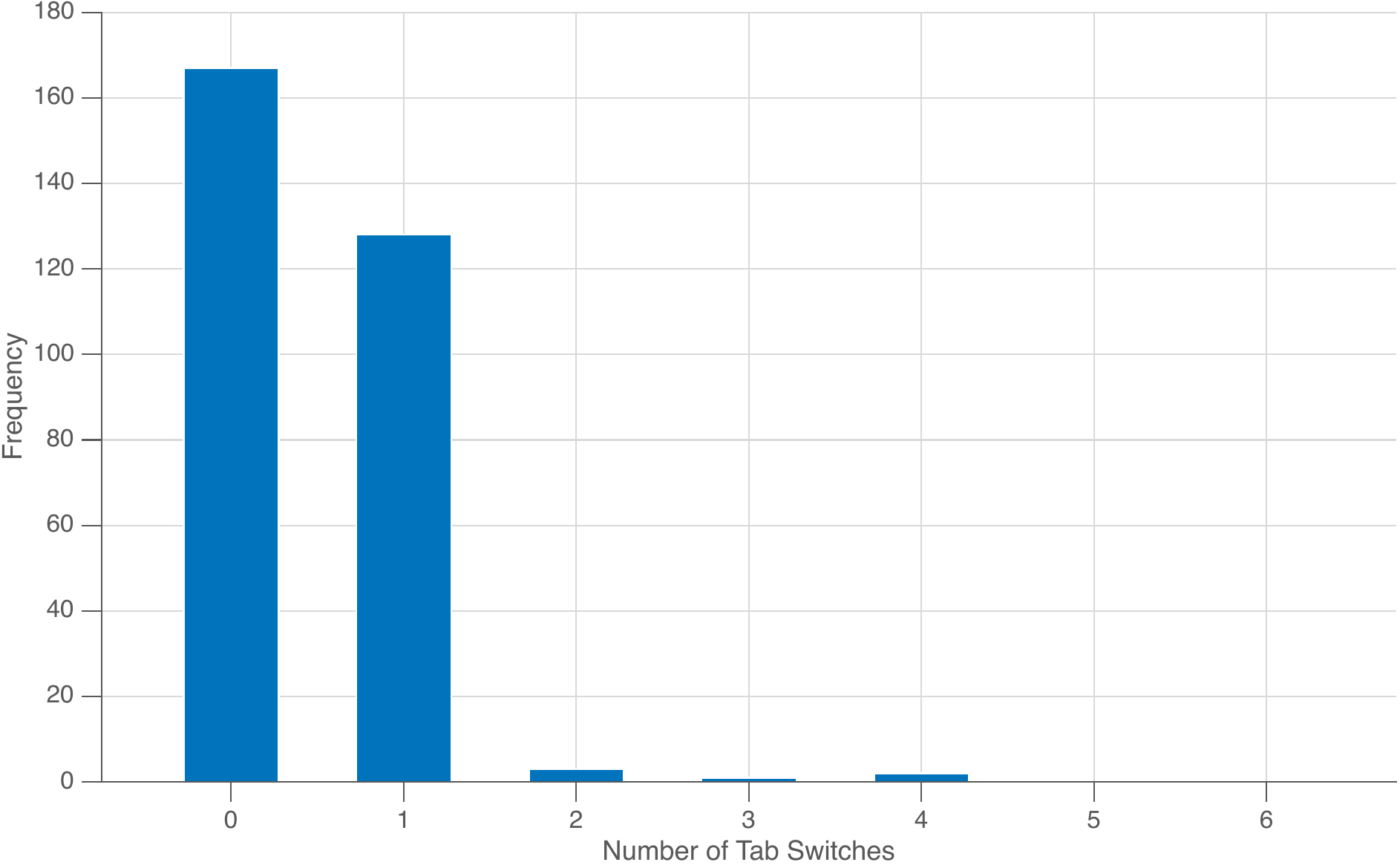}

\caption*{{\footnotesize Notes: The figure displays the distribution of tab switching counts, measured as the number of times a human subject switched browser tabs during the online survey.}}
\end{figure}

\begin{figure}[h!]
\caption{Fraction of topics with different modal response relative to the picking arm, by model and treatment}
\label{fig:share_response_change_relative_baseline}
\centering
\includegraphics[width=\textwidth]{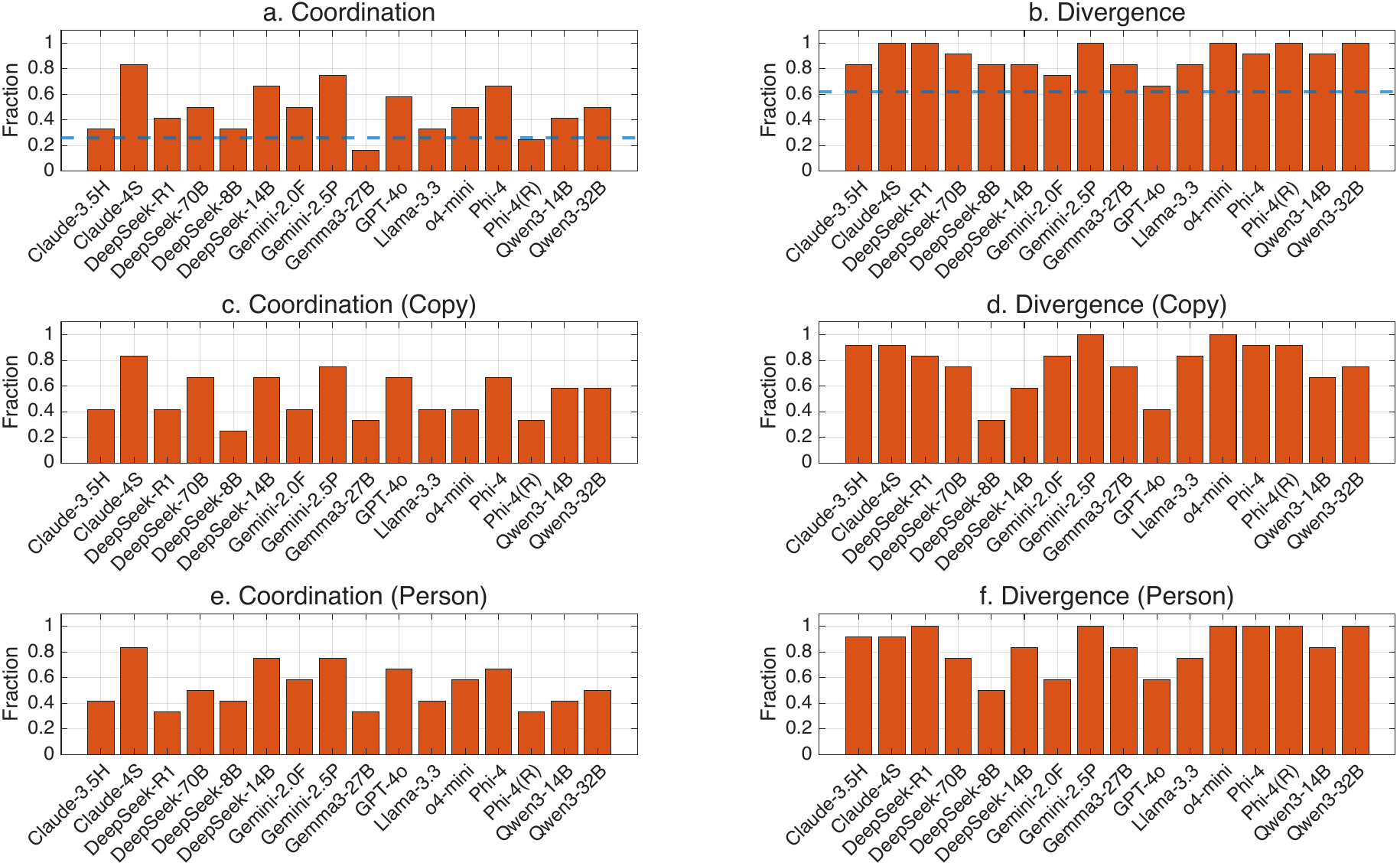}
\vspace{0.5mm}
\caption*{{\footnotesize Notes: This figure reports the fraction of topics in which the modal (most common) response changes relative to the picking treatment for human and LLM subjects, by treatment and condition. Since LLM subjects were queried 50 times, to ensure comparability, we repeatedly sample 50 human responses per treatment arm. Specifically, we draw 1,000 such samples, compute 
the fraction of topics in which the modal response differs from the picking arm, and report the average across these draws to reduce sampling noise. The horizontal dashed blue line reports this average for human subjects.}}
\end{figure}

\begin{figure}[h!]
\caption{Agreement rate by treatment, topic, and subject type}
\label{fig:agreement_rate_topics}
\centering
\includegraphics[width=.8\textwidth]{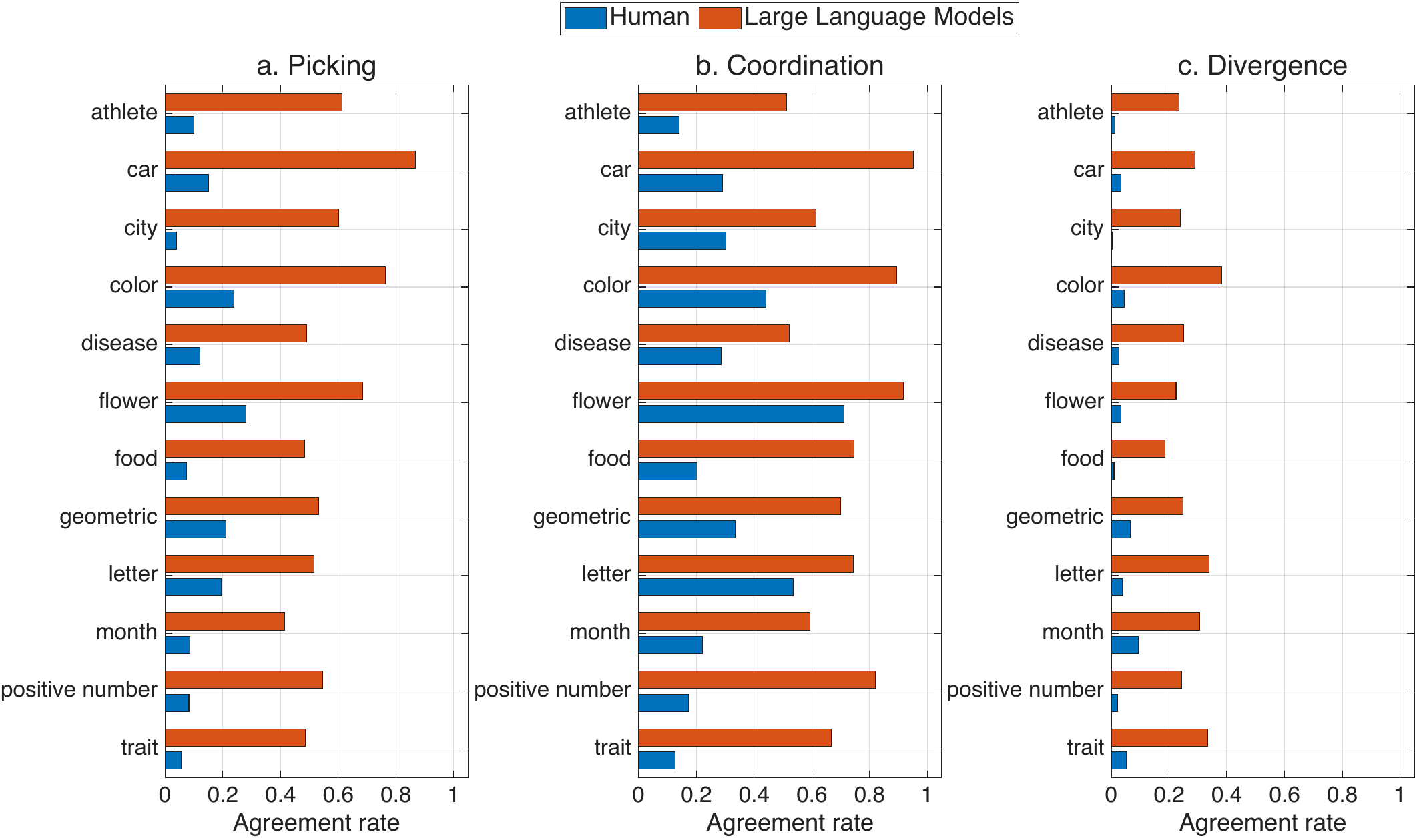}
\vspace{.3cm}
\caption*{{\footnotesize Notes: The figure reports the average agreement rate for each topic, by treatment and subject type. For LLM subjects, we report the average self-agreement rate over all models.}}
\end{figure}

\begin{figure}[h!]
\centering
\caption{Textual analysis of strategic reasoning by treatment and model}
\label{fig:textual_reasoning_model}
\begin{subfigure}[t]{\textwidth}
\centering
\caption{Distribution of screened sentences}
\label{fig:histogram_cot_sentences_disaggregated_6cat}
\includegraphics[width=.8\textwidth]{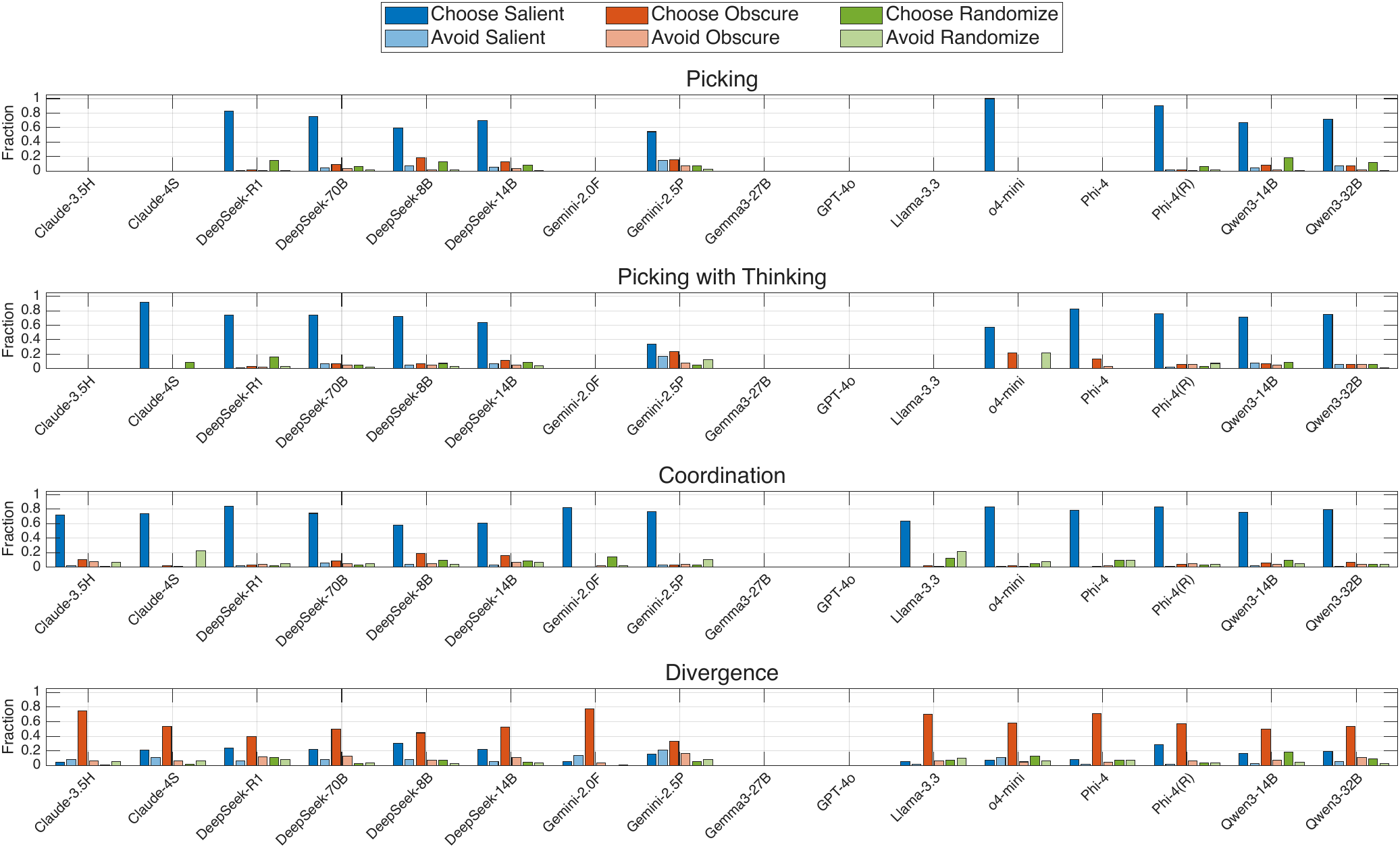}

\end{subfigure}

\vspace{0.8em}

\begin{subfigure}[t]{\textwidth}
\centering
\caption{LLM-judge classification of strategic reasoning}
\label{fig:llm_judge_by_model}
\includegraphics[width=.8\textwidth]{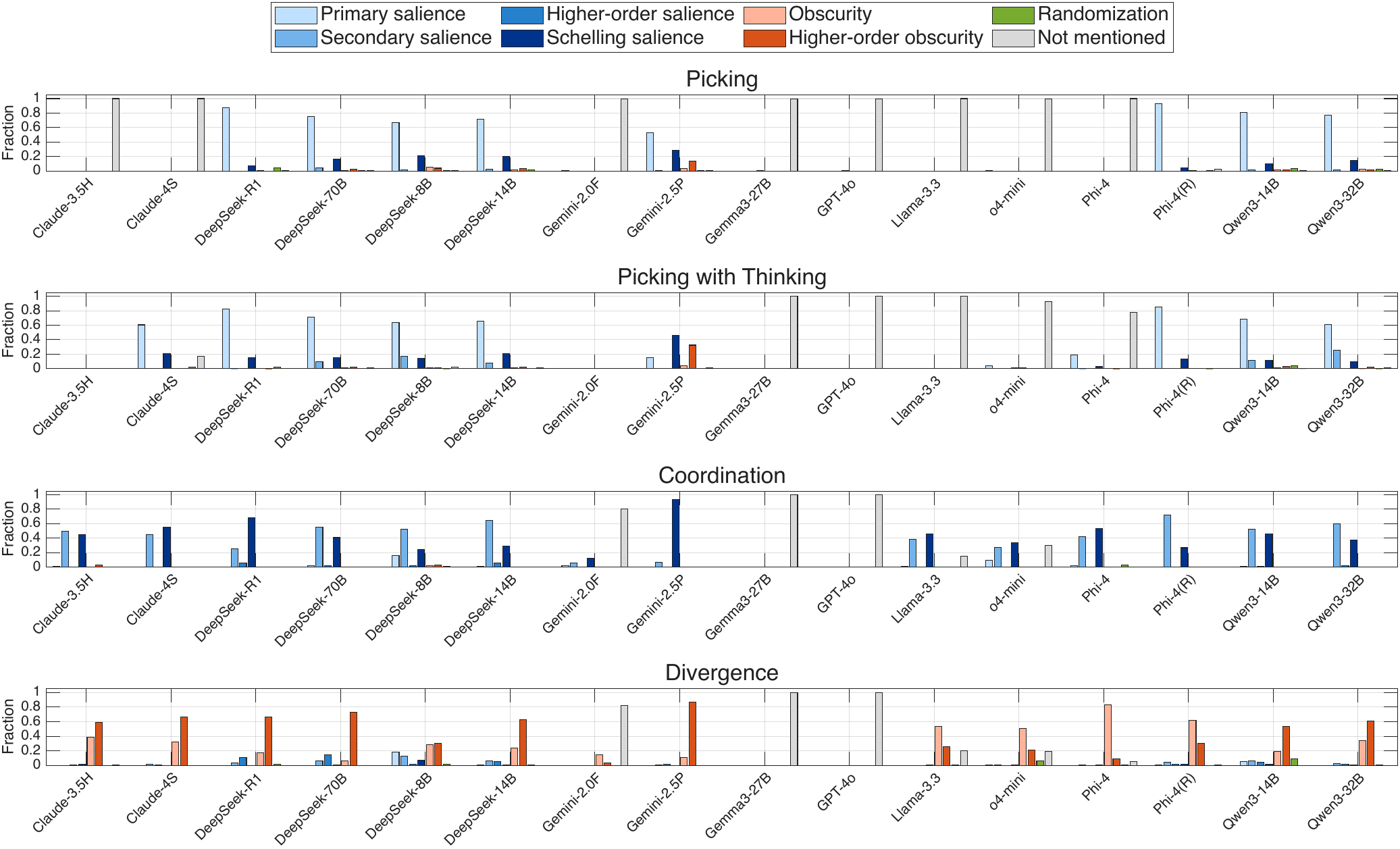}

\end{subfigure}

\vspace{0.5mm}

\caption*{\footnotesize
Notes: Panel (a) displays the distribution of screened sentences closest to each reference vector by treatment. Panel (b) reports the share of LLM-reasoning texts that fall under each of the eight strategic reasoning categories according to the LLM-judge's determination. The judge is implemented using Gemini-2.5-Flash.  For details, see \cref{app:CoT_textual_analysis_judge_prompt}.
}

\end{figure}

 \begin{figure}[h!]
\caption{Pairwise agreement rates for \cref{sec:capabilities}}
\label{fig:agreement_divergence_vs_random}
\centering
\includegraphics[width=0.9\textwidth]{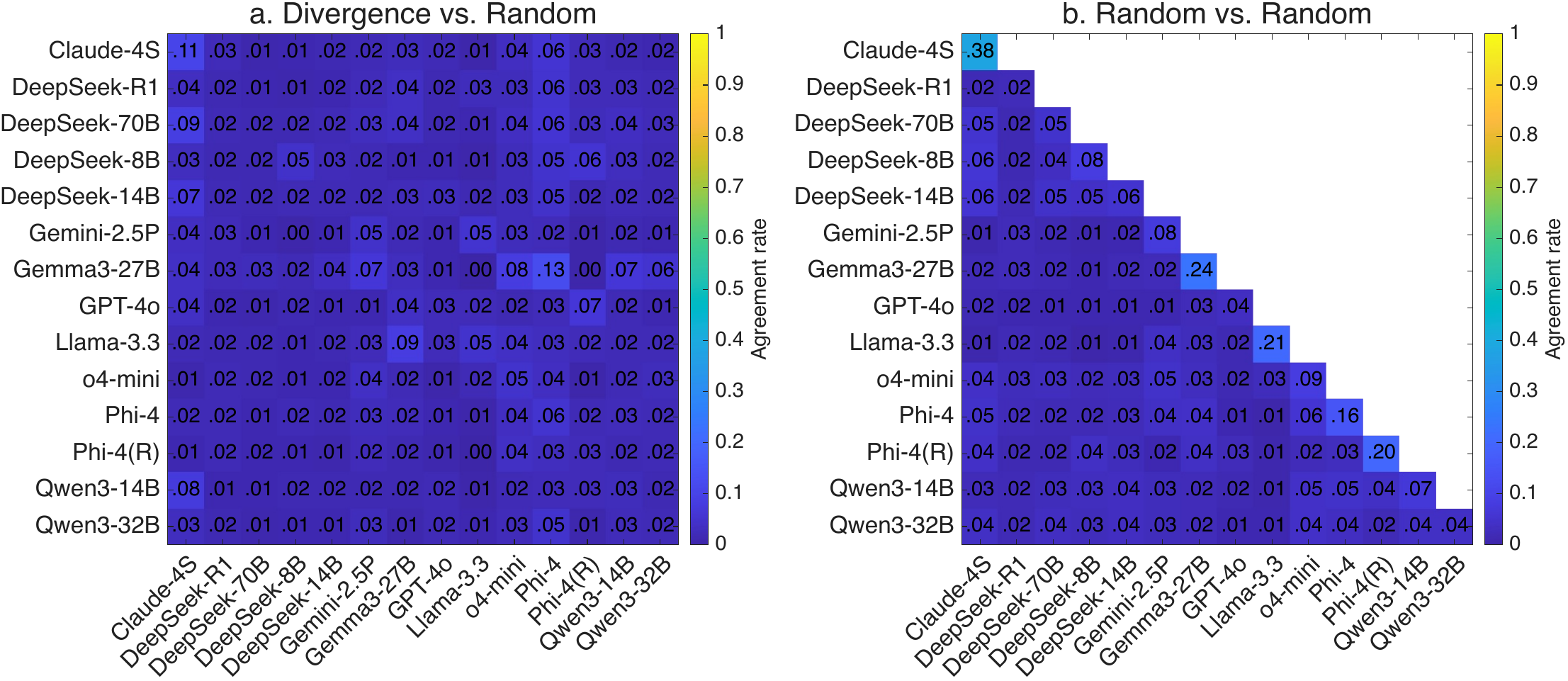}
\vspace{0.2cm}
\caption*{{\footnotesize Notes: The figure displays heatmaps of pairwise agreement rates for all pairs of LLM subjects as described in \cref{sec:capabilities}. Panel~(a) reports agreement rates between models from the divergence arm (rows) and models from the random arm (columns). Panel~(b) reports agreement rates between pairs of models from the random arm. Diagonal cells correspond to the same model. Lighter (yellow) colors indicate higher agreement
and darker (blue) colors indicate lower agreement.}}
\end{figure}

\begin{figure}[h]
\caption{Agreement rate by treatment, model, and  temperature}
\label{fig:temperature_analysis_heterogeneity}
\centering
\includegraphics[width=.8\textwidth]{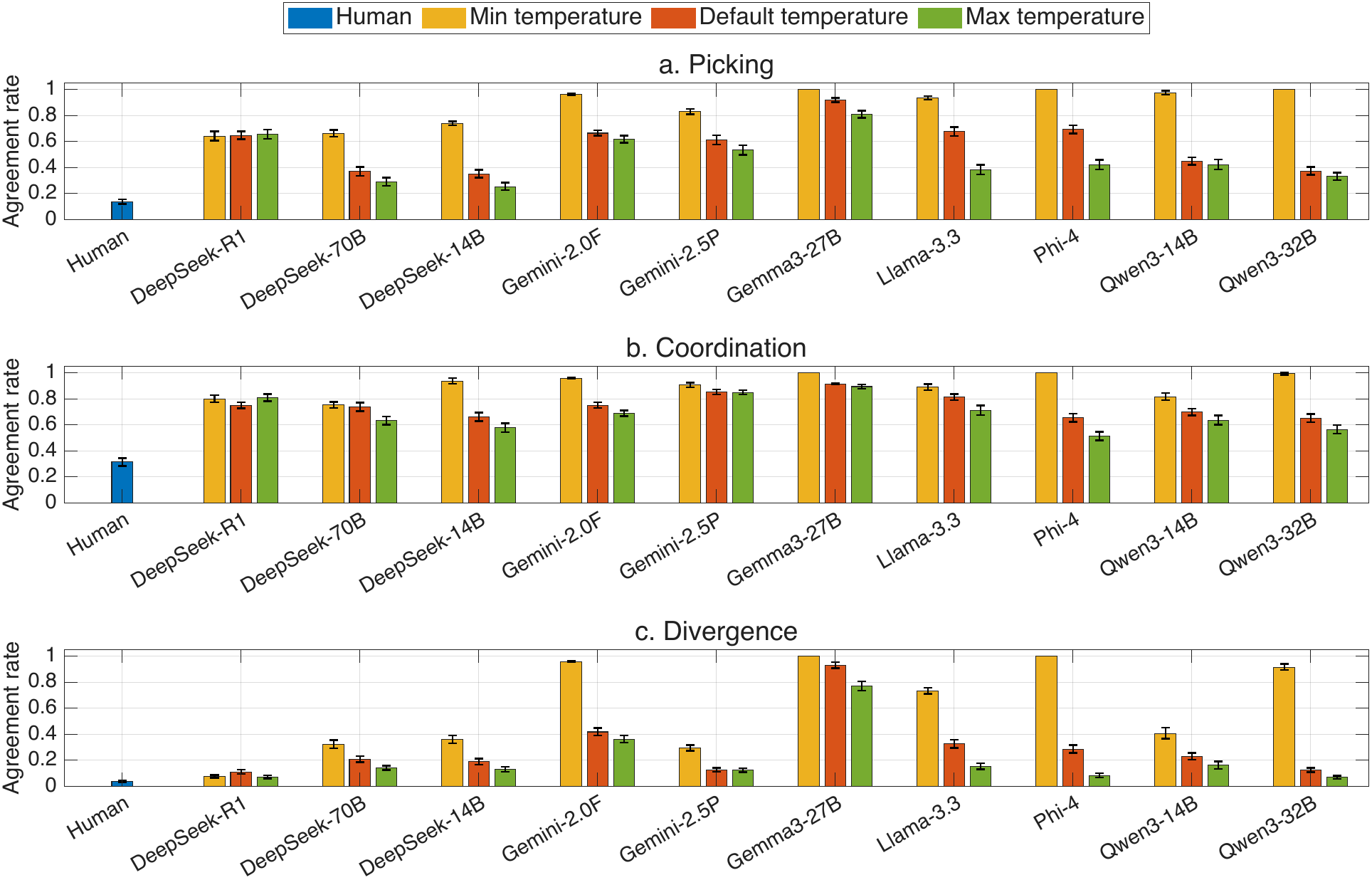}
\vspace{0.2cm}
\caption*{{\footnotesize Notes: Bars represent the average self-agreement rates across topics for human and LLM subjects, by treatment and temperature level. Vertical lines indicate 95\% bootstrapped confidence intervals.}}
\end{figure}

\begin{figure}[h!]
\caption{Effect of information on agreement rate by model, treatment, and condition}
\label{fig:beliefs_effect_by_model}
\centering
\includegraphics[width=.8\textwidth]{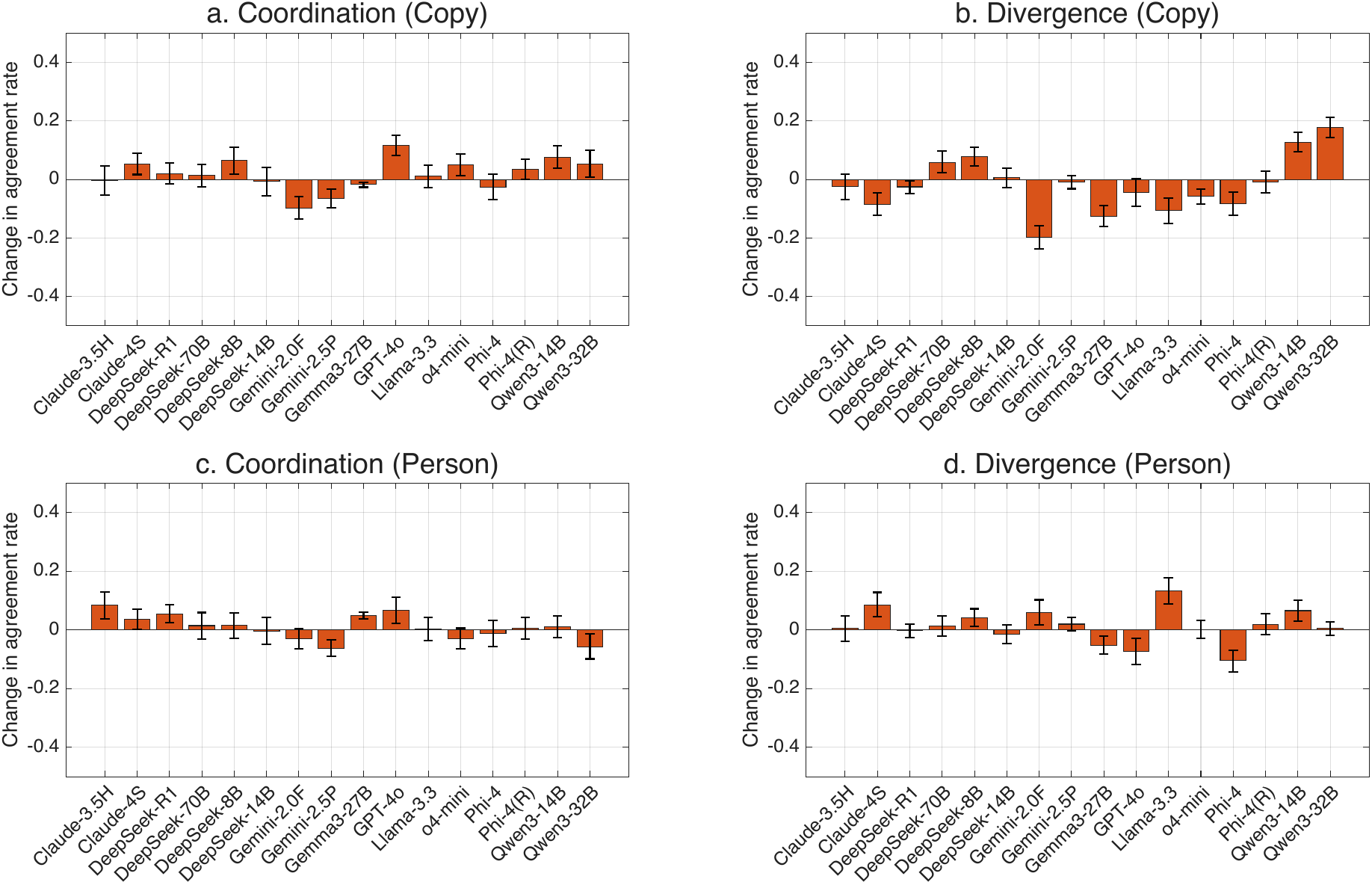}
\vspace{.5mm}
\caption*{{\footnotesize Notes: Bars represent the estimated effects of the copy and person conditions on self-agreement rates in the coordination and divergence arms. Estimates are obtained from the linear regression specified in \cref{eq:information_ols}, estimated separately for each model. The top panels correspond to the copy condition and the bottom panels to the person condition. The left-hand side panels correspond to the coordination treatment arm, and the right-hand side panels to the divergence treatment arm. Vertical lines indicate 95\% bootstrapped confidence intervals.}}
\end{figure}

\begin{figure}[h!]
\centering
\caption{Textual analysis of strategic reasoning by treatment, condition, and model}
\label{fig:textual_copy_person_by_model}
\begin{subfigure}[t]{.9\textwidth}
\centering
\caption{Distribution of screened sentences}
\includegraphics[width=\textwidth]{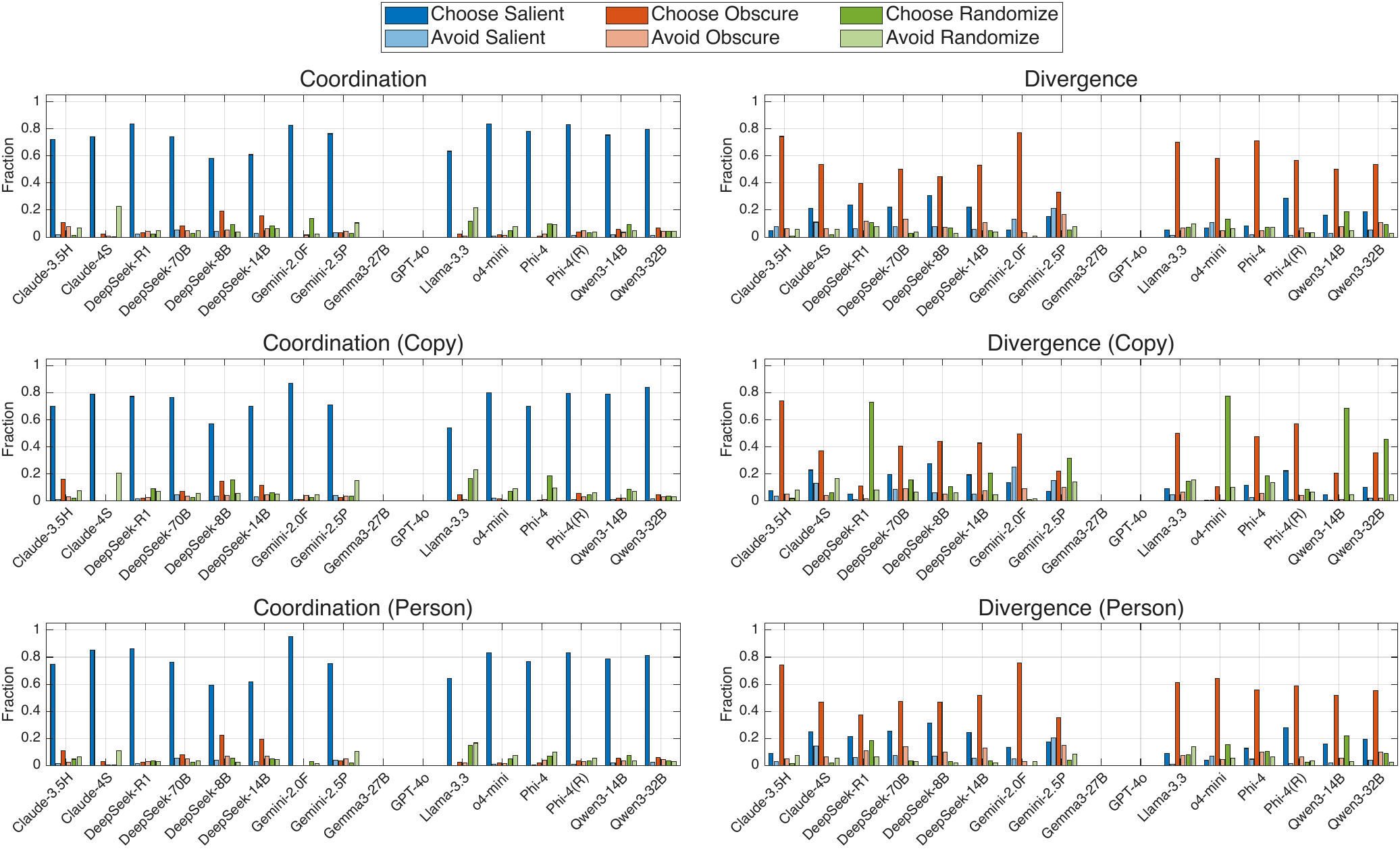}

\end{subfigure}

\vspace{0.8em}

\begin{subfigure}[t]{.9\textwidth}
\centering
\caption{LLM-judge classification of strategic reasoning}
\includegraphics[width=\textwidth]{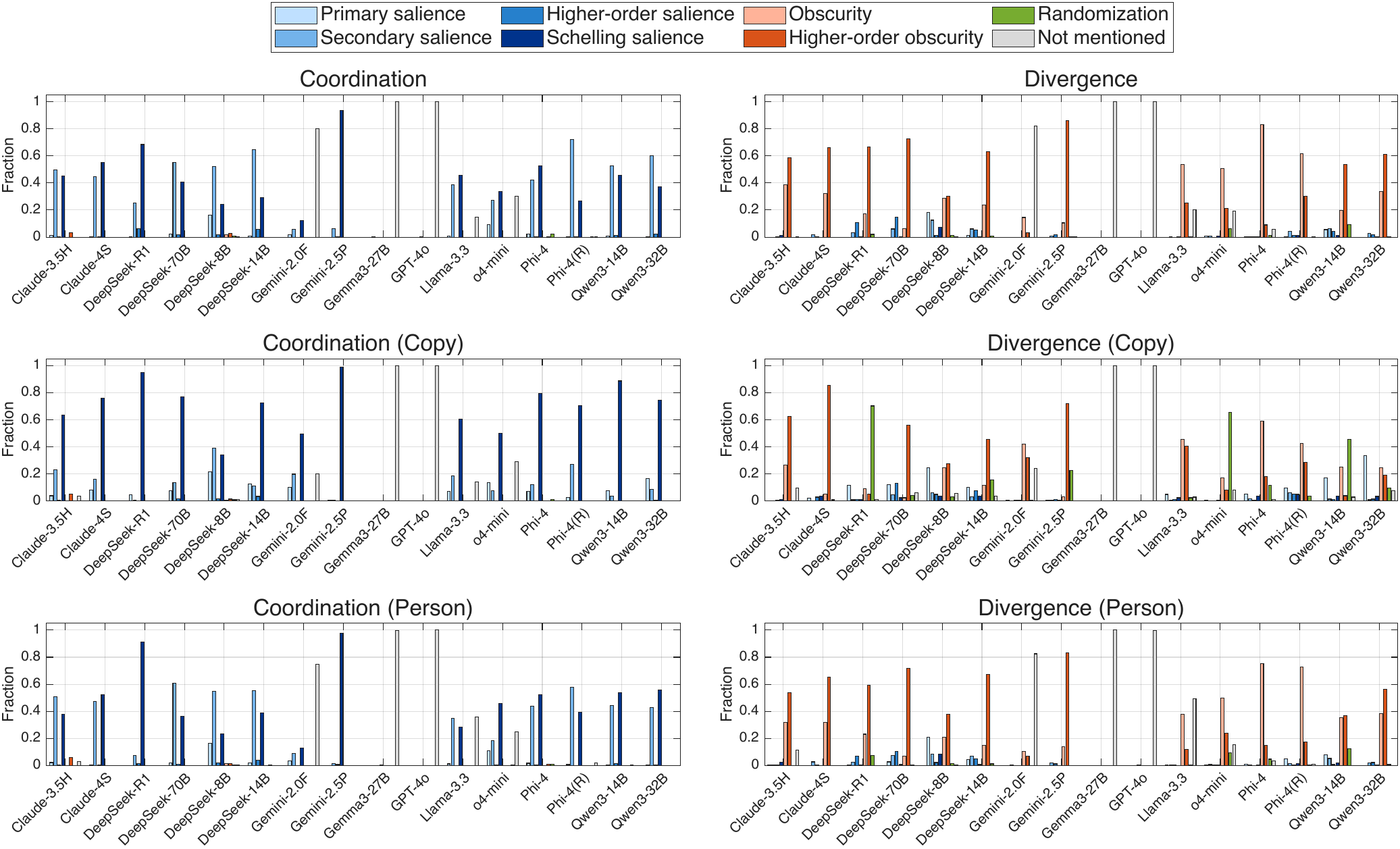}

\end{subfigure}

\vspace{0.2cm}

\caption*{\footnotesize
Notes: Panel (a) displays the distribution of the screened sentences' closest reference vector by model. Panel (b) reports the share of LLM-reasoning texts that fall under each of the eight strategic reasoning categories according to the LLM-judge's determination. The judge is implemented using Gemini-2.5-Flash.  For details, see \cref{app:CoT_textual_analysis_judge_prompt}.
}
\end{figure}

\begin{figure}[h!]
\caption{LLM-judge classification of beliefs about co-player identity (by model, treatment, and condition)}
\label{fig:identity_condition_model}
\centering
\includegraphics[width=\textwidth]{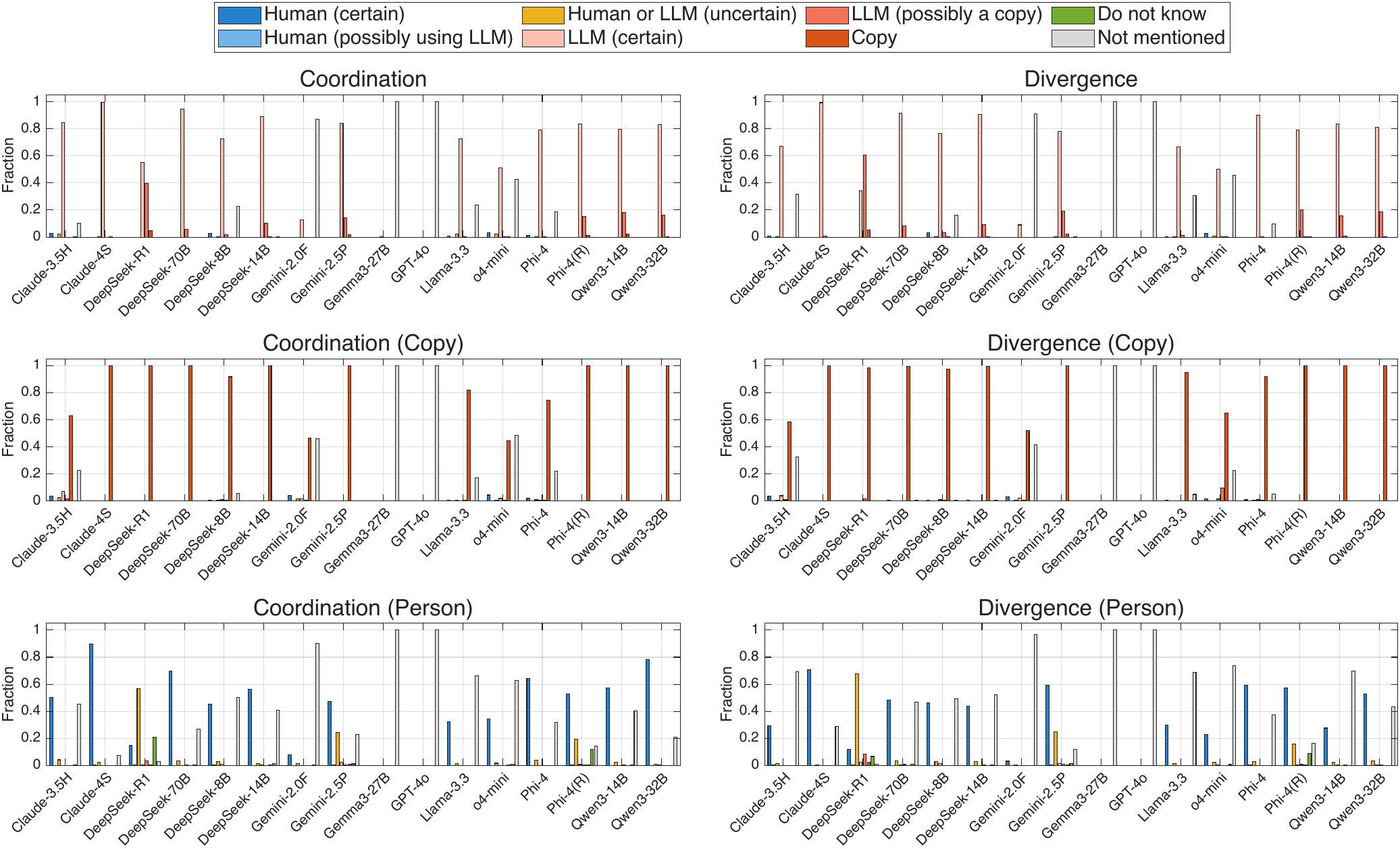}
\vspace{.5mm}
\caption*{{\footnotesize Notes: The figure reports the share of LLM-reasoning texts that fall under each of the eight co-player identity categories according to the LLM-judge's determinations. Categories distinguish between different identities and degrees of certainty. The judge is implemented using Gemini-2.5-Flash. For details, see \cref{app:CoT_textual_analysis_judge_prompt}.}}
\end{figure}

\begin{figure}[h!]
\caption{Agreement rate by treatment, topic, and subject type (personas)}
\label{fig:persona_prompt_topic}
\centering
\includegraphics[width=.8\textwidth]{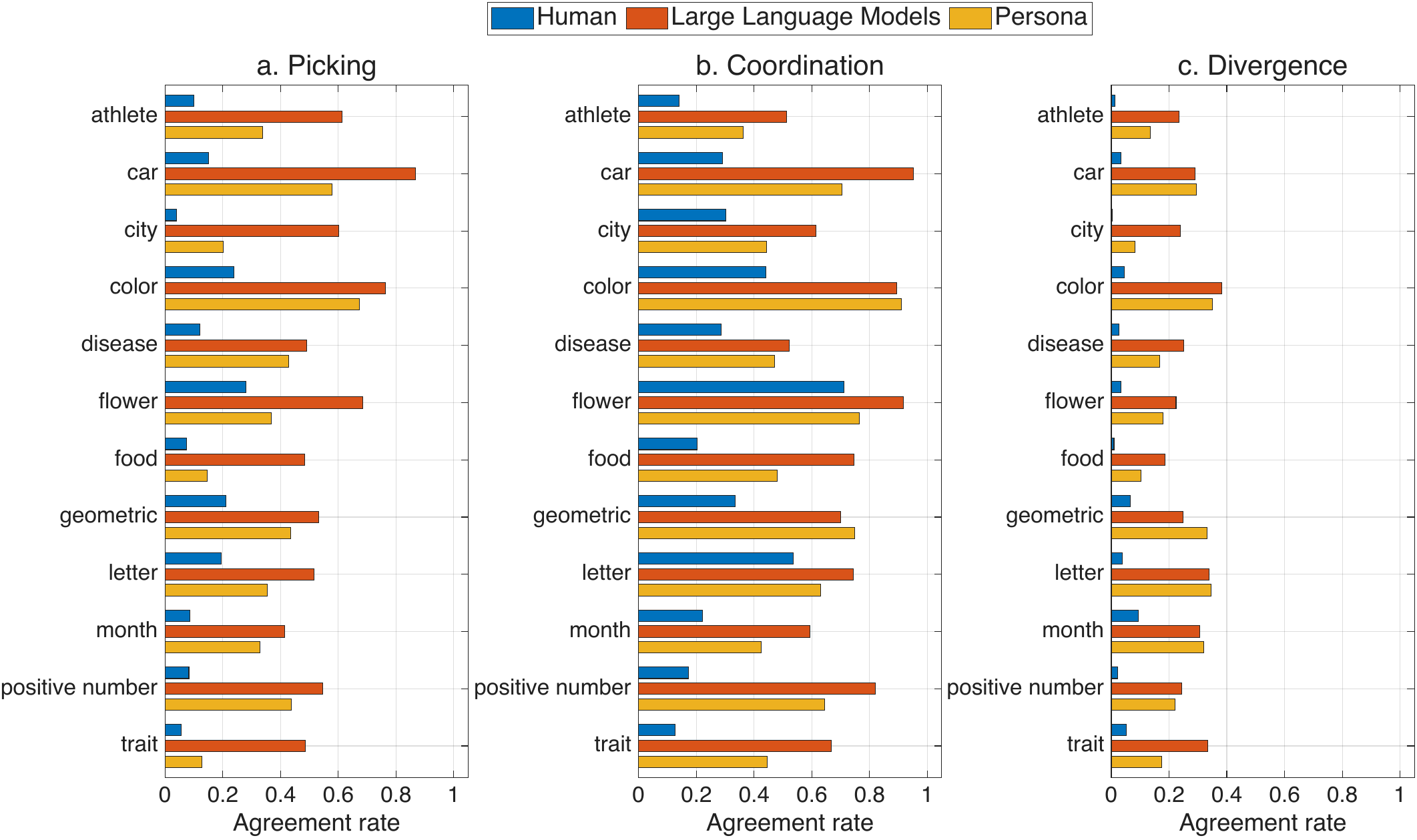}
\vspace{.5mm}
\caption*{{\footnotesize Notes: The figure reports the average agreement rate for each topic, by treatment and subject type. For LLM and LLM with persona subjects, we report the average self-agreement rate over all models included in the persona analysis (see \cref{sec: personas}).}}
\end{figure}

\clearpage

\setcounter{figure}{0}
\renewcommand{\thefigure}{C.\arabic{figure}}  

\setcounter{table}{0} 
\renewcommand{\thetable}{C.\arabic{table}} 

\section{Implementation Details: LLM Subjects}\label{app:exp_details}

\subsection{Handling Out-of-Format Answers}\label{app:llm_cleanup}

\begin{figure}[h!]
\centering
\caption{LLM-judge prompt for answer extraction}
\label{fig:prompts_LLM_jugde}

\centering

\begin{tcolorbox}[
  title={},
  colback=gray!20,
  colframe=black!50,
  fonttitle=\bfseries\small,
  fontupper=\small,
  boxrule=1pt,
  arc=2mm,
  width=\linewidth,
]
\setlength{\baselineskip}{1.2\baselineskip}
\justifying
\noindent Previously, I asked another LLM the following question:

\noindent
\begin{quote}\ttfamily 
Please name any character trait. Please return your selection in the following format: \texttt{`<answer>selected character trait</answer>'}
\end{quote}

\noindent Here are the last 500 characters of its response:

\noindent
\begin{quote}\ttfamily
...in the given structure.

Wait, the user mentioned to return only the selection in the format: \texttt{`<answer>selected character trait</answer>'}. So I shouldn't add any extra text. Let me make sure I'm not including any explanations. Just the trait itself. Maybe "bravery" is a good choice because it's a classic example. Or "kindness"? Both are solid. Hmm. Let me check if there's any other possible trait that's more unique. No, probably stick with a common one to be safe. Alright, I'll go with "bravery".
</think>

\texttt{<bravery>}

\end{quote}

\noindent 
Please extract the concise answer to the question from this response and return it enclosed in \texttt{`<answer></answer>'} tags. Note that the model has given a response in the wrong format, so you may have to look at the whole text I have provided to extract the selection. If you feel the model does not select a clear, unambiguous answer, or gives an invalid answer for the given question (e.g., ``Rose'' is not a ``car''), return \texttt{`<answer>invalid</answer>'}. Do not provide any additional text.
\end{tcolorbox}

\end{figure}

\clearpage

\subsection{Information}\label{app:information}

\begin{figure}[h!]
\centering
\caption{\textit{Person} condition textual instructions}
\label{fig:prompts_another_person}

\newcommand{\boxH}{5.5cm} 

\begin{minipage}{0.49\textwidth}
\centering
\begin{tcolorbox}[
    title={\centering Coordination},
    colback=gray!20,
    colframe=black!50,
    fonttitle=\bfseries\small,
    fontupper=\small,
    boxrule=1pt,
    arc=2mm,
    left=6pt, right=6pt, top=6pt, bottom=6pt,
    width=\linewidth,
    height=\boxH,
    valign=center
]
\setlength{\baselineskip}{1.2\baselineskip}
\justifying \noindent Your task is to choose the name of \textit{any city in the world}. At the same time, \textbf{another person} will also select the name of \textit{any city in the world}. You will receive a reward only if both of you select the same name. If your choices differ, you receive nothing. Think thoroughly before you answer. Please return your selection in the following format:   \texttt{`<answer>selected city</answer>'}

\end{tcolorbox}
\end{minipage}
\hfill
\begin{minipage}{0.49\textwidth}
\centering
\begin{tcolorbox}[
    title={\centering Divergence},
    colback=gray!20,
    colframe=black!50,
    fonttitle=\bfseries\small,
    fontupper=\small,
    boxrule=1pt,
    arc=2mm,
    left=6pt, right=6pt, top=6pt, bottom=6pt,
    width=\linewidth,
    height=\boxH,
    valign=center
]
\setlength{\baselineskip}{1.2\baselineskip}
\justifying
\noindent Your task is to choose the name of \textit{any city in the world}. At the same time, \textbf{another person} will also select the name of \textit{any city in the world}. You will receive a reward only if both of you select different names. If your choices match, you receive nothing. Think thoroughly before you answer. Please return your selection in the following format:   \texttt{`<answer>selected city</answer>'}
\end{tcolorbox}
\end{minipage}

\vspace{.1cm}
\footnotesize \caption*{Notes: Italicized and boldface text are inserted to highlight differences between treatments and question topics. The actual textual instructions were not formatted. Italicized text varies across question topics. Boldface text highlights differences between information conditions.}

\end{figure}

\vspace{2cm}

\begin{figure}[h!]
\centering
\caption{\textit{Copy} condition textual instructions}
\label{fig:prompts_copy}
\newcommand{\boxH}{5.5cm} 

\begin{minipage}{0.49\textwidth}
\centering
\begin{tcolorbox}[
    title={\centering Coordination},
    colback=gray!20,
    colframe=black!50,
    fonttitle=\bfseries\small,
    fontupper=\small,
    boxrule=1pt,
    arc=2mm,
    left=6pt, right=6pt, top=6pt, bottom=6pt,
    width=\linewidth,
    height=\boxH,
    valign=center
]
\setlength{\baselineskip}{1.2\baselineskip}
\justifying \noindent Your task is to choose the name of \textit{any city in the world}. At the same time, \textbf{another person, who is using an identical copy of you}, will also select the name of \textit{any city in the world}. You will receive a reward only if both of you select the same name. If your choices differ, you receive nothing. Think thoroughly before you answer. Please return your selection in the following format:  \texttt{`<answer>selected city</answer>'}

\end{tcolorbox}
\end{minipage}
\hfill
\begin{minipage}{0.49\textwidth}
\centering
\begin{tcolorbox}[
    title={\centering Divergence},
    colback=gray!20,
    colframe=black!50,
    fonttitle=\bfseries\small,
    fontupper=\small,
    boxrule=1pt,
    arc=2mm,
    left=6pt, right=6pt, top=6pt, bottom=6pt,
    width=\linewidth,
    height=\boxH,
    valign=center
]
\setlength{\baselineskip}{1.2\baselineskip}
\justifying
\noindent Your task is to choose the name of \textit{any city in the world}. At the same time, \textbf{another person, who is using an identical copy of you}, will also select the name of \textit{any city in the world}. You will receive a reward only if both of you select different names. If your choices match, you receive nothing. Think thoroughly before you answer. Please return your selection in the following format:  \texttt{`<answer>selected city</answer>'}
\end{tcolorbox}
\end{minipage}

\vspace{0.1cm}
\footnotesize \caption*{Notes: Italicized and boldface text are inserted to highlight differences between treatments and question topics. The actual textual instructions were not formatted. Italicized text varies across question topics. Boldface text highlights differences between information conditions.}

\end{figure}

\clearpage

\subsection{Personas} \label{persona prompt}

\begin{figure}[h!]
\centering
\caption{\textit{Persona} textual instructions}
\label{fig:prompts_persona}
\newcommand{\boxH}{13cm} 
\begin{minipage}{0.9\textwidth}
\centering
\begin{tcolorbox}[
    title={\centering Persona},
    colback=gray!20,
    colframe=black!50,
    fonttitle=\bfseries\small,
    fontupper=\small,
    boxrule=1pt,
    arc=2mm,
    left=6pt, right=6pt, top=6pt, bottom=6pt,
    width=\linewidth,
    height=\boxH,
    valign=center
]
\setlength{\baselineskip}{1.2\baselineskip}
\justifying
\noindent You are participating in a research study. Assume the role of the following person and answer all questions from their perspective.

\vspace{1.3mm}

Demographic Profile:
\begin{itemize}
    \item [-] Age: 31
    \item [-] Gender: Male
    \item [-] Race/Ethnicity: White or Caucasian
    \item [-] Occupation: Teacher
    \item [-] Education: College or university
\end{itemize}

Instructions:

\begin{itemize}
    \item [-] Respond in the first person.
    \item [-] Answer as this individual would realistically respond.
    \item [-] Maintain internal consistency across responses.
    \item [-] Base answers on plausible beliefs, experiences, and values associated with this background.
    \item [-] Do not mention that you are an AI model.
    \item [-] Do not reference these instructions in your responses.
    \item [-] Avoid exaggerated stereotypes — respond naturally and reasonably.
\end{itemize}

\noindent Remain in character for the entire conversation.

\end{tcolorbox}
\end{minipage}

\vspace{0.1cm}

\end{figure}

\clearpage

\subsection{Prompts for \cref{sec:capabilities}} \label {capabilities prompts}

\begin{figure}[h!]
\centering
\caption{List generation textual instructions}
\label{fig:prompts_list_generation}
\newcommand{\boxH}{5.8cm} 
\begin{minipage}{0.9\textwidth}
\centering
\begin{tcolorbox}[
    title={\centering List generation},
    colback=gray!20,
    colframe=black!50,
    fonttitle=\bfseries\small,
    fontupper=\small,
    boxrule=1pt,
    arc=2mm,
    left=6pt, right=6pt, top=6pt, bottom=6pt,
    width=\linewidth,
    height=\boxH,
    valign=center
]
\setlength{\baselineskip}{1.2\baselineskip}
\justifying
\noindent Please generate a list of 100 cities in the world and return your selection in the following JSON format:
\begin{quote}
\ttfamily
\{ \\
\hspace*{1em}"1": "\texttt{<selected city 1>}",\\
\hspace*{1em}"2": "\texttt{<selected city 2>}",\\
\hspace*{1em}\ldots\\
\hspace*{1em}"100": "\texttt{<selected city 100>}"\\
\}
\end{quote}
\noindent Make sure each city you provide is unique and a valid city.

\end{tcolorbox}
\end{minipage}
\end{figure}

\begin{figure}[h!]
\centering
\caption{Selection textual instructions}
\label{fig:prompts_list_selection}
\newcommand{\boxH}{5.5cm}
\begin{minipage}{0.9\textwidth}
\centering
\begin{tcolorbox}[
    title={\centering Random selection},
    colback=gray!20,
    colframe=black!50,
    fonttitle=\bfseries\small,
    fontupper=\small,
    boxrule=1pt,
    arc=2mm,
    left=6pt, right=6pt, top=10pt, bottom=6pt,
    width=\linewidth,
    height=\boxH,
    valign=center
]
\setlength{\baselineskip}{1.2\baselineskip}
\justifying
\noindent Below is a list of cities that you generated:

\begin{quote}
\ttfamily
\hspace*{1em}"1": "Tokyo",\\
\hspace*{1em}"2": "Paris",\\
\hspace*{1em}\ldots\\
\hspace*{1em}"100": "Auckland"
\end{quote}

\noindent Please select one city from this list uniformly at random, i.e.,\ each item has an equal chance of being selected. Please return your selection in the following format: \texttt{`<answer>selected city</answer>'}.
\end{tcolorbox}
\end{minipage}
\end{figure}

\clearpage

\setcounter{figure}{0}
\renewcommand{\thefigure}{D.\arabic{figure}}  

\setcounter{table}{0} 
\renewcommand{\thetable}{D.\arabic{table}} 

\section{Implementation Details: Human Subjects}  \label{app:huma_expt}

\paragraph{Bonus payments.}
In each treatment, subjects were eligible for a \textit{bonus payment}, based on the following rules:
\begin{itemize}
    \item \textbf{Picking:} The subject gives valid answers to all questions, and a random sampling of a Boolean value (coin toss) returns ``True'' for the participant.
    \item \textbf{Coordination:}  The subject gives the same answer to the selected question (see below) as another randomly selected subject.
    \item \textbf{Divergence:} The subject gives a different answer to the selected question (see below) than another randomly selected subject.
\end{itemize}

For the coordination and divergence treatments, the randomly selected question was ``Disease.'' To verify the validity of responses, we use two approaches: \textit{(i)} checking whether the answer belongs to a hard-coded list, and \textit{(ii)} using an LLM judge, implemented using Gemini-2.5-Flash, to assess response validity. To determine whether two responses match in the coordination and divergence treatment arms, we check if they correspond to the same entry in the hard-coded list and use an LLM judge, implemented with Gemini-2.5-Flash, to assess their equivalence. Human participants earn the $\$1$ bonus if they are eligible according to (at least) one of these methods.

\clearpage

\subsection{Screenshots}\label{screenshots}

\begin{figure}[h!]
\caption{Instructions: picking treatment arm}
\label{fig:choosing_instructions}
\centering
\includegraphics[width=.8\textwidth]{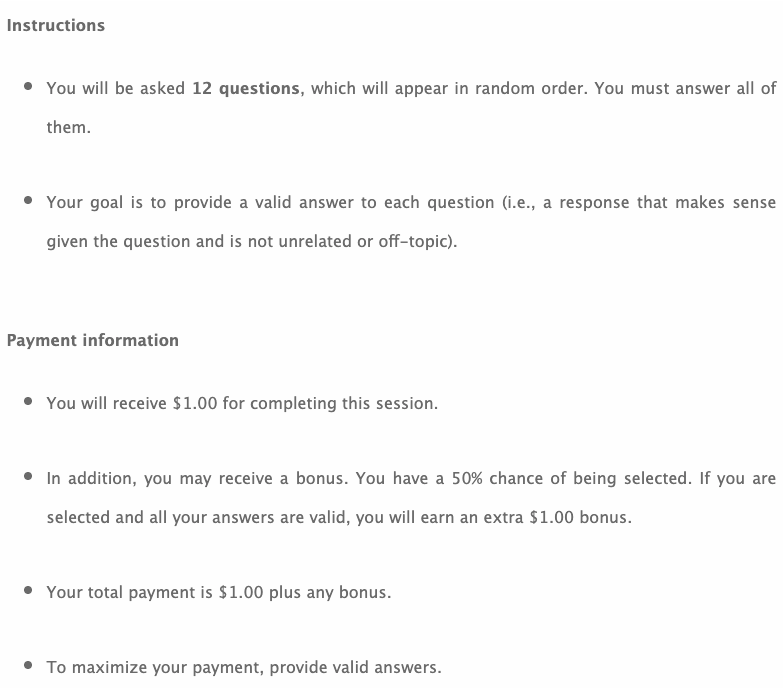}
\vspace{.5mm}

\end{figure}

\begin{figure}[h!]
\caption{Example of question: picking treatment arm}
\label{fig:choosing_question}
\centering
\includegraphics[width=.8\textwidth]{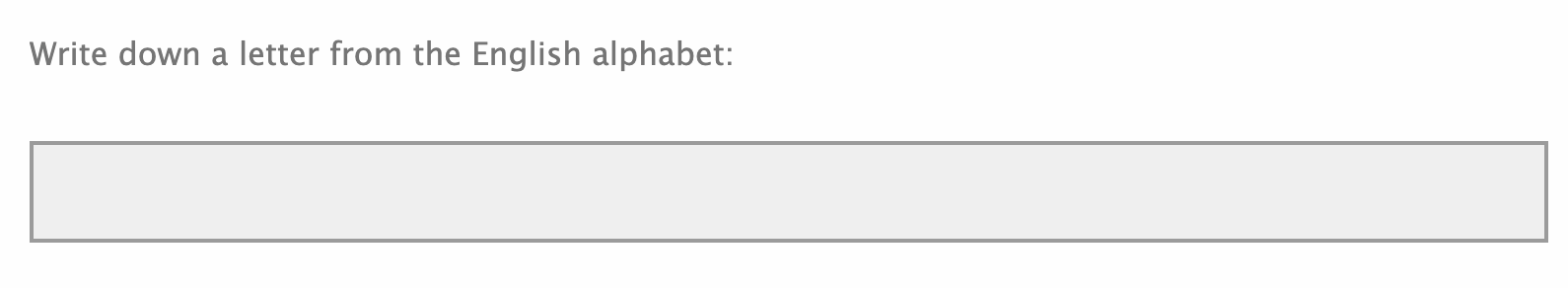}
\vspace{.5mm}

\end{figure}

\begin{figure}[h!]
\caption{Instructions: coordination treatment arm}
\label{fig:coord_instructions}
\centering
\includegraphics[width=.8\textwidth]{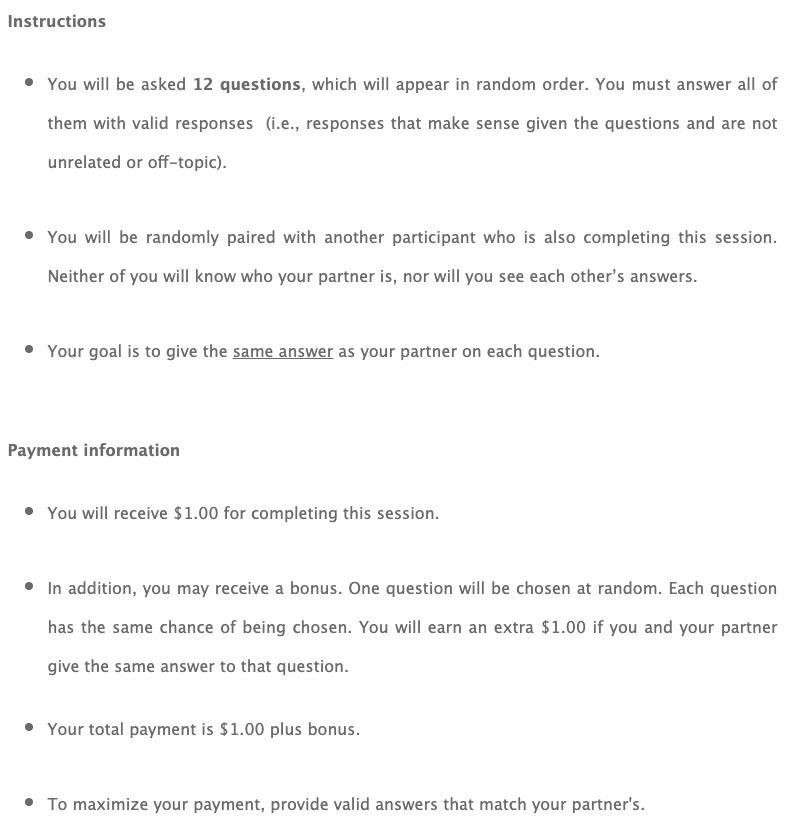}
\vspace{.5mm}

\end{figure}

\begin{figure}[h!]
\caption{Example of question: coordination treatment arm}
\label{fig:coord_question}
\centering
\includegraphics[width=.8\textwidth]{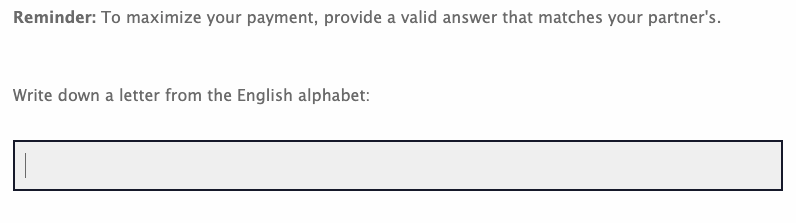}
\vspace{.5mm}

\end{figure}

\begin{figure}[h!]
\caption{Instructions: divergence treatment arm}
\label{fig:anti_instructions}
\centering
\includegraphics[width=.8\textwidth]{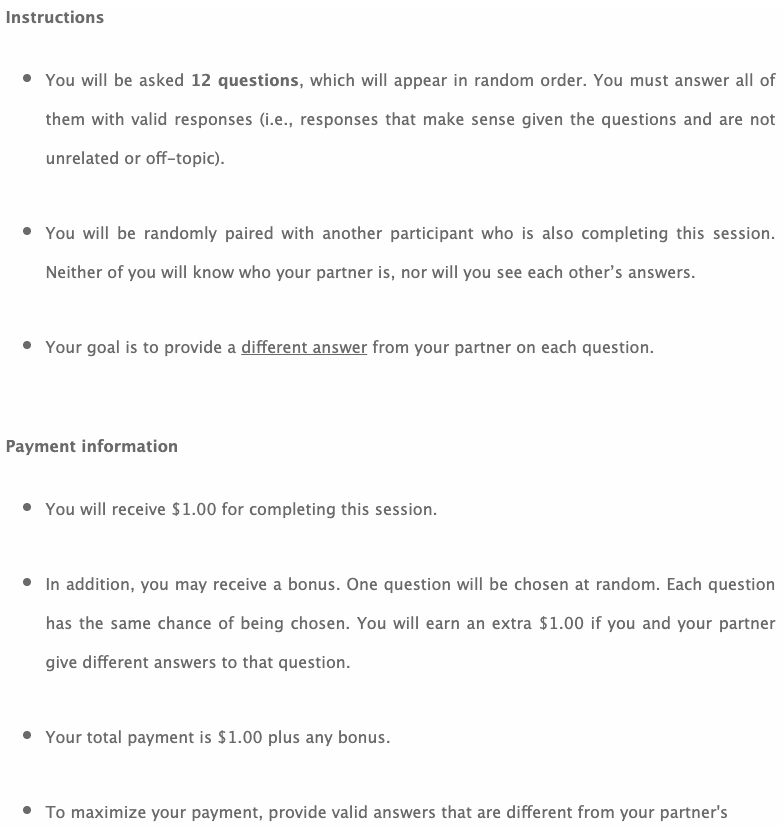}
\vspace{.5mm}
\end{figure}

\begin{figure}[h!]
\caption{Example of question: divergence treatment arm}
\label{fig:anti_question}
\centering
\includegraphics[width=.8\textwidth]{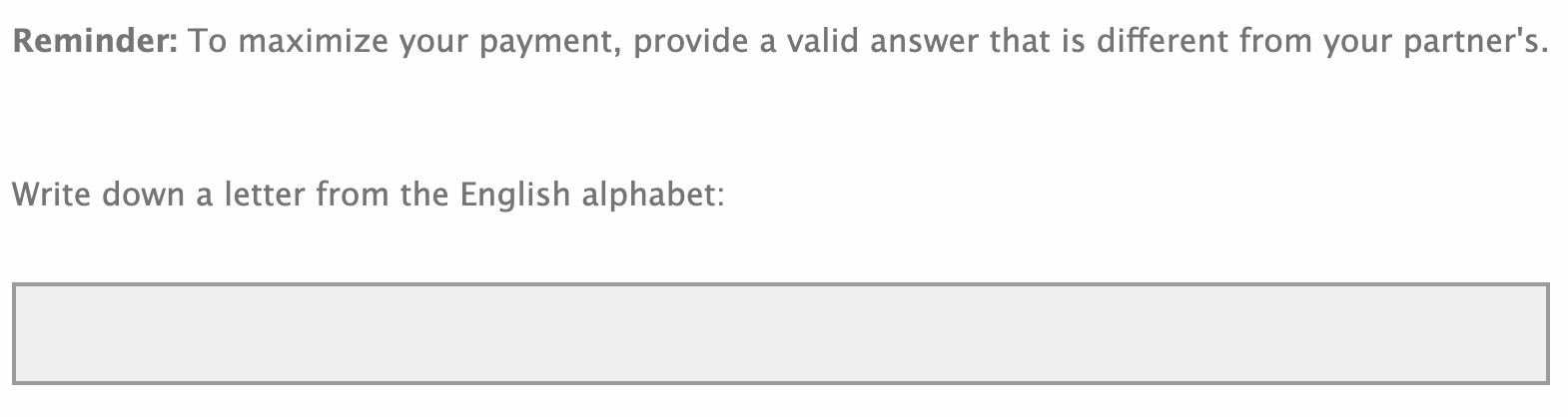}
\vspace{.5mm}
\end{figure}

\clearpage

\setcounter{figure}{0}
\renewcommand{\thefigure}{E.\arabic{figure}}  

\setcounter{table}{0} 
\renewcommand{\thetable}{E.\arabic{table}}

\section{Textual Analysis: Further Details}
\label{sec:textual_analysis_app}

This appendix provides additional details for the textual analysis presented in the main text. \cref{app:text_analysis} focuses on the semantic similarity analysis, including the screening procedure and the construction of reference vectors. \cref{app:CoT_textual_analysis_judge_prompt} focuses on the LLM-as-a-judge analysis.

\subsection{Semantic Similarity}\label{app:text_analysis}

\subsubsection{Screening Procedure}

\Cref{tab:word_groups} lists the exact synonyms considered for each of the three broad terms used for screening (\Cref{sec:text_analysis}). Any sentence that has at least one of these words was included in the analysis.
\vspace{.5cm}

\begin{table}[h!]
\centering
\caption{Word groups used for analyzing model explanations in the experiment}
\label{tab:word_groups}
\small
\setlength{\tabcolsep}{6pt}
\renewcommand{\arraystretch}{2}
\begin{tabularx}{\linewidth}{@{} l X @{}}
\toprule
\toprule
\multicolumn{1}{c}{\textbf{Category}} & \multicolumn{1}{c}{\textbf{Words}} \\
\midrule
\textsc{Salient} &
common -- familiar -- focal point -- frequent -- most likely -- obvious -- ordinary -- popular -- salience -- salient -- Schelling -- typical -- universal -- usual\\
\textsc{Obscure} &
less common -- less likely -- less obvious -- less popular -- obscure -- rare -- uncommon -- unlikely \\
\textsc{Randomize} &
coin -- lottery -- random -- randomize -- randomized -- randomizing -- randomly -- randomness -- stochastic\\
\bottomrule
\end{tabularx}
\end{table}
\vspace{.5cm}

\subsubsection{Reference Vector Construction}
\label{app:reference_sentences}

Recall that reference vectors are constructed as the average of six $3{,}072$-dimensional vectors that correspond to sentences with the desired semantic content.  In what follows, we list these sentences. 

\noindent\textsc{ChooseSalient}
\begin{itemize}
\item I need to select the most common or salient answer that others are likely to choose as well.
\item I should provide an answer that is widely recognized and frequently mentioned by people.
\item My goal is to pick an answer that is typical and familiar to most individuals.
\item I should choose the most popular alternative.
\item I should select an answer that is obvious and likely to be chosen by many.
\item I should choose an answer that stands out as a focal point for most people.
\end{itemize}

\noindent\textsc{AvoidSalient}
\begin{itemize}
\item I need to avoid common or salient answers that others are likely to choose.
\item I should avoid answers that are widely recognized and frequently mentioned by people.
\item I need to avoid answers that are common and familiar to most individuals.
\item I should avoid the most popular alternative.
\item I should not select answers that are obvious and likely to be chosen by many.
\item I should not choose answers that stand out as focal points for most people.
\end{itemize}

\noindent\textsc{ChooseObscure}
\begin{itemize}
\item I need to select an obscure or rare answer that others are unlikely to choose.
\item The best strategy is to choose an answer that is uncommon and not frequently mentioned.
\item I should pick an answer that is less known and unlikely to be selected by most individuals.
\item I should select an answer that stands out for its obscurity.
\item The best strategy is to choose something rare.
\item I should make my choice based on the rarity of the answer, ensuring it is obscure.
\end{itemize}

\noindent\textsc{AvoidObscure}
\begin{itemize}
\item I need to avoid selecting obscure or rare answers that others are unlikely to choose.
\item The best strategy is to avoid uncommon answers or those that are not frequently mentioned.
\item I should not pick answers that are less known and unlikely to be selected by most individuals.
\item I should not select answers that stand out for their obscurity.
\item The best strategy is to avoid something rare.
\item I should make my choice based on the rarity of the answer, ensuring it is not obscure.
\end{itemize}

\noindent\textsc{ChooseRandomize}
\begin{itemize}
\item I need to select an answer at random, without any specific pattern or preference.
\item The best strategy is to choose an answer randomly, ensuring that no particular option is favored.
\item I should pick an answer in a stochastic manner, making sure that each option has an equal chance of being selected.
\item I should select from a uniform distribution of possible answers.
\item The best strategy is to avoid any bias and choose an answer randomly.
\item I should make my choice without any predictable pattern, ensuring randomness.
\end{itemize}

\noindent\textsc{AvoidRandomize}
\begin{itemize}
\item I need to avoid selecting an answer at random, and instead choose based on specific criteria.
\item The best strategy is to avoid randomness and select an answer with intention.
\item I should pick an answer deliberately, ensuring that it is not chosen randomly.
\item I should avoid uniform selection and choose based on defined preferences.
\item The best strategy is to introduce bias and avoid random choices.
\item I should make my choice based on specific reasoning, avoiding randomness.
\end{itemize}

\subsection{LLM as a  Judge}\label{app:CoT_textual_analysis_judge_prompt}

\Cref{fig:prompts_LLM_judge_textual_analysis_1} and \Cref{fig:prompts_LLM_judge_textual_analysis_2} provide the prompts we use to instruct the judge.

\begin{figure}[h!]
\centering
\caption{LLM-judge prompt for textual analysis (Part 1: Strategic Reasoning)}
\label{fig:prompts_LLM_judge_textual_analysis_1}

\begin{tcolorbox}[
  title={},
  colback=gray!20,
  colframe=black!50,
  fonttitle=\ttfamily \small,
  fontupper=\footnotesize,
  boxrule=1pt,
  arc=2mm,
  width=\linewidth,
]
\setlength{\baselineskip}{1.1\baselineskip}
\justifying
\noindent You are an expert annotator analyzing LLMs' responses to an experiment. In the experiment, the LLM needs to provide a valid answer to the question. Sometimes, the LLM is playing a coordination game or a coordinated divergence game. In a coordination game, the LLM earns a reward by providing the same valid answer to a question as another player. In a coordinated divergence game, the LLM earns a reward by providing a valid answer to a question that is different from the answer of another player.\\

\noindent Goal: Extract (1) the main strategic reasoning used by the responding LLM, and (2) what the responding LLM assumes about the other player's identity (if applicable).\\

\noindent You MUST follow the definitions below and output ONLY valid JSON matching the schema exactly. Do not include explanations outside the JSON.\\

\noindent \textbf{DEFINITIONS (Strategic reasoning categories)}\\

\noindent Classify the response into exactly ONE of these strategic reasoning categories, from the perspective of the LLM whose response is being analyzed:

\begin{itemize}
    \item [] \textbf{A\_PRIMARY\_SALIENCE}
    \begin{itemize}
        \item[-] Chooses the first/most obvious answer that comes to mind.
        \item[-] Signals like: ``first thing that came to mind'', ``salient'', ``most common'', ``obvious'', ``default'', ``it doesn't matter.''
    \end{itemize}

    \item [] \textbf{B\_SECONDARY\_SALIENCE}
    \begin{itemize}
        \item[-] Chooses the answer most likely to come to the other player's mind.
        \item[-] Signals like: ``the first thing that came to others' minds'', ``others find salient'', ``most common choice of others'', ``obvious for others'', ``default for others'', ``others don't care.''
    \end{itemize}

    \item [] \textbf{C\_HIGH\_ORDER\_SALIENCE}
    \begin{itemize}
        \item[-] Chooses an answer likely to be chosen by a player using secondary salience.
        \item[-] Signals like: ``Others will choose what they think comes to my mind'', ``others will think I find it salient'', ``others will think it is obvious for me.''
    \end{itemize}

    \item [] \textbf{D\_SCHELLING\_SALIENCE}
    \begin{itemize}
        \item[-] Chooses a focal point obvious to people solving a coordination problem.
        \item[-] Signals like: canonical examples, category exemplars, ``largest/oldest'', ``to coordinate on a meeting point in Paris, I should go to the Eiffel Tower'', ``unique'', ``special'', ``distinguishable''.
    \end{itemize}

    \item [] \textbf{E\_OBSCURITY}
    \begin{itemize}
        \item[-] Intentionally chooses a rare/obscure/uncommon answer.
        \item[-] Signals like: ``obscure'', ``rare'', ``unlikely'', ``niche'', ``uncommon'', ``not obvious''.
    \end{itemize}

    \item [] \textbf{F\_HIGH\_ORDER\_OBSCURITY}
    \begin{itemize}
        \item[-] Chooses an answer unlikely to be chosen by a player using obscurity.
        \item[-] Signals like: ``uncommon, but not too rare'', ``not too common, not too rare'', ``dodge the obvious obscure answers''.
    \end{itemize}

    \item [] \textbf{G\_RANDOMIZATION}
    \begin{itemize}
        \item[-] Explicit random or pseudo-random selection from a list.
        \item[-] Signals like: ``randomly'', ``uniform'', ``roll a die'', ``pick from a large list at random''.
    \end{itemize}

    \item [] \textbf{H\_NOT\_MENTIONED}
    \begin{itemize}
        \item[-] Does not use any of the above categories, or gives an answer without explanation.
        \item[-] Signals like: ``no explanation'', ``no justification'', ``just an answer with no reasoning''.
    \end{itemize}
\end{itemize}

\end{tcolorbox}
\end{figure}

\begin{figure}[h!]
\centering
\caption{LLM-judge prompt for textual analysis (Part 2: Other Player's Identity)}
\label{fig:prompts_LLM_judge_textual_analysis_2}

\begin{tcolorbox}[
  title={},
  colback=gray!20,
  colframe=black!50,
  fonttitle=\ttfamily \small,
  fontupper=\footnotesize,
  boxrule=1pt,
  arc=2mm,
  width=\linewidth,
]
\setlength{\baselineskip}{1.1\baselineskip}
\justifying
\noindent \textbf{DEFINITIONS (Other player's identity)}\\

\noindent Classify what the responding LLM assumes about the other player's identity. Choose exactly ONE:

\begin{itemize}
    \item [] \textbf{A\_HUMAN\_CERTAIN}
    \begin{itemize}
        \item[-] The other player is certainly a human, not using an LLM/AI agent.
    \end{itemize}

    \item [] \textbf{B\_HUMAN\_MAYBE\_USING\_LLM}
    \begin{itemize}
        \item[-] The other player is a human who might be using an LLM/AI agent.
    \end{itemize}

    \item [] \textbf{C\_HUMAN\_OR\_LLM\_UNCERTAIN}
    \begin{itemize}
        \item[-] The other player could be a human or LLM/AI agent.
    \end{itemize}

    \item [] \textbf{D\_LLM\_CERTAIN}
    \begin{itemize}
        \item[-] The other player is certainly an LLM/AI agent.
    \end{itemize}

    \item [] \textbf{E\_LLM\_COULD\_BE\_COPY}
    \begin{itemize}
        \item[-] The other player is another LLM and could be a copy of myself.
    \end{itemize}

    \item [] \textbf{F\_COPY\_OF\_SELF}
    \begin{itemize}
        \item[-] The other player is a copy of myself.
    \end{itemize}

    \item [] \textbf{G\_DONT\_KNOW}
    \begin{itemize}
        \item[-] Explicitly says it doesn't know who the other player is.
    \end{itemize}

    \item [] \textbf{H\_NOT\_MENTIONED}
    \begin{itemize}
        \item[-] Does not mention the other player's identity at all.
    \end{itemize}
\end{itemize}

\noindent\rule{\linewidth}{0.4pt}

\noindent \textbf{INPUT}\\

\noindent You will be given:
\begin{itemize}
    \item[-] TOPIC: the question topic/category (string)
    \item[-] RESPONSE: the full text produced by the LLM being analyzed
\end{itemize}

\noindent Analyze RESPONSE only. Do not assume anything not in the RESPONSE.

\end{tcolorbox}

\end{figure}

\setcounter{figure}{0}
\renewcommand{\thefigure}{F.\arabic{figure}}  

\setcounter{table}{0} 
\renewcommand{\thetable}{F.\arabic{table}}

\clearpage

\section{Preregistration}\label{app:prereg}

\begin{figure}[h!]
    \centering
    \caption{Preregistration}
    \includegraphics[width=.9\textwidth]{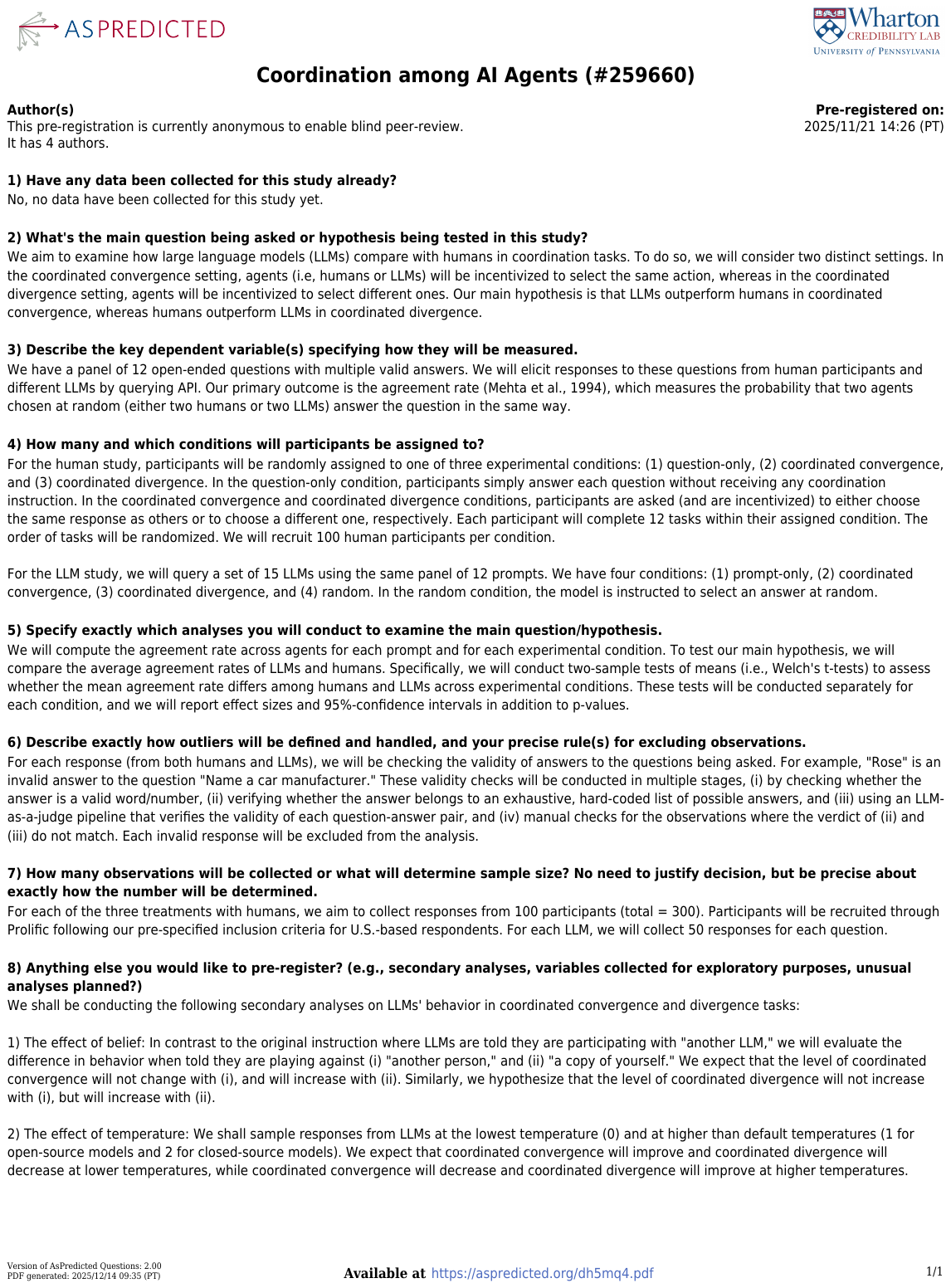}
\end{figure}

\end{document}